\documentclass[10pt,journal,compsoc]{IEEEtran}
\ifCLASSOPTIONcompsoc
  
  \usepackage[nocompress]{cite}
\else
 
  \usepackage{cite}
\fi
\usepackage{times}
\usepackage{bm}
\usepackage{url}  
\usepackage{float}
\usepackage{multirow}

\ifCLASSINFOpdf

\usepackage{epsfig}
\usepackage{graphicx}
\usepackage{amsmath}
\usepackage{amsthm}
\usepackage{amssymb}
\usepackage{algorithmic}
\usepackage{adjustbox,lipsum}
\usepackage[vlined,boxed,commentsnumbered]{algorithm2e}
\usepackage{subfigure}
\usepackage{tablefootnote}
\usepackage{xcolor}
\usepackage{bm}
\usepackage[font=small,skip=5pt]{caption}
\else
\fi

\newcommand{\vx}{\mathbf{x}}
\newcommand{\vy}{\mathbf{y}}
\newcommand{\vw}{\mathbf{w}}
\newcommand{\va}{\mathbf{a}}
\newcommand{\vc}{\mathbf{c}}
\newcommand{\vb}{\mathbf{b}}
\newcommand{\vd}{\mathbf{d}}
\newcommand{\mQ}{\mathbf{Q}}

\newcommand{\mE}{\mathbf{E}}
\newcommand{\vr}{\mathbf{r}}

\newcommand{\valpha}{\bm{\alpha}}
\newcommand{\vOmega}{\bm{\Omega}}
\newcommand{\sv}{v}
\newcommand{\half}{\frac{1}{2}}
\newcommand{\reals}[1]{\mathbb{R}^{#1}}
\newcommand{\mW}{\mathbf{W}}
\newcommand{\mA}{\mathbf{A}}
\newcommand{\mD}{\mathbf{D}}
\newcommand{\mZ}{\mathbf{Z}}
\newcommand{\mU}{\mathbf{U}}
\newcommand{\fK}{\mathcal{K}}
\newcommand{\mZZ}{\left(\mathbf{Z}\odot\mathbf{Z}\right)}
\newcommand{\sqZ}{\mZ^2}
\newcommand{\sqY}{\mY^2}
\newcommand{\mX}{\mathbf{X}}
\newcommand{\mY}{\mathbf{Y}}
\newcommand{\mYY}{\left(\mathbf{Y}\odot\mathbf{Y}\right)}
\newcommand{\stiefel}[1]{\mathcal{S}^{#1}_d}

\newcommand{\gs}[1]{\mathcal{G}_{\kernel}^{#1}}
\newcommand{\dataset}{\mathcal{D}}
\newcommand{\dl}{\mathcal{D}_o}
\newcommand{\dbar}{\mathcal{\overline{D}}}
\newcommand{\subjectto}{\text{ s.t. }}
\newcommand{\unitsphere}[1]{U^{#1}}
\newcommand{\eye}[1]{\mathbf{I}_{#1}}

\newcommand{\norm}[1]{\left\|{#1}\right\|}
\newcommand{\enorm}[1]{\left\|{#1}\right\|_2}
\newcommand{\fnorm}[1]{\left\|{#1}\right\|_F}

\newcommand{\set}[1]{\left\{#1\right\}}
\newcommand{\hinge}[1]{\left[{#1}\right]_+}
\newcommand{\indexset}[1]{[{#1}]}

\newcommand{\kernel}{\mathbb{K}}
\newcommand{\one}{\mathbf{1}}
\newcommand{\zero}{\mathbf{0}}

\newcommand{\problem}{P}

\newcommand{\kods}{\mathcal{K}}
\newcommand{\ob}[1]{{\mathcal{OB}}^{#1}_d}
\DeclareMathOperator*{\diag}{diag}
\DeclareMathOperator*{\rowmax}{rowmax}
\DeclareMathOperator*{\rowmin}{rowmin}

\DeclareMathOperator*{\sign}{sgn}
\DeclareMathOperator*{\dist}{dist^2}
\DeclareMathOperator*{\Dist}{dist}
\DeclareMathOperator*{\tr}{Tr}
\DeclareMathOperator*{\distW}{dist_W^2}
\DeclareMathOperator*{\DistW}{dist_W}
\DeclareMathOperator*{\argmax}{arg\,max}
\DeclareMathOperator*{\argmin}{arg\,min}
\DeclareMathOperator*{\grad}{grad}
\DeclareMathOperator{\gods}{GODS}

\newcommand{\trace}[1]{\tr\left({#1}\right)}
\newcommand{\rottxt}[1]{\parbox[t]{2mm}{\multirow{4}{*}{\rotatebox[origin=c]{90}{#1}}}}

\newtheorem{theorem}{Theorem}
\newtheorem{proposition}{Proposition}
\newtheorem{lemma}{Lemma}

\begin{document}

\title{Generalized One-Class Learning \\Using Pairs of Complementary Classifiers}

\author{ Anoop Cherian$^*$ \qquad\qquad Jue Wang$^*$  \\
\IEEEcompsocitemizethanks{
\IEEEcompsocthanksitem Anoop Cherian (corresponding author) is with Mitsubishi Electric Research Labs (MERL), Cambridge, MA, E-mail: cherian@merl.com\protect\\
\IEEEcompsocthanksitem Jue Wang is with the Research School of Engineering, The
Australian National University, ACT 2601, Australia. E-mail: jue.wang@anu.edu.au. Work done while interning at MERL.\protect\\ 
$^*$ Equal contribution.
}
\thanks{}}

\markboth{TRANSACTIONS ON PATTERN ANALYSIS AND MACHINE INTELLIGENCE}%
{Shell \MakeLowercase{\textit{et al.}}:Learning $\alpha\beta$-Divergences for Positive Definite Matrices}

\IEEEtitleabstractindextext{%
\begin{abstract}
One-class learning is the classic problem of fitting a model to the data for which annotations are available only for a single class. In this paper, we explore novel objectives for one-class learning, which we collectively refer to as \emph{Generalized One-class Discriminative Subspaces} (GODS). Our key idea is to learn a pair of complementary classifiers to flexibly bound the one-class data distribution, where the data belongs to the positive half-space of one of the classifiers in the complementary pair and to the negative half-space of the other. To avoid redundancy while allowing non-linearity in the classifier decision surfaces, we propose to design each classifier as an orthonormal frame and seek to learn these frames via jointly optimizing for two conflicting objectives, namely: i) to minimize the distance between the two frames, and ii) to maximize the margin between the frames and the data. The learned orthonormal frames will thus characterize a piecewise linear decision surface that allows for efficient inference, while our objectives seek to bound the data within a minimal volume that maximizes the decision margin, thereby robustly capturing the data distribution. We explore several variants of our formulation under different constraints on the constituent classifiers, including kernelized feature maps. We demonstrate the empirical benefits of our approach via experiments on data from several applications in computer vision, such as anomaly detection in video sequences, human poses, and human activities. We also explore the generality and effectiveness of GODS for non-vision tasks via experiments on several UCI datasets, demonstrating state-of-the-art results. 
\end{abstract}

}

\maketitle
\IEEEdisplaynontitleabstractindextext

\IEEEpeerreviewmaketitle

\section{Introduction}
\label{sec:intro}
\IEEEPARstart{T}{here} are several real-world applications for which it may be straightforward to characterize the normal operating behavior and collect data to train learning systems, however may be difficult or sometimes even impossible to have data when abnormalities or rare events happen. Examples include but not limited to, an air conditioner making a spurious vibration, a network attacked by an intruder, sudden variations in a patient's vitals, or an accident captured in a video surveillance camera, among others~\cite{chandala2009anomaly}. In machine learning literature, such problems are usually called one-class problems~\cite{bishop1994novelty,ritter1997outliers}, signifying the fact that we may have an unlimited supply of labeled training data for the one-class (corresponding to the normal operation of the system), but do not have any labels or training data for situations corresponding to abnormalities. 

\begin{figure}[!ht]
    \centering
    \subfigure[BODS-Gaussian]{\label{fig:bods-vis}\includegraphics[width=2.5cm,height=2.5cm, trim={6cm 3cm 3cm 1cm},clip]{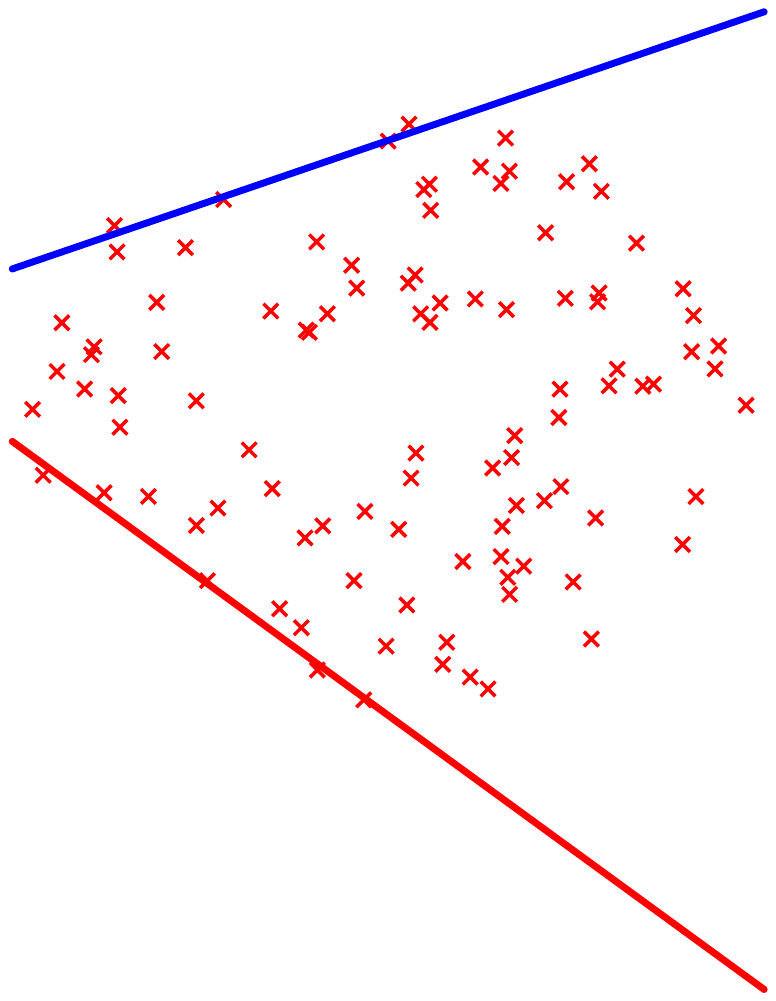}}
    \subfigure[GODS-Gaussian]{\label{fig:gods-vis}\includegraphics[width=2.5cm,height=2.5cm,trim={6cm 3cm 3cm 1cm},clip]{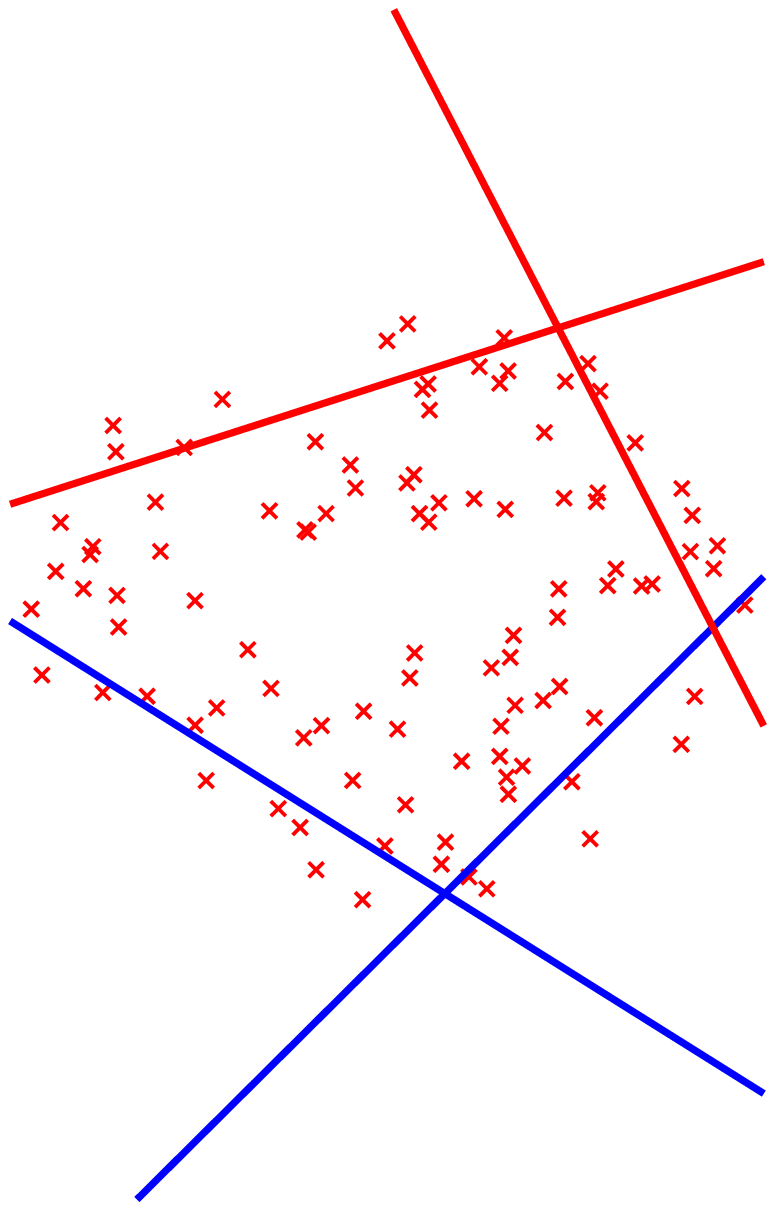}}
    \subfigure[GODS-Arbitrary]{\label{fig:gods-vis-arb}\includegraphics[width=2.5cm,height=2.5cm,trim={6cm 3cm 3cm 1cm},clip]{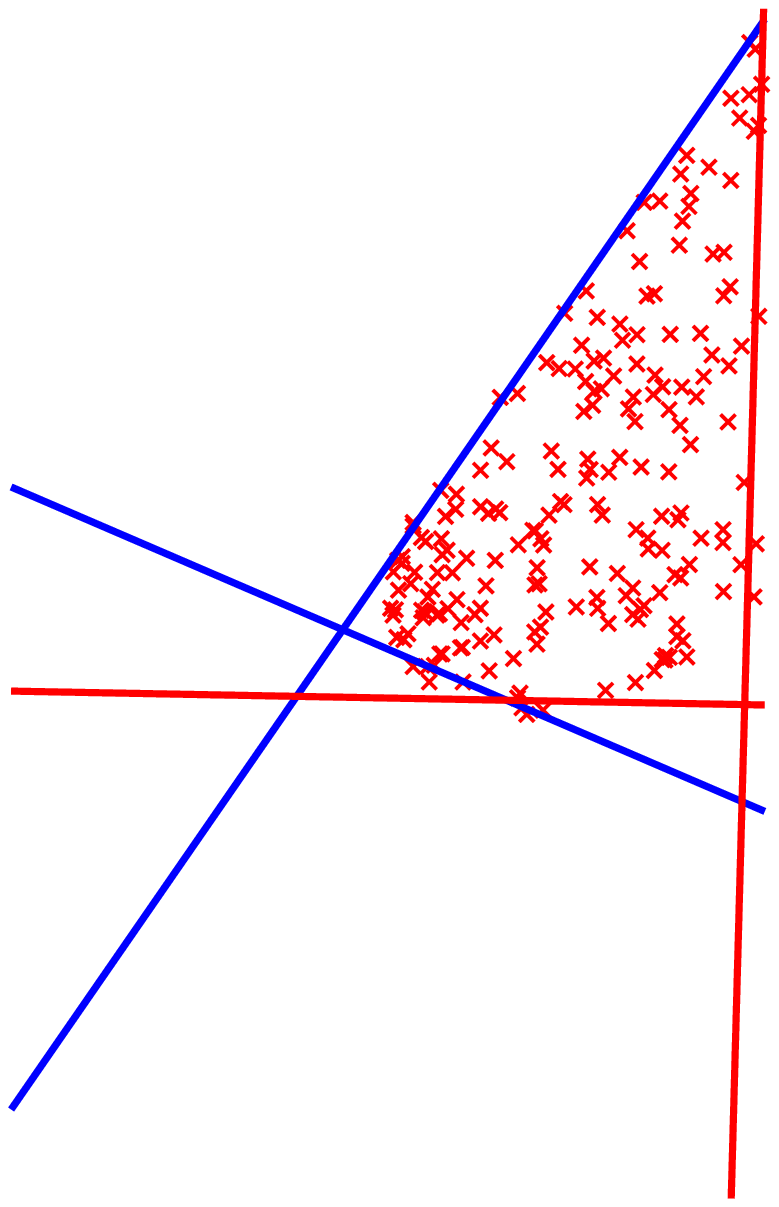}}\\
    \subfigure[KODS on 2D data]{\label{fig:kods-vis}\includegraphics[width=2.5cm,trim={4cm 3cm 4cm 3cm},clip]{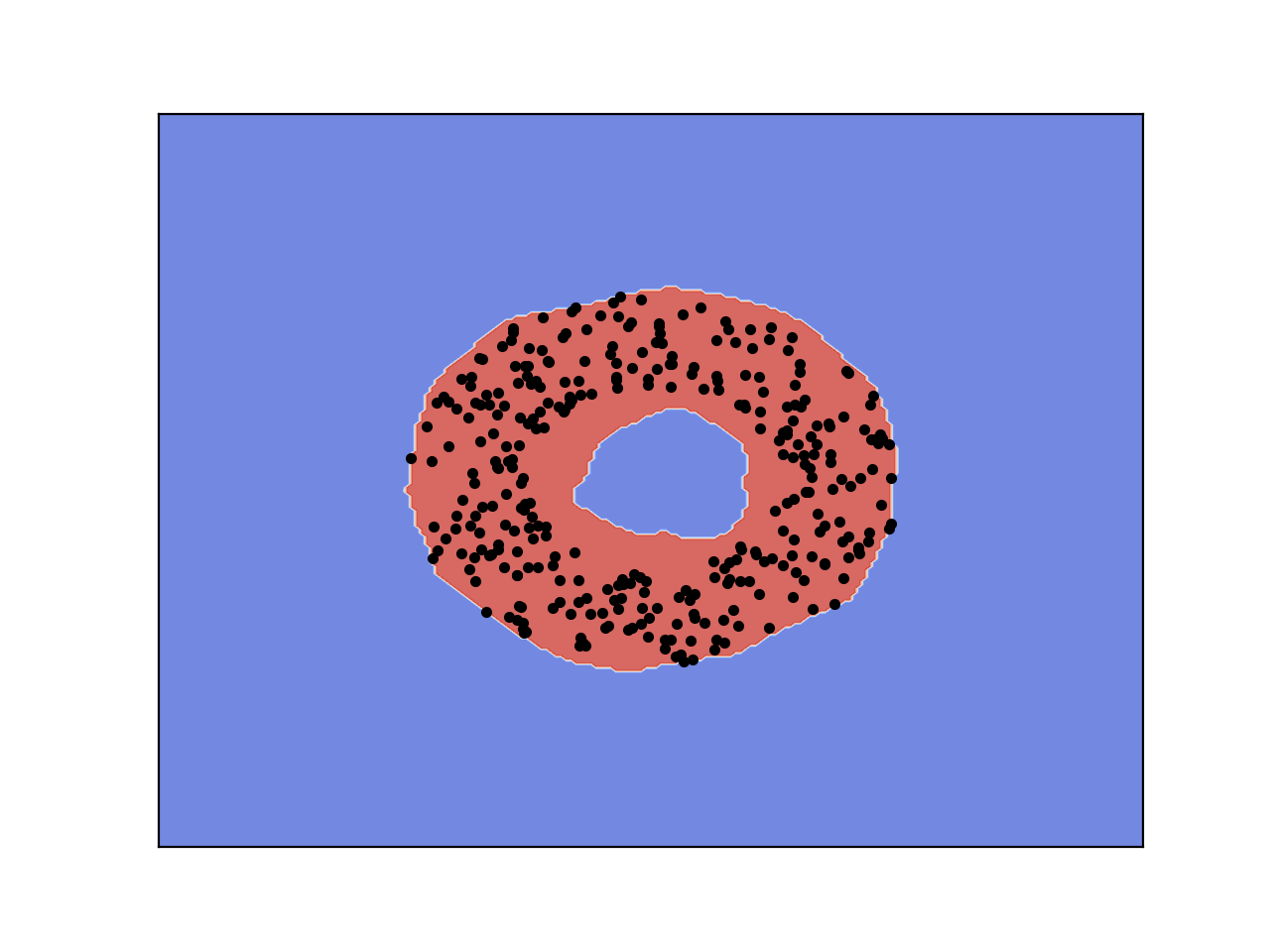}}
    \subfigure[KODS-$\mW_1$]{\label{fig:kods-vis-w1}\includegraphics[width=3cm,trim={1cm 2cm 1cm 1cm},clip]{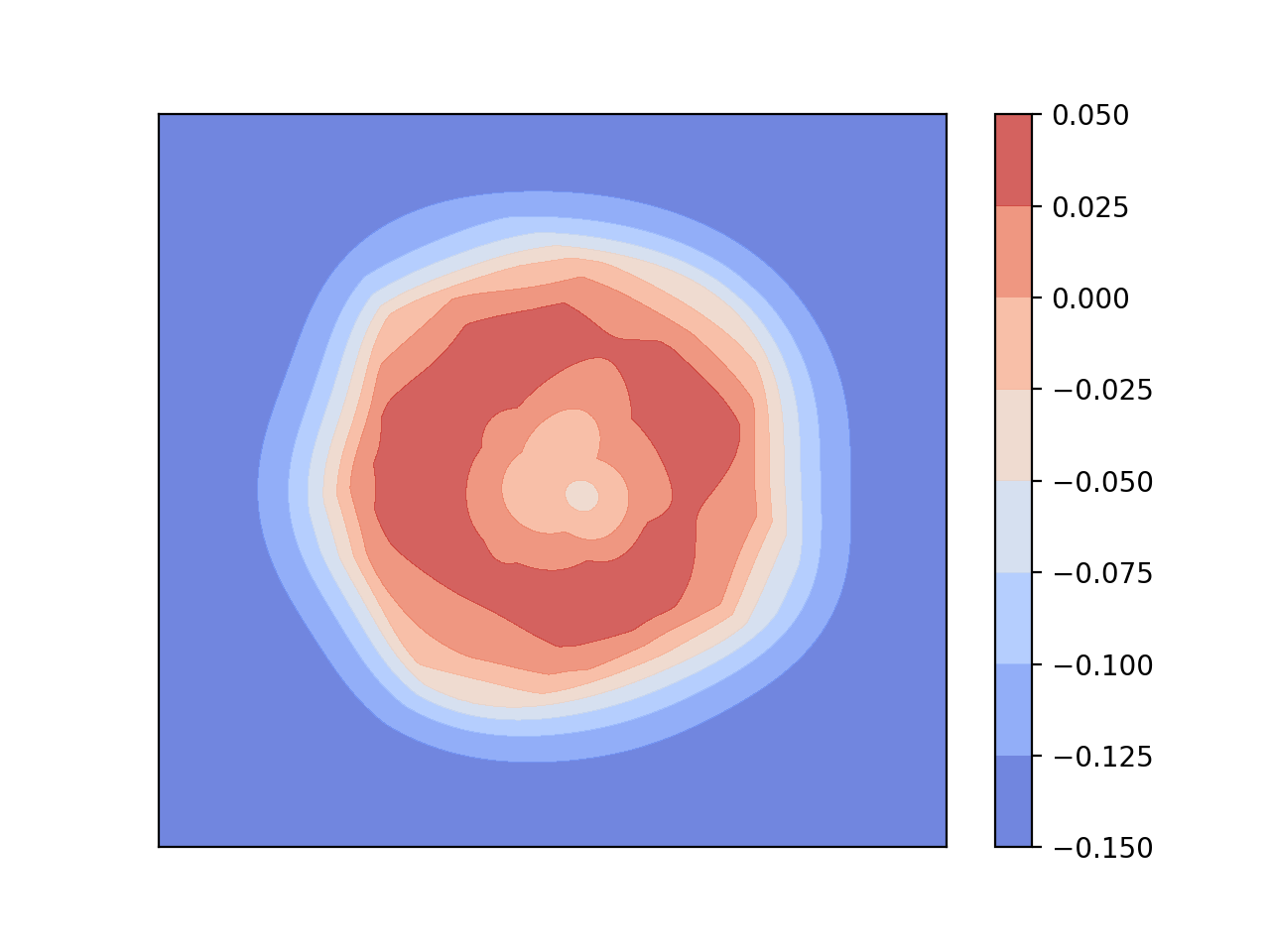}}
    \subfigure[KODS-$\mW_2$]{\label{fig:kods-vis-w2}\includegraphics[width=3cm,trim={1cm 2cm 1cm 1cm},clip]{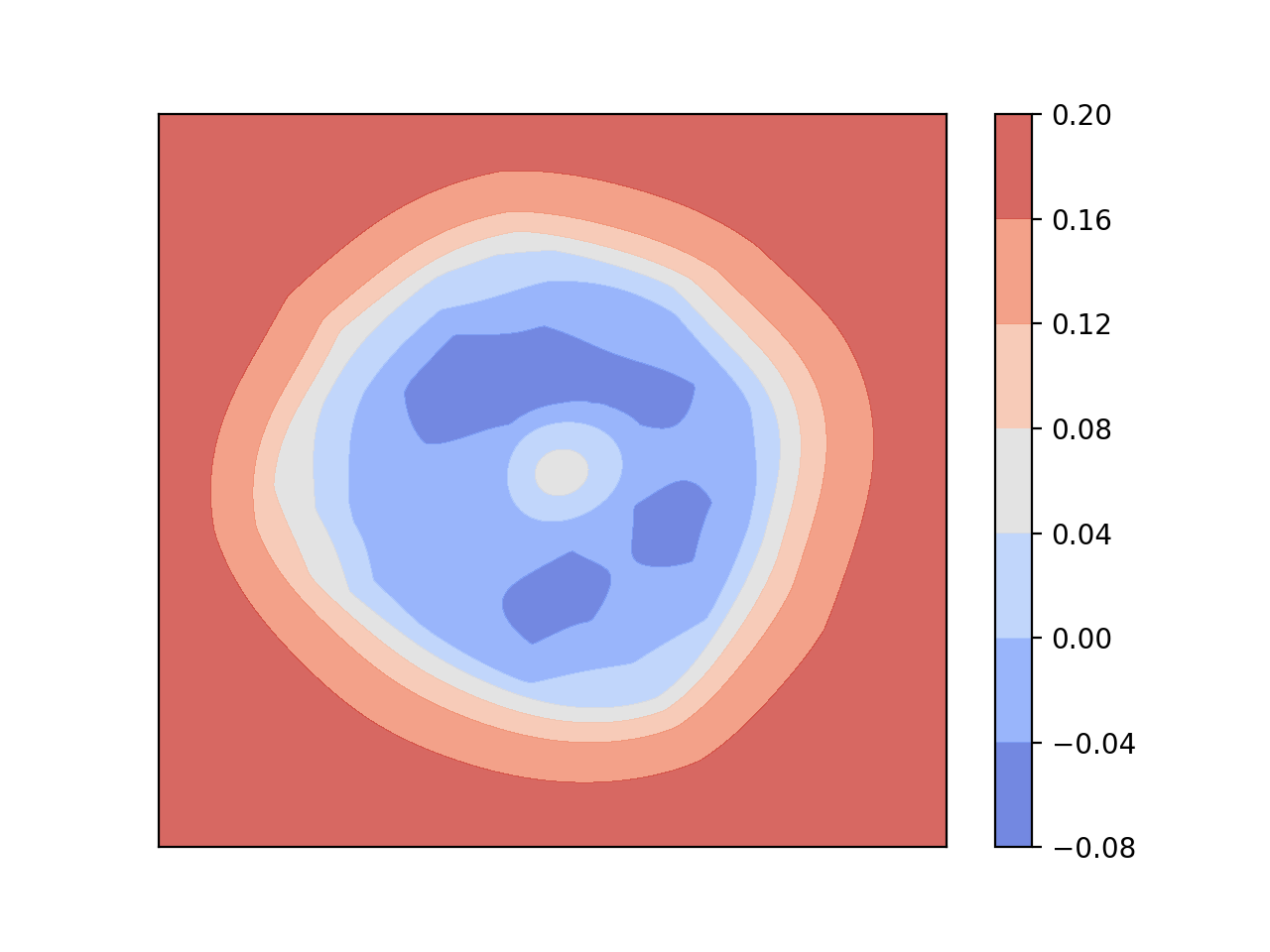}}
    \subfigure[3D data points\qquad\quad (h) KODS-$\mW_1$\qquad\qquad (i) KODS-$\mW_2$\qquad\quad ]{\label{fig:kods-3d}\includegraphics[width=9cm,trim={1cm 6cm 4cm 5cm},clip]{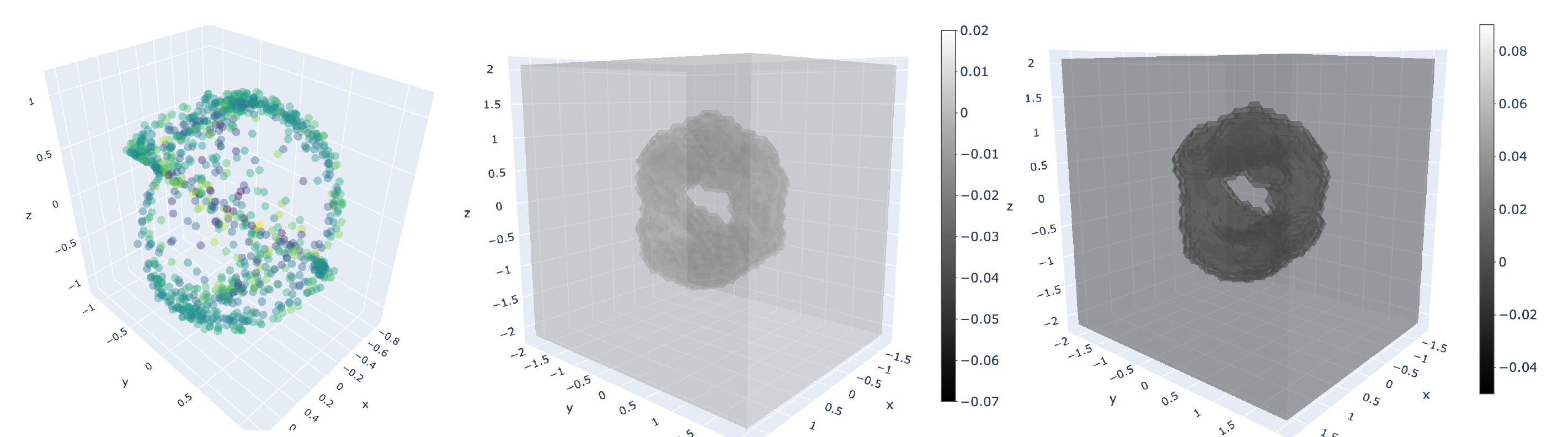}}
\caption{Visualizations of decision regions using various GODS formulations on synthetic data. Figures (a, b, c) show subspaces found by BODS and GODS on various data distributions (see Section~\ref{exp_result}). The colors identify hyperplanes within a classifier in the complementary pair. Figure (d) shows KODS decision regions for ring-shaped data (black dots) using an RBF kernel. Figures (e, f) are the decision regions of the classifiers; $\mW_1$ bounding data from outside and $\mW_2$ from inside, together they define the region in (d). Figure (g,h,i) show 3D points and the decision surfaces of the two classifiers.}
\label{fig:toy}
\end{figure}

Popular problems in computer vision, such as video novelty detection~\cite{gardner2006one,khan2009survey,sabokrou2018adversarially}, surveillance anomaly detection~\cite{del2016discriminative,saligrama2012video,liu2018future}, image denoising~\cite{buades2005non}, and outlier detection~\cite{xia2015learning,you2017provable}, are variants of the standard one-class setting. The main goal of such methods is usually to learn a model that fits to the normal set, such that abnormalities can be characterized as outliers of this learned model. Given that the distribution of abnormal samples are often unbounded or could even partially overlap with the normal data samples (e.g., in a video surveillance application, when both normal and abnormal features could be present in a time window), achieving one-class classification can be practically challenging.

Classical solutions to one-class problems are mainly extensions to support vector machines (SVMs), such as the one-class SVM (OC-SVM), that maximizes the margin of the discriminative hyperplane from the origin~\cite{scholkopf2001estimating}. There are extensions of this scheme, such as the least-squares one-class SVM (LS-OSVM)~\cite{choi2009least} or its online variants~\cite{wang2013online},  that learn to find a tube of minimal diameter that includes all the labeled data. Another popular approach is the support-vector data description (SVDD) that finds a hypersphere of minimum radius that encapsulates the training data~\cite{tax2004support}. There have also been kernelized extensions of these schemes that use the kernel trick to embed the data points in a reproducible kernel Hilbert space, potentially enclosing the `normal' data with arbitrarily-shaped boundaries. 

Apart from the classic one-class solutions, there is an increasing number of recent works that use deep neural networks for building the one-class model~\cite{abati2019latent,chalapathy2018anomaly,sabokrou2018adversarially,liu2018future,luo2017revisit,hasan2016learning}. In these approaches, typically a deep auto-encoder model is trained on the one-class data such that its reconstruction error is minimized. When such a model is provided with an out-of-distribution data sample (anomaly), the reconstruction error can be large, which could be used as an anomaly cue~\cite{hasan2016learning,sabokrou2018adversarially}. There are extensions of this general architecture using generative adversarial networks (GANs) to characterize the distribution of the in-class samples~\cite{schlegl2017unsupervised,ravanbakhsh2019training}. For anomaly detection on time-varying inputs, there are also adaptations using predictive auto-encoders, such as~\cite{liu2018future}, that attempts to generate the (latent) future samples, and flags  anomalies if the predicted sample is significantly different from the observed one.  

While, these approaches  have been widely adopted in several applications (see e.g., Chandala et al.~\cite{chandala2009anomaly}), they have drawbacks. For example, the OC-SVM uses only a single hyperplane, however using multiple hyperplanes may be beneficial~\cite{wang2018learning}. The SVDD scheme makes a strong assumption on the spherical nature of the data distribution. Using kernel methods may impact scalability, while deep learning methods may need specialized hardware (such as GPUs) and large training sets. Thus, trading-off between the pros and cons of these diverse prior methods, we propose novel generalizations of these techniques, which we call \emph{generalized one-class discriminative subspaces} (GODS). The key goal of GODS is to combine the linearity properties of OC-SVM, and the non-linear bounded characterization of SVDD in a single framework. However, our proposed one-class model is neither linear nor uses a spherical classifier, instead uses a pair of  orthornormal frames\footnote{Orthonormal frames are matrices with linearly independent unit norm columns.} whose columns characterize linear classifiers. Specifically, these columns are optimized such that the one-class data belongs to the positive half-spaces of the columns in one of these classifiers and to the negative half-spaces of the columns in the other; thus these classifiers jointly form a complementary pair. These classifiers, with their respective half-spaces defined by their orthonormal columns, non-linearly bound the data from different directions. Our learning objective, to find these complementary classifiers, jointly optimizes two opposing criteria: i) to minimize the distance between the two classifiers, thus bounding the data within the smallest volume, and ii) to maximize the margin between the hyperplanes and the data, thereby avoiding overfitting, while improving classification robustness. 

Our proposed GODS model offers several advantages against prior methods: (i) the piecewise linear decision boundaries approximate a non-linear classifier, while providing computationally cheap inference, and (ii) the use of the complementary classifier pair allows flexible bounding of the data distribution of arbitrary shapes, as illustrated in Figure~\ref{fig:toy}. For example, our kernelized variant of GODS (called KODS), that we introduce in Section~\ref{sec:kgods}, bounds data from outside as well as inside (see Figure~\ref{fig:kods-vis}), which is usually not possible in prior methods. Albeit these benefits, our objective is non-convex due to the orthogonality constraints. However, such non-convexity fortunately is not a significant practical concern as they naturally place the optimization objective on the Stiefel manifold~\cite{edelman1998geometry}. This is a well-studied Riemannian manifold~\cite{boothby1986introduction} for which there exist efficient non-linear optimization methods at our disposal. We use one such optimization scheme, called Riemannian conjugate gradient~\cite{absil2009optimization}, which is fast and efficient.

To empirically evaluate the benefits of our GODS formulations, we apply them to one-class data arising from various anomaly detection problems in computer vision and machine learning. One novel task we consider in this paper is that of out-of-pose (OOP) detection in cars~\cite{trivedi2007looking,trivedi2004occupant}. Specifically, in this task, our goal is to detect if the passengers or the driver are seated OOP as captured by an inward-looking dashboard camera.  For this task, we showcase the effectiveness of our approaches on a new dataset, which we call Dash-Cam-Pose. Apart from this task, we also report experimental results on several standard and public anomaly detection benchmarks in computer vision, such as on the UCF-crime~\cite{sultani2018real} and the UCSD Ped2~\cite{li2013anomaly} datasets. We also provide experiments on the standard JHMDB action recognition dataset~\cite{jhuang2013towards} re-purposed for the anomaly detection task. We further analyze the generalizability of our approach to non-computer vision applications by providing results on five UCI datasets. Our results demonstrate that GODS variants lead to significant performance improvements over the state of the art.

We summarize below the main contributions of this paper:
\begin{itemize}
\item We first introduce a basic one-class discriminative subspace (BODS) classifier that uses a pair of hyperplanes.
\item We generalize BODS to use multiple hyperplanes, termed generalized one-class discriminative subspaces (GODS) and derive a kernelized variant, termed KODS. We also present several formulations of GODS under different assumptions on the classifiers.
\item We explore Riemannian conjugate gradient algorithms for optimizing our objectives. Specifically, the BODS and GODS formulations use a Stiefel manifold, while KODS is modeled on the generalized Stiefel manifold.
\item We present a novel task of out-of-pose detection, and provide a new dataset, termed Dash-Cam-Pose. 
\item We provide experiments on Dash-Cam-Pose, three standard vision datasets, and five UCI datasets, demonstrating substantial performance benefits of our methods.
\end{itemize}

We note that this work is an extension of the ICCV conference paper~\cite{GODS} and extends it in the following ways:  (i) we provide novel extensions to the GODS using different assumptions on the complementary classifiers (Section~\ref{gods_extension}), (ii) we derive a kernelized variant of GODS (Section~\ref{sec:kgods}) and provide practical simplifications for optimization on the generalized Stiefel manifold (Section~\ref{sec:kods_optimization}), (iii) we provide adaptive classifiation rules in Section~\ref{sec:classification}, and (iv) provide additional experiments and ablative studies, including new results on UCSD Ped2 and UCI datasets.

\section{Related Work}
\label{sec:related_work}
As one-class problems arise in numerous practical settings, they have been explored to great depth in a variety of disciplines, including but not limited to remote sensing~\cite{matteoli2014overview}, network intruder detection~\cite{lazarevic2003comparative}, and fraud detection~\cite{heller2003one}. In computer vision, a few illustrative problems are novelty detection~\cite{gardner2006one,khan2009survey,sabokrou2018adversarially}, video anomaly detection~\cite{del2016discriminative,popoola2012video,saligrama2012video}, diagnosis on medical images~\cite{krawczyk2015one}, and anomalous object attribute recognition in image collections~\cite{saleh2013object}. For a comprehensive review of applications, we refer the interested reader to surveys, such as~\cite{chandala2009anomaly,pimentel2014review,pang2020deep}. 
\\
\noindent\textbf{Classic methods} for modelling such one-class problems are extensions of data density estimation techniques~\cite{tsybakov1997nonparametric,nolan1991excess}. These methods attempt to model the density of the given data in the input space by trading-off between maximizing their inclusivity within a (given) density quantile while minimizing its volume. The dependence on minimal density volume in the input space is discarded in Scholkopf et al.~\cite{scholkopf2001estimating,tax2004support} against smoothness of the decision function in a non-linear (kernelized) feature space. Working in the kernel space not only allows for more flexible characterizations of the distribution of the data samples, but also allows transferring the max-margin machinery (and the associated theory) developed for support vector machines to be directly used in the one-class setting. However, as alluded to earlier, working with kernel feature maps can be demanding in large data settings, and thus our main focus in this paper is on deriving one-class algorithms in the input (primal) space. That said, we also explore a kernelized dual variant of our scheme for problems that are impossible to be modelled using our primal variant. 

Modern efforts to one-class learning typically use either (i) good hand-crafted representations combined with effective statistical learning models, or (ii) implicit representation and learning via neural networks. Below, we review these efforts in detail.

\noindent\textbf{Explicit Modelling Approaches.} Performance of any one-class approach inevitably depends on the effectiveness of the data representation. For example, visual representations such as histogram of oriented gradients~\cite{dalal2005histograms} (HOG) and histogram of optical flows (HOF)~\cite{dalal2006human} have been beneficial in developing several anomaly detection algorithms. A Markov random filed (MRF) on HOG and HOF descriptors is proposed in Zhang et al.~\cite{zhang2005semi} for modeling the normal patterns in a semi-supervised manner, where an abnormal sample model is iteratively derived from the normal one using Bayesian adaptation. In Xu and Caramanis~\cite{xu2010robust}, an outlier pursuit algorithm is proposed using convex optimization for the robust PCA problem. Similarly, Kim et al.~\cite{kim2009observe} propose a space-time MRF to detect abnormal activities in videos. This method uses a mixture of probabilistic principal component analysis to characterize the distribution of normal data characterized as densities on optical flow. Detecting out-of-context objects is explored in~\cite{choi2012context, park2012abnormal} using support graph and generative models. Motion trajectory analysis~\cite{wu2010chaotic,calderara2011detecting,tung2011goal} of objects in video sequences has been a common approach for modeling anomalies, under the strong assumption that deviant trajectories may correspond to abnormal data. Detecting salient regions in images has also been implemented by some researchers~\cite{itti2000saliency,judd2009learning}. In contrast to these approaches that propose problem-specific anomaly detection models, our solution is for a general setting.

Sparse reconstruction analysis has been a powerful workhorse in the recent times in developing several one-class solutions~\cite{cong2011sparse,lu2013abnormal,zhao2011online,li2013visual}. The assumption in these methods is that normal data can be encoded as sparse linear combinations of columns in a dictionary that is learned only on the normal data; however the reconstruction error of any out-of-distribution sample using this dictionary could be significant. In addition to the reconstruction loss, a study from Ren et al.~\cite{ren2016comprehensive} points out that the sparsity term should be taken into consideration for improving the anomaly detection accuracy. However, sparse reconstruction methods can be computationally expensive.  To improve their efficiency, Bin et al.~\cite{zhao2011online} proposes an online detection scheme using sparse reconstructibility of query signals from an atomically learned event dictionary. Yang et al.~\cite{cong2011sparse} improves efficiency via learning multiple small dictionaries to encode image patterns at multiple scales. While, our proposed GODS algorithm could also be treated as a dictionary of orthonormal columns, our inference is significantly cheaper in contrast to solving $\ell_1$-regularized problems in these works as GODS involves only evaluations of inner products of the data samples to the learned hyperplanes. 

\noindent\textbf{Deep Learning Based Methods.} The huge success of deep neural networks on several fundamental problems in computer vision~\cite{deng2009imagenet,girshick2015fast} has also casted its impact in devising schemes for anomaly detection~\cite{hasan2016learning}. Extending classical methods, a deep variant of SVDD is proposed in Ruff et al.~\cite{ruff2018deep}, however assumes the one-class data is unimodal. Variants of OC-SVM are explored in~\cite{chalapathy2018anomaly,xu2017detecting}. Parera and Patel~\cite{perera2018learning} proposes a trade-off between compactness and descriptiveness using an external reference dataset to train a deep model on one-class data. Liang et al.~\cite{liang2017enhancing} proposes to use statistical trends in the softmax predictions. Deep learning based feature representations have been used as replacements for hand-crafted features in several one-class problem settings. For example, Xu et al.~\cite{xu2015learning} design a multi-layer auto-encoder embracing data-driven feature learning. Similarly, Hasan et al.~\cite{hasan2016learning} propose a 3D convolutional auto-encoder to capture both spatial and temporal cues in video anomaly detection. Leveraging the success of convolutional neural networks (CNNs) to capture spatial cues,~\cite{chong2017abnormal,luo2017remembering}, and~\cite{luo2017revisit} propose to embed recurrent networks, such as LSTMs, to model the appearance dynamics of normal data. In~\cite{lee2018simple} and  \cite{principled_srikant}, frameworks to minimize the in-distribution sample distances is proposed thereby maximizing the distance to out-of-distribution samples. In Sabokrou et al.~\cite{sabokrou1609fully},  a pre-trained CNN is used for extracting region features, and a cascaded outlier detection scheme is applied. Multiple instance learning (MIL) in a deep learning setting is attempted in~\cite{sultani2018real} for anomaly detection using weakly-labeled training videos via applying an MIL ranking loss with sparsity and smoothness constraints; however includes both normal and abnormal samples in the training set. Deep generative adversarial networks (GAN) have also been proposed to characterize the single class~\cite{ravanbakhsh2017abnormal,sabokrou2018adversarially,schlegl2017unsupervised,ravanbakhsh2019training}. These methods typically follow the same philosophy of training auto-encoders, however uses an adversarial discriminator to improve quality of the decoded data sample; the discriminator confidence is then used as an abnormality cue during inference.

In contrast to these approaches, we focus on explicit modelling of one-class data distributions, allowing better and more controlled characterization of the single class. Deep learning approaches reviewed above are complimentary to our contributions; in fact we use deep-learned data representations in our experiments and simultaneously demonstrate our performances on non-deep-learned features as well. More recently, explicit characterization of the diversity of the normal data is explored in~\cite{park2020learning}, self-supervised learning is proposed in~\cite{mohseni2020self}, and normalizing flows for the task in~\cite{zisselman2020deep}. We note that using our formulations within a deep neural network, while straightforward using methods described in~\cite{li2020efficient,gould2016differentiating}, is currently beyond the focus of this paper.

\begin{figure*}[t]
\centering

\subfigure[BODS]{\label{fig:bods}\includegraphics[width=6cm,trim={5cm 6cm 7cm 3.5cm},clip]{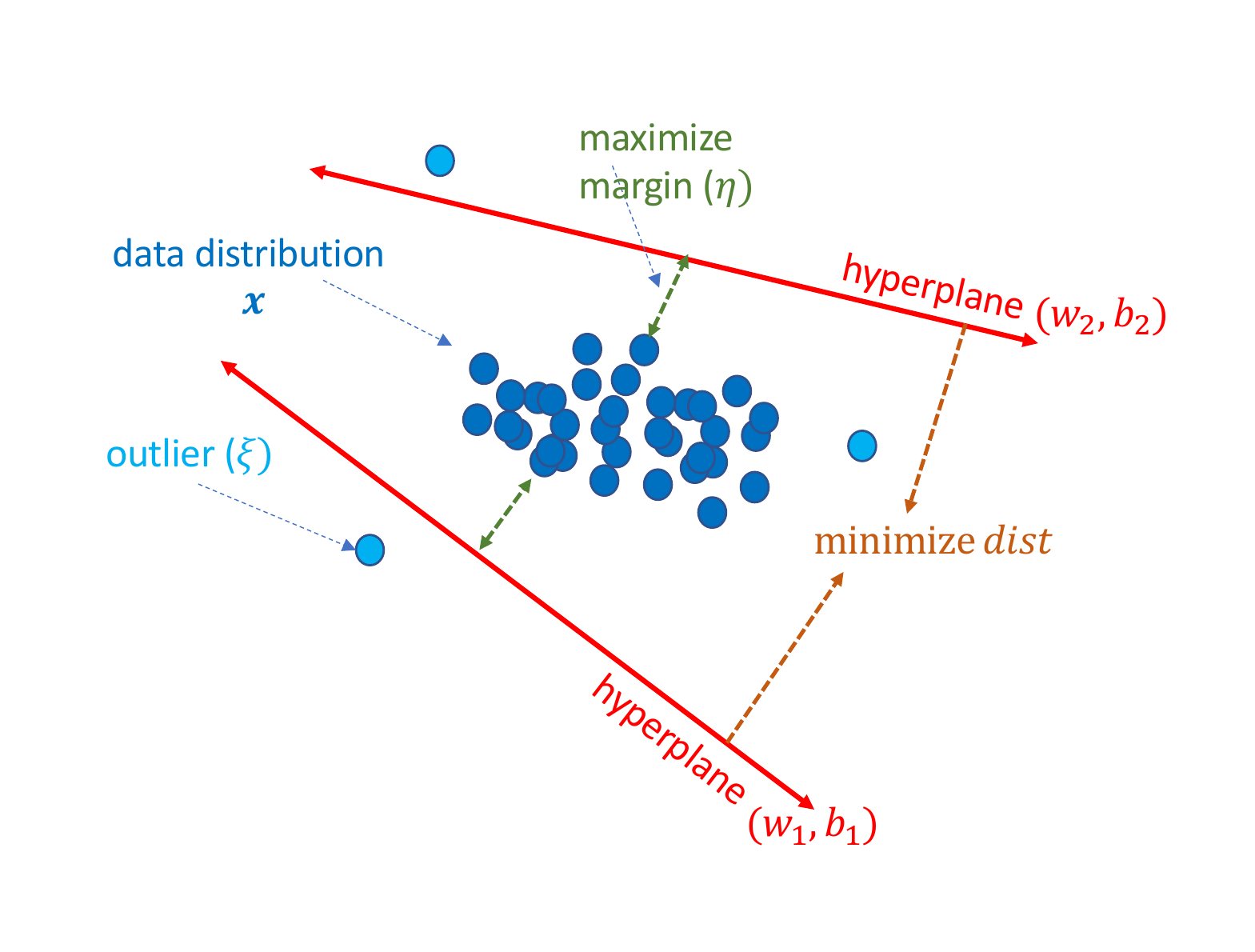}}
\subfigure[GODS]{\label{fig:gods}\includegraphics[width=6cm,trim={5cm 8cm 8cm 4cm},clip]{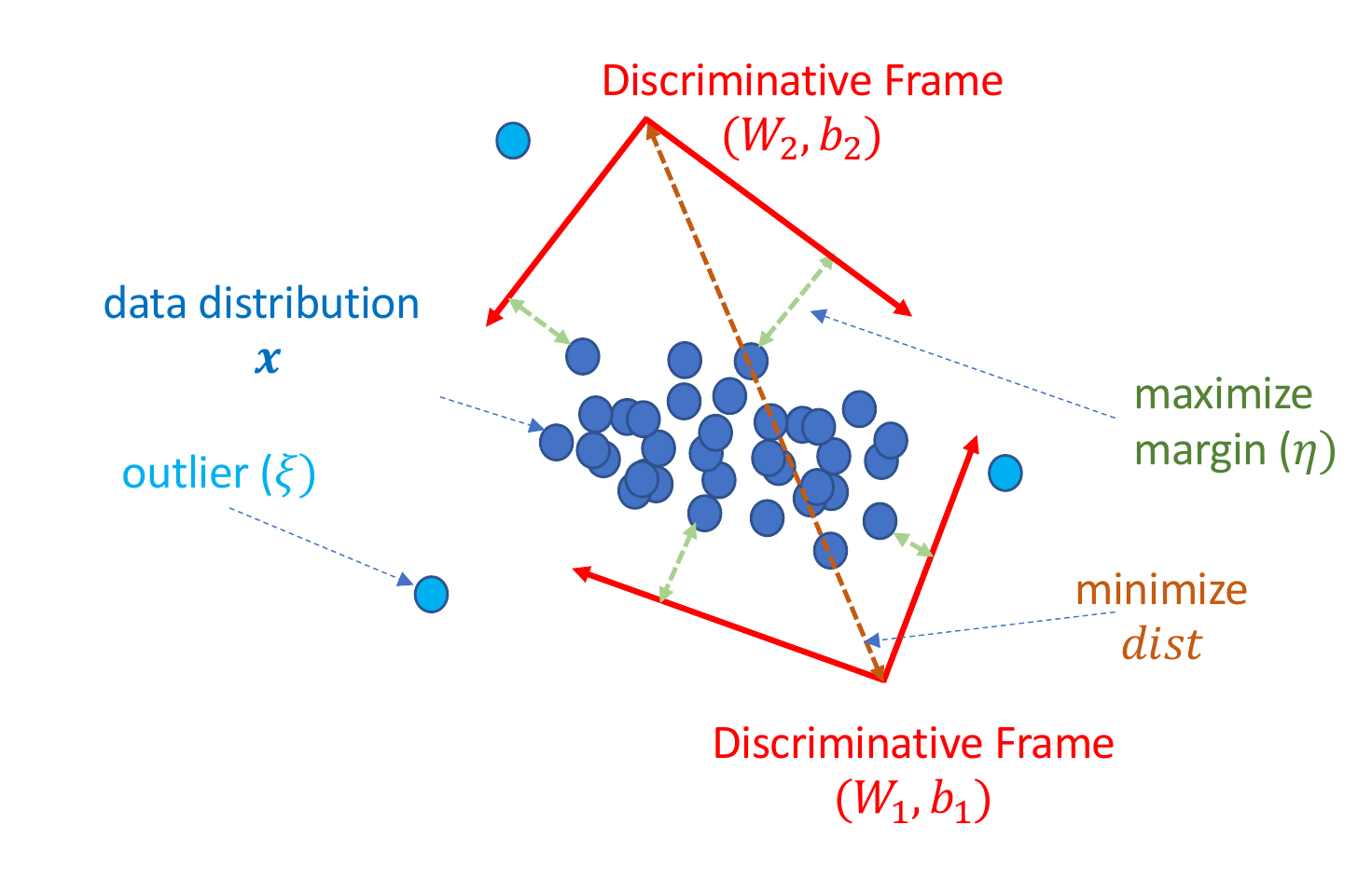}}
\subfigure[GODS in detail]{\label{fig:gods_detail}\includegraphics[width=6cm,trim={7cm 8.5cm 8cm 4cm},clip]{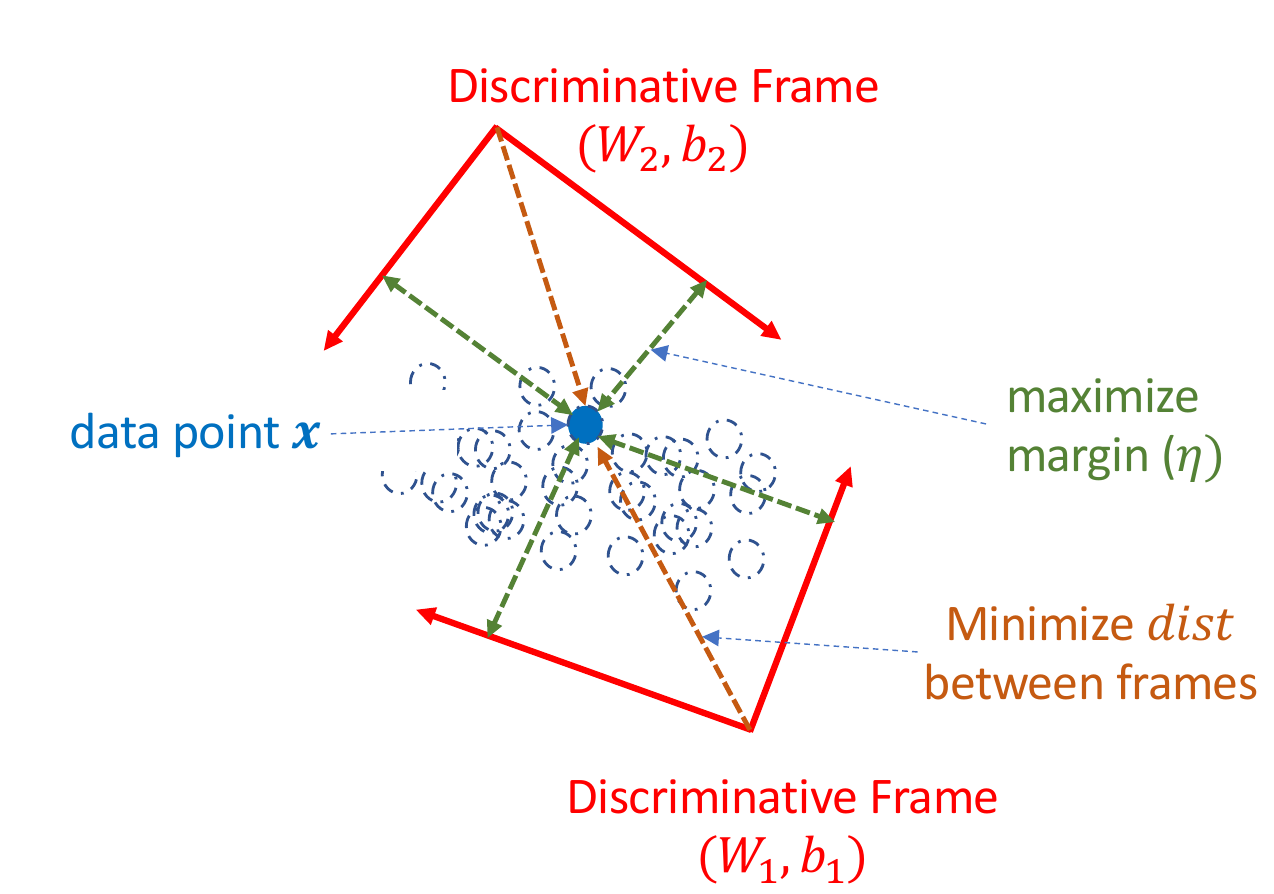}}
\caption{(a), (b) graphically illustrate our BODS and GODS formulations. Our one-class variants learn discriminative hyperplanes (BODS) or orthonormal frames (GODS) that maximize the margin with the data distribution while minimizing the volume of the one-class region captured (via minimizing the distance $\dist$ between hyperplanes or frames). In Figure~\ref{fig:gods_detail}, we show a detailed depiction of the objectives in GODS using a single data point $\vx$. See Section~\ref{sec:gods} for more details.  }
\label{fig:illustration}
\end{figure*}

\section{Background}
\label{sec:background}
Let $\dataset\subset\reals{d}$ denote the data distribution consisting of our one class-of-interest and everything outside it, denoted $\dbar$, be the anomaly set. Suppose we are given $n$ data instances $\dl=\set{\vx_1, \vx_2,\cdots, \vx_n} \subset\dataset$. The goal of one-class classifiers is to use $\dl$ to learn a functional $f$ which is positive on $\dataset$ and negative on $\dbar$. Typically, the label of $\dataset$ is assumed $+1$ and that of $\dbar$ is $-1$. We emphasize that we assume to have access \emph{only} to the one-class data for training our models; i.e., the set $\dbar$ is not available for training, and samples from $\dbar$ are used only at test time. 

In One-Class Support Vector Machine (OC-SVM)~\cite{scholkopf2001estimating}, $f$ is modeled as an extension of an SVM objective by learning a max-margin hyperplane that separates the origin from the data points in $\dl$. Mathematically, $f$ has the form $\sign(\vw^T\vx+b$), where $(\vw,b)\in\reals{d}\times\reals{1}$ and is learned by minimizing the following objective:
\begin{equation}
\min_{\vw,b,\xi\geq 0} \half\enorm{\vw}^2 - b+C\!\sum_{i=1}^n\xi_i,\subjectto \vw^T\vx_i +b +\xi_i\geq 0, \forall \vx_i\in\dl\notag,
\label{eq:oc-svm}
\end{equation}
where $\xi_i$'s are non-negative slacks, $b$ is the hyperplane intercept, and $C$ is the slack penalty.  As a single hyperplane in the input space might be insufficient to capture non-linear data, kernel feature maps are proposed in~\cite{scholkopf2001estimating}. 

Another popular variant of one-class classifiers is the support vector data description (SVDD)~\cite{tax2004support} that instead of modeling data to belong to an open half-space of $\reals{d}$ (as in the primal OC-SVM), assumes the labeled data inhabits a bounded set; specifically, the optimization seeks the centroid $\vc\in\reals{d}$ of a hypersphere of minimum radius $R>0$ that can contain all points (or a given quantile of points) in $\dl$. Mathematically, the objective reads:

\begin{equation}
\min_{\vc,R,\xi\geq 0} \half R^2\!+C\sum_{i=1}^n\xi_i, \subjectto \enorm{\vx_i-\vc}^2\!\leq\!R^2-\xi_i, \forall \vx_i \in \dl,\notag
\label{eq:svdd}
\end{equation}
where, as in OC-SVM, the $\xi$'s model the slack. As is discussed in~\cite{scholkopf2001estimating}, OC-SVM and SVDD are equivalent when the underlying kernel is translation invariant. There have been extensions of this scheme, such as the mSVDD that uses a mixture of hyperspheres~\cite{lai2015mixture}. Other variants include (i) the density-induced SVDD~\cite{lee2007density}, (ii) kernelized variants~\cite{tax2001one}, and (iii) more recently, those that use subspaces for data description~\cite{sohrab2018subspace}. A major drawback of SVDD in general is the strong assumption it makes on the isotropic nature of the underlying data distribution. Such a demand is ameliorated by combining OC-SVM with the idea of SVDD in least-squares one-class SVM (LS-OSVM)~\cite{choi2009least} that learns a tube around the discriminative hyperplane that contains the input; however, this scheme also makes strong assumptions on the data distribution (such as being cylindrical). 

Unlike OC-SVM that learns a compact data model to enclose as many training samples as possible, a different approach is to use principal component analysis (PCA) (or its varaints, such as Robust PCA\cite{candes2011robust,de2003framework,hoffmann2007kernel,nguyen2009robust,xu2013outlier}, and Kernel PCA) to summarize the data by using its principal subspaces. However, such an approach is usually unfavorable due to its high computational cost, especially when the dataset is large. Similar in motivation to the proposed technique, Bodesheim et al.~\cite{bodesheim2013kernel} uses a null space transform for novelty detection, while Liu et al.~\cite{liu2014unsupervised} optimizes a kernel-based max-margin objective for outlier removal and soft label assignment. However, their problem setups are different from ours in that \cite{bodesheim2013kernel} requires multi-class labels in the training data and \cite{liu2014unsupervised} is proposed for unsupervised learning.

In contrast to these prior methods, we explore the one-class learning problem from a distinct perspective; specifically, to use orthonormal frames as in PCA, however instead of approximating the one-class data using these frames, our objective seeks a pair of complementary orthonormal frames that bounds the data discriminatively and in a piece-wise linear manner, such that the data subspace is sandwiched between these frames. We first present a simplified variant of this idea using two complementary hyperplanes, dubbed Basic One-class Discriminative Subspaces (BODS) (Fig.~\ref{fig:bods}); these hyperplanes are independently parameterized and bounds the data distribution. Note that there is a similar prior work, termed Slab-SVM~\cite{fragoso2016one}, that learns two hyperplanes for one-class classification. However, their hyperplanes are constrained to have the same slope, which we do not impose in our BODS model; as a result, our model is more flexible than Slab-SVM. We extend BODS using multiple orthogonal hyperplanes, which we call Generalized One-class Discriminative Subspaces (GODS) (see Fig.~\ref{fig:gods}). The use of such discriminative subspaces has been recently explored in the context of representation learning on videos in Wang and Cherian~\cite{wang2018learning} and Wang et al.~\cite{wang2018video}, however requires a surrogate negative bag of features, that is found via adversarial learning.

\noindent\textbf{Notation:} We use bold-face upper-case for matrices, bold-face lower-case for vectors, non-bold-face lower case for scalars. We use $\mW$ to represent an orthonormal frame. We sometimes succinctly refer to a set of variables without their subscripts; e.g., $\mW=\set{\mW_1,\mW_2}$. The notation $\one_{p\times q}$ represents a matrix of ones with $p$ rows and $q$ columns.

\section{Proposed Method}
\label{sec:algo}
In this section, we formally introduce our schemes. First, we present BODS using a pair of hyperplanes, which we generalize to GODS using a pair of discriminative frames in Section~\ref{sec:gods}. We explore variants of GODS in Section~\ref{gods_extension} and further generalize GODS using kernel feature maps, proposing KODS in Section~\ref{sec:kgods}. With a slight abuse of terminology, we call our entire suite of formulations as GODS.
 
\subsection{Basic One-class Discriminative Subspaces}
\label{sec:bods}
In this section, we introduce a basic variant of our objective, which we call Basic One-class Discriminative Subspaces (BODS). The key goal of BODS is to bound the one-class data distribution using a pair of hyperplanes. Similar to OC-SVM, we seek these hyperplanes to have a large margin from the data distribution, thus allowing robustness to minor differences in the test data distribution. Further, inspired by SVDD, we also demand the two hyperplanes to bound the data within a minimal volume. BODS combines these two conflicting objectives into a joint formulation. Mathematically, suppose $(\vw_1, b_1)$ and $(\vw_2, b_2)$ define the parameters of the pair of hyperplanes. Our goal in BODS is then to minimize an objective such that all data points $\vx_i$ be classified to the positive half-space of $(\vw_1, b_1)$ and to the negative half-space of $(\vw_2, b_2)$, while also minimizing a suitable distance between the two hyperplanes (see Figure~\ref{fig:bods}). To this end, we propose the following BODS objective:

\begin{align}
\label{eq:0}\min_{\substack{(\vw_1,b_1), (\vw_2,b_2),\\ \xi_1,\xi_2\geq 0}}& \half\enorm{\vw_1}^2\!+\!\half\enorm{\vw_2}^2\!-\!b_1\!-\!b_2 + \vOmega(\xi_{1i},\xi_{2i})\notag\\
&\qquad\qquad+\half\dist\!\left(\left(\vw_1, b_1\right), \left(\vw_2,b_2\right)\right), \\
\label{eq:1}\subjectto &\left(\vw_1^T\vx_i + b_1\right) \geq \eta - \xi_{1i},\\
\label{eq:2}&\left(\vw_2^T\vx_i + b_2\right) \leq -\eta + \xi_{2i},
\end{align}

\noindent where the first four terms of~\eqref{eq:0} seek large margins similar to a standard OC-SVM objective. The key element of BODS is the last term in~\eqref{eq:0} that aims to minimize a suitable distance $\Dist$ between the two hyperplanes. In~\eqref{eq:1} and~\eqref{eq:2}, we capture the complementary classification constraints noted above. We use the notation $\vOmega(\xi_{1i},\xi_{2i})=C\sum_{i=1}^n\left(\xi^2_{1i}+\xi^2_{2i}\right)$ for the squared-slack regularization\footnote{We use squared-$\ell_2$ norm on the slacks instead of the standard $\ell_1$ norm to allow for smooth gradients in our optimization setup. While, this choice may hurt the sparsity of the slacks, it is usually not empirically seen to affect the classifier generalization~\cite{lee2001ssvm,mangasarian2001lagrangian}.}  and $\eta>0$ specifies a (given) classification margin. For BODS, we assume $\Dist$ is the Euclidean distance, i.e.,
$\dist\!\left(\left(\vw_1,b_1\right),\left(\vw_2,b_2\right)\right) = \enorm{\vw_1-\vw_2}^2 + (b_1-b_2)^2$. 

\subsubsection{BODS on the Unit Sphere}
\label{sec:unit_sphere}
It is often seen (especially for computer vision problems) that feature normalization, specifically unit normalizing of the inputs, demonstrate better performances \cite{jegou2009burstiness,koniusz2016higher,graf2003classification}. While, such a step may be counter-intuitive as we lose the discriminativeness in the scale of the input features, such normalization often allows counteracting the effects of \emph{burstiness}~\cite{jegou2009burstiness} -- a statistical phenomenon due to a non-uniform distribution of data elements. Such normalization is even found to improve the performances of deep-learned features when used in max-margin frameworks~\cite{ranjan2017l2,cherian2018non} and its importance to one-class tasks is ascertained in~\cite{tax2000feature}, in the context of SVDD. As the primary focus of this paper is in developing one-class solutions for vision problems, we decided to make this normalization as part of the pre-processing steps when generating our input features and thus in deriving our core formulations, (albeit we consider variants of our formulation without such assumptions in the next section). We empirically validate this assumption in Section~\ref{exp_result}.

Following this idea, we assume that our data is unit normalized, i.e., $\enorm{\vx_i}=1$, and thus belongs to a unit hypersphere  $\unitsphere{d-1}$, which is a sub-manifold of the Euclidean manifold $\reals{d}$. This assumption\footnote{Such an SVM formulation is classically known as \emph{normalized margin SVMs}, theoretically analyzed for its generalization performance in~\cite{herbrich2001pac} and~\cite{smola1998learning}[Sec. 10.6.2].} on the data naturally places our hyperplanes also to belong to $\unitsphere{d-1}$; i.e., $\enorm{\vw_1} = \enorm{\vw_2}=1$.  Using these manifold constraints, our BODS formulation can be rewritten  as follows: 
\begin{align}
&\label{eq:7}\problem_1:=\min_{\vw_1,\vw_2\in\unitsphere{d-1}, b_1,b_2}\half\valpha(b_1,b_2) -\vw_1^T\vw_2  \\
&+\frac{\nu}{2n}\sum_{i}\hinge{\eta-\left(\vw_1^T\vx_i+b_1\right)}^2+\hinge{\eta\!+\!\left(\vw_2^T\vx_i+b_2\right)}^2\notag,
\end{align}
where $\valpha(b_1,b_2)=(b_1-b_2)^2-2(b_1-b_2)$. We further simplify the BODS objective by substituting the constraints on the slacks $\xi$'s in~\eqref{eq:1} and~\eqref{eq:2} into $\vOmega$ in~\eqref{eq:0} and include them in the objective as soft constraints using the hinge loss $\hinge{\ \ }$. We use $\nu$ to denote a penalty factor on these soft constraints. In Fig.~\ref{fig:bods}, we illustrate the decision boundaries of the BODS model. 

While, BODS offers a simple and flexible model to capture the one-class distribution, the linearity of the classifiers may limit its applications to sophisticated data models. A natural idea is then to empower these classifiers to have non-linear decision boundaries. While, using a kernel method is perhaps a standard approach in this regard (which we present subsequently), we first propose to achieve non-linearity via piecewise linear decision boundaries. To this end, we equip each classifier in BODS with a set of hyperplanes; each set forming a complementary pair with the other. The use of piecewise linear decision boundaries makes inference computationally cheap as it requires only $2K$ inner products during inference, assuming $K$ hyperplanes per set. Further, we also avoid the need for computing kernel matrices, allowing for scalability of our approach to larger datasets. However, using sets of hyperplanes brings in the challenge of how to effectively regularize them to avoid overfitting and redundancy. To this end, in the following subsections, we generalize BODS to use pairs of multiple hyperplanes, regularized as orthonormal frames, thus providing a richer and non-linear discriminative setup, and subsequently present other regularizations and kernel embeddings.

\subsection{One-class Discriminative Subspaces}
\label{sec:gods}
Let us continue to assume the input data is unit normalized, we will remove this assumption in the next section. Formally, suppose $\mW_1,\mW_2\in\stiefel{K}$ be orthonormal frames -- that is, matrices of dimensions $d\times K$, each with $K$ columns where each column is orthonormal to the rest; i.e., $\mW_1^T\mW_1=\mW_2^T\mW_2=\eye{K}$, where $\eye{K}$ is the $K\times K$ identity matrix (see Fig.~\ref{fig:gods}). Such frames belong to the so-called Stiefel manifold, denoted~$\stiefel{K}$, with $K$ $d$-dimensional directions. Note that the orthogonality assumption on the $\mW_i$'s is to ensure they capture diverse discriminative directions, leading to better regularization, while also improving their characterization of the data distribution. A direct extension of $\problem_1$ then leads to:
\begin{align}
&\problem_2:=\min_{\substack{\mW\in\stiefel{K}, \vb}}\half\distW(\mW_1,\mW_2) +\valpha(\vb_1, \vb_2)\\
&\qquad\label{eq:10}+\frac{\nu}{2n}\sum_i\hinge{\eta-\min(\mW_1^T\vx_i+\vb_1)}^2\\
&\qquad\label{eq:11}+\frac{\nu}{2n}\sum_i\hinge{\eta + \max(\mW_2^T\vx_i+\vb_2)}^2,
\end{align}
where $\DistW$ is a suitable distance between orthonormal frames, and $\vb\in\reals{K}$ is a vector of biases, one for each hyperplane. Note that in~\eqref{eq:10} and~\eqref{eq:11},  unlike BODS, $\mW^T\vx_i+\vb$ is a $K$-dimensional vector. Thus,~\eqref{eq:10} says that the minimum value of this vector should be greater than $\eta$ and~\eqref{eq:11} says that the maximum value of it is less than $-\eta$. To simplify the notation, let us use $\zeta(\mW,\vb)=\valpha(\vb_1, \vb_2)+\eqref{eq:10}+\eqref{eq:11}$. Then, $\problem_2$ can be written as follows:

\begin{align}
    \problem_2':= \min_{\substack{\mW\in\stiefel{K}, \vb}} &-\trace{\mW_1^T\mW_2} +\zeta(\mW,\vb).
    \label{eq:26}
\end{align}

The formulation $\problem_2'$, due to the first term, enforces a \emph{tight coupling} between $\mW_1$ and $\mW_2$; such a  coupling might prevent the frames from freely aligning to the data distribution, resulting in sub-optimal performance. To circumvent this issue, we propose the following work around. Recall that the main motivation to define the distance between the frames is so that they sandwich the (one-class) data points compactly. Thus, rather than defining a distance between the frames directly, one could also use a measure that minimizes the Euclidean distance of each data point from both the hyperplanes; thereby achieving the same effect. Such a distance via the data points will also make the frames loosely coupled. More formally, we propose to redefine $\distW$ as:
\begin{equation}
\distW(\mW_1,\mW_2, \vb_1,\vb_2|\vx) = \half\sum_{j=1}^2\enorm{\mW_j^T\vx+\vb_j}^2,
\label{eq:12}
\end{equation}
where now we minimize the sum of the lengths of each $\vx$ after projecting on to the respective frames; thereby pulling both the frames closer to the data point. Using this definition of $\distW$, we formulate our \emph{generalized one-class discriminative subspace} (GODS) classifier as:
\begin{align}
&\label{Problem3}\gods:=\!\!\!\min_{\substack{\mW\in\stiefel{K}, b}} F= \frac{1}{2n}\sum_{i=1}^n \sum_{j=1}^2\enorm{\mW_j^T\vx_i+\vb_j}^2\\
&+\frac{\nu}{2n}\!\!\sum_i\!\hinge{\eta\!-\!\min(\mW_1^T\vx_i+\vb_1)}^2\!\!\!+\!\!\hinge{\eta\!+\! \max(\mW_2^T\vx_i+\vb_2)}^2\!\!\!\!.\notag
\end{align}
See Figure~\ref{fig:gods_detail} for a graphical illustration of this variant of our GODS objective.

\subsection{Extensions to GODS Formulation}
\label{gods_extension}
The technical development of GODS in the previous section assumes the input data is unit normalized, as otherwise the orthornormal frames for discriminating them may not be generally fruitful. Our other important assumption in the previous section -- that the hyperplanes in our discriminative decision parameters $\mW$ are orthonormal -- can be restrictive as well. In this section, we provide extensions of our GODS framework that relax or remove these restrictions. In the following variants, we will assume $\lambda>0$ to generically denote a penalty on the respective regularization.

\subsubsection{Non-Compact Stiefel Manifold}
A matrix $\mW\in\reals{d\times K}$ with its $K$ columns being linearly independent, however not unit normalized, belongs to the so-called non-compact Stiefel manifold (Absil et al.~\cite{absil2009optimization}[Chapter 3]), denoted $\reals{d\times K}_*$, which is an open subspace of the Euclidean space $\reals{d\times K}$. One may represent such a manifold as a product manifold between a $d\times K$ Stiefel manifold and a $K\times 1$ Euclidean vector; i.e., $\reals{d\times K}_*=\stiefel{K}\times \reals{K}$. For example, if $\mW\in\reals{d\times K}_*$, then $\mW={\mQ}\diag(\vr)$, where $\mQ\in\stiefel{K}$, $\vr\in\reals{K}$, and the $i$-th dimension $\vr_i=\enorm{\mW_{:,i}}$, the $\ell_2$ norm of the $i$-th column in $\mW$. For non-unit-norm input data, we can extend the GODS formulation in~\eqref{Problem3} using a non-compact Stiefel manifold as:
\begin{align}
    \gods_N :=\!\!\!\min_{(\mQ,\vr)\in\stiefel{K}\times \reals{K}\!, b} \!\!F\left(\mQ\diag(\vr), b\right) + \frac{\lambda}{2}\norm{\vr}_p,
    \label{eq:ncstiefel}
\end{align}
where $F$ is the objective in~\eqref{Problem3} and $\lambda>0$ is a penalty on the $\ell_p$-norm regularization over $\vr$. Note that in~\eqref{eq:ncstiefel}, for brevity we assume $(\mQ,\vr)=\set{(\mQ_i,\vr_i)}_{=1}^2$, i.e., it is technically a product of two non-compact Stiefel manifolds corresponding to the two discriminative frames. There is an additional advantage with the proposed subspace representation -- it can allow automatic selection of the number of subspace components one may need. For example, with the $\ell_p$ regularization on $\vr$, some of the dimensions in $\vr$ can go to zero (say using $p=1$), thereby removing the respective subspace component from the final representation. 

\subsubsection{Oblique Manifold}
We may also relax the orthogonality constraints on $\mW$, however maintain their unit normality. A set of matrices $\ob{K}=\set{\mW\in\reals{d\times K}: \diag(\mW^\top\mW) = \eye{K}}$ forms a regular submanifold of the Euclidean manifold and is usually called an Oblique manifold~\cite{absil2006joint} under the canonical inner product metric. This manifold is isometric to the product of $K$ spheres $\times_1^K\stiefel{1}$. We can rewrite a variant of GODS with the optimization on $\ob{K}$ as:
\begin{align}
    \gods_O :=\!\!\!\!\!\min_{\mW\in\ob{K},b}\!\!\!F(\mW, b)\!+\! \frac{\lambda}{2}\sum_{i=1}^2\fnorm{\mW_i^{\top}\mW_i\!-\!\eye{K}}^2,
    \label{eq:gods0}
\end{align}
the last term \emph{softly} controls the correlations among columns in $\mW$.

\subsubsection{Euclidean Manifold}
Removing both the orthogonality and the unit norm constraints on the classifiers and the input data results in our most general form of the GODS formulation, that assumes $\mW$ belongs to the Euclidean manifold. We can write such a variant as:
\begin{align}
    \gods_E := &\min_{\mW\in\reals{d\times K}, b} F(\mW, b) + \frac{\lambda}{2}\sum_{i=1}^2\fnorm{\mW_i^{\top}\mW_i-\eye{K}}^2.
    \label{eq:euc}
\end{align}
Similar to~\eqref{eq:gods0}, the last term in $\gods_E$ controls the correlations between the columns in $\mW$; a large $\lambda$ will promote $\mW$ to be similar to the original GODS formulation using the Stiefel manifold in~\eqref{Problem3}.

\subsection{Kernelized One-class Discriminative Subspaces}
\label{sec:kgods}
While, the classifier as described in our GODS formulation can offer computationally efficient, yet non-linear decision functions in the input space, it may fail in situations when input data cannot be bounded using rectilinear coordinates. For example, when the outliers form dense regions inside the one-class, or when the normal data forms islands, are xor-shaped, or ring-shaped. This lends a kernelized variant of GODS scheme inevitable.  

To derive kernelized GODS, we use our formulation\footnote{We attempted to use other GODS variants, however they resulted in objectives that seemed computational expensive.} in $\problem_2$, however expand the min and max constraints in~\eqref{eq:10} and~\eqref{eq:11} via propagating the $\eta$ to each of the $K$ hyperplanes. Such a simplification allows for a direct application of the Langrange multiplers to derive the dual. We also assume that there are no outliers in the one-class data provided. This allows us to simplify the expressions we derive.\footnote{We also note that the use of slacks brings in additional constraints, making our optimization setup computationally difficult.} With these simplifications, we rewrite our modified $\problem_2$ as:
\begin{align}
\label{eq:kgods-primal}
\min_{\substack{\mW\in\stiefel{K}, \vb}} & \half\fnorm{\mW_1-\mW_2}^2+\half\enorm{\vb_1-\vb_2}^2 \\
\subjectto & \eta - \left(\mW_{1j}^T\vx_i + \vb_1\right) \leq 0,\quad\forall j\in\indexset{K}, i\in\indexset{n}\notag\\
    &\eta + \left(\mW_{2j}^T\vx_i +\vb_2\right) \leq 0,\quad\forall j\in\indexset{K}, i\in\indexset{n}\notag.
\end{align}
Using the fact that $\mW^{\top}\mW=\eye{K}$, and using non-negative dual variables $\mY,\mZ\in\reals{K\times n}_+$, we have the following Lagrangian formulation of~\eqref{eq:kgods-primal}:
\begin{align}
    L(\mW,\vb,\mY,\mZ) &:= -\trace{\mW_1^\top\mW_2} + \half\enorm{\vb_1-\vb_2}^2 +\notag\\
    & \trace{\mY^\top\left(\eta\one_{K\times n}-\mW_1^{\top}\mX - \vb_1\one_n^{\top}\right)}+\notag\\
    &\trace{\mZ^{\top}\left(\mW_2^{\top}\mX + \vb_2\one_n^\top+\eta\one_{K\times n}\right)}.
\end{align}
A straightforward reduction provides the following dual:
\begin{align}
    \min_{\mY,\mZ\in\reals{K\times n}_+}  \half\one_{n}^\top\mY^{\top}&\mY\one_{n}\!\!+\!\!\trace{\mY\kernel\mZ^T}
   \!-\!\eta\trace{\!\left(\mY\!\!+\!\!\mZ\right)^{\top}\!\!\one_{K\times n}}\!\!\notag\\
    \label{eq:sumcons}\text{s. t. } &(\mY - \mZ)\one_n = \zero\\
    &\mY\kernel\mY^{\top} = \mZ\kernel\mZ^{\top} = \eye{K},\label{eq:kgods-3}
\end{align}

\noindent where $\kernel=\mX^{\top}\mX$ is a linear kernel, however could be replaced by any other positive definite kernel via the kernel trick. We call our formulation above as \emph{kernelized one-class discriminative subspaces} (KODS). Recall that, for $\kernel\in\reals{n\times n}\succ 0$, the constraints in~\eqref{eq:kgods-3} pose the KODS objective on the generalized Stiefel manifold $\gs{K\times n}$, formally defined as $\gs{K\times n} = \set{\Lambda\in\reals{K\times n}: \Lambda\kernel\Lambda^{\top}=\eye{K}, \kernel\succ 0}$.  However, there are two constraints in our objective that adds hurdle to directly using this manifold for optimization: (i) the null-space constraint in~\eqref{eq:sumcons}, and (ii) the requirement that the dual variables $\mY, \mZ$ are non-negative. Below, we present soft-constraints circumventing these challenges and derive an approximate KODS objective.

\subsubsection{Approximate KODS Formulation}
We avoid the null-space constraint in KODS via incorporating~\eqref{eq:sumcons} as a soft-constraint into the KODS objective using a regularization penalty, $\lambda>0$. To circumvent the non-negative constraints, we replace the dual variables $\mY,\mZ$ by their element-wise squares, e.g., $\mY\in\reals{K\times n}: \mY\to\mYY$, while retaining $\mY\in\gs{K\times n}$. Note that the latter heuristic has been used before, such as in approximating quadratic assignment problems~\cite{wen2013feasible}. With these changes, we provide our \emph{approximate KODS formulation} as:
\begin{align}
      \min_{\mY,\mZ\in\gs{K\times n}}&\!\!\!\mathcal{K}(\mY_,\mZ) = \half\one_{n}^\top\mYY^{\top}\mYY\one_{n} \notag\\
        &+ \trace{\mYY\kernel\mZZ^T} \notag\\
        &-\eta\trace{\left(\mYY+\mZZ\right)^{\top}\!\!\one_{K\times n}}\notag\\
        &+\frac{\lambda}{2} \enorm{(\mYY-\mZZ)\one}^2,\label{eq:kgods-5}
\end{align}
where the last factor corresponds to~\eqref{eq:sumcons}.
Note that we use the squared form only on the optimization variables, and not on the constraints, and thus our objective is still on the generalized Stiefel product manifold. 

To derive the classification rules at test time (in the next section), we will need expressions for the primal variables in terms of the duals, which we provide below:
\begin{align}
    &\mW_1(.) = \mZZ\kernel(\mX, .)\\
    & \mW_2(.) = -\mYY\kernel(\mX,.)\\
     &\vb_1 = \rowmax\left(\eta - \mZZ\kernel\right)\\
     &\vb_2 = \rowmin\left(-\eta+\mYY\kernel\right), 
\end{align}
where $\rowmax$ and $\rowmin$ corresponds to the maximum and minimum values along the rows of the respective matrices.

\section{One-class Classification}
\label{sec:classification}
During inference, we use the decision functions with the learned parameters to classify a given data point as in-class or out-of-class. Specifically, for a new data point $\vx$, it is classified as in-class if the following criteria is met:
\begin{equation}
    \min(\mW_1(\vx) + \vb_1) \geq \eta \wedge \max(\mW_2(\vx)+\vb_2) \leq -\eta,
    \label{eq:decision}
\end{equation}
where the variables $\mW$ and $\vb$ are either learned in the KODS formalism or the GODS. In case, we have access to a validation set consisting of in-class and out-of-class data (for which we know the class labels), then we may calibrate the threshold $\eta$ to improve our decision rules. Specifically, suppose we have access to $m$ such validation data points, denoted $\dataset_{v}$. Then, to estimate an updated threshold $\eta'$, we propose to compute the decision scores $\sv_l=\set{\min(\mW_1(\vx)+\vb_1)}_{\vx\in\dataset_v}$ and $\sv_u=\set{\max(\mW_2(\vx)+\vb_2)}_{\vx\in\dataset_v}$. Next, we apply K-Means (or spectral clustering) on $\sv_l$ and $\sv_u$ with $K=2$ clusters. Suppose $c_{lk}$ and $c_{uk}$ ($k=1,2$) are the respective centroids for the two clustering problems; then we propose to update $\eta'$ as the average of the smaller of the two centroids thresholded by $\eta$; i.e.,
\begin{equation}
    \Delta \eta=\half\left(\hinge{\eta-\min(c_{l1},c_{l2})}\!\!-\!\hinge{\eta+ \min(c_{u1},c_{u2})}\right),
\end{equation}
and use $\eta'=\eta+\Delta\eta$ to form the new decision rules in~\eqref{eq:decision}.

\section{GODS Optimization}
In contrast to OC-SVM and SVDD, the $\gods$ formulation in~\eqref{Problem3} is non-convex due to the orthogonality constraints on $\mW_1$ and $\mW_2$.\footnote{Note that the function $\max(0, \min(z))$ for $z$ in some convex set is also non-convex.} However, these constraints naturally impose a geometry on the solution space  and in our case, puts optimization on the Stiefel manifold~\cite{muirhead2009aspects} -- a Riemannian manifold characterizing the space of all orthogonal frames. There exist several schemes for geometric optimization over Riemannian manifolds (see~\cite{absil2009optimization} for a detailed survey) from which we use the Riemannian conjugate gradient (RCG) scheme in this paper, due to its stable and fast convergence. In the following, we review some essential components of the RCG scheme and provide the necessary formulae for using it to solve our objectives.

\subsection{Riemannian Conjugate Gradient}
Recall that the standard (Euclidean) conjugate gradient (CG) method~\cite{absil2009optimization}[Sec.8.3] is a variant of the steepest descent method, however chooses its descent along directions conjugate to previous descent directions with respect to the parameters of the objective. Formally, suppose $F(\mW)$ represents our objective.\footnote{The other optimization variable -- $\vb$, belongs to the Euclidean manifold, and thus $(\mW,\vb)\in\stiefel{K}\times \reals{K}$. However for brevity and focus, we omit these variables from our optimization discussion.} Then, the CG method uses the following recurrence at the $k$-th iteration:
\begin{equation}
    \mW^{k} = \mW^{k-1} + \lambda^{k-1} \alpha^{k-1},
    \label{eq:recc}
\end{equation}
where $\lambda$ is a suitable step-size (found using line-search) and $\alpha^{k-1}=-\grad F(\mW^{k-1}) + \mu^{k-1} \alpha^{k-2}$, where $\grad F(\mW^{k-1})$ defines the gradient of $F$ at $\mW^{k-1}$ and $\alpha^{k-1}$ is a direction built over the current residual, which is conjugate to previous descent directions (see~\cite{absil2009optimization}[pp.182])). 

When $\mW$ belongs to a curved Riemannian manifold, we may use the same recurrence, however there are a few important differences from the Euclidean CG case, namely (i) we need to ensure that the updated point $\mW^k$ belongs to the manifold, (ii) there exists efficient vector transports\footnote{This is required for computing $\alpha_{k-1}$ that involves the sum of two terms in potentially different tangent spaces, which would need vector transport for moving between them; see~\cite{absil2009optimization}[pp.182].} for computing $\alpha^{k-1}$, and (iii) the gradient $\grad$ is along tangent spaces to the manifold. For (i) and (ii), we may resort to computationally efficient retractions (using QR factorizations; see~\cite{absil2009optimization}[Ex.4.1.2]) and vector transports~\cite{absil2009optimization}[pp.182], respectively. For (iii), there exist standard ways that take as input a Euclidean gradient of the objective (i.e., assuming no manifold constraints exist), and maps them to the Riemannian gradients~\cite{absil2009optimization}[Chap.3]. Specifically, for the Stiefel manifold, let $\nabla_{\mW} F(\mW)$ define the Euclidean gradient of $F$ (without the manifold constraints), then the Riemannian gradient is given by:
\begin{equation}
    \grad F(\mW) = \nabla_{\mW} F(\mW) - \mW\nabla_{\mW} F(\mW)^{\top}\mW .
    \label{eq:grad}
\end{equation}
The direction $\grad F(\mW)$ corresponds to a curve along the manifold, descending along which ensures the optimization objective is decreased (atleast locally).

Now, getting back to our one-class objective, all we need to derive to use the RCG, is compute the Euclidean gradients $\nabla_\mW F(\mW)$ of our objective in $\gods$ with regard to the variables $\mW_j$s. The other variables, such as the biases and slacks, belong to the Euclidean space and their gradients are straightforward. The expression for the Euclidean gradient of our objective with respect to the $\mW$'s is given by: 

\begin{align}
\label{eq:grad_1}\frac{\partial F}{\partial \mW_1} &= \sum_{i=1}^n\vx_i\left(\mW_1^T\vx_i+\vb_1\right)^T-\mZ_{k_i^*}\hinge{\eta-\mW_{1,k_i^*}^T\vx_{i}-b_1},\\
\label{eq:grad_2}\frac{\partial F}{\partial \mW_2} &= \sum_{i=1}^n\vx_i\left(\mW_2^T\vx_i+\vb_2\right)^T +\mZ_{k_i^*}\!\!\hinge{\eta+\mW_{2,k_i^*}^T\vx_{i}+b_2},
\end{align}
where $k_i^*\in\indexset{K}$ denotes the hyperplane index for the respective subspaces; $k_i^*=\argmin_k(\mW_1^T\vx_i+\vb_1)$ for~\eqref{eq:grad_1} and  $k_i^*=\argmax_k(\mW_2^T\vx_i+\vb_2)$ for~\eqref{eq:grad_2}. The variable $\mZ_{k_i^*}$ is a $d\times K$ matrix with all zeros, except $k_i^*$-th column, which is set to $\vx_i$.

\subsection{KODS Optimization}
\label{sec:kods_optimization}
In this section, we will derive the gradients for our approximate KODS formulation provided in~\eqref{eq:kgods-5}. Similar to~\eqref{eq:grad}, the mapping from the Euclidean gradient to the Riemannian gradient for the generalized Stiefel manifold is provided in the following theorem.
\begin{theorem}
 For the optimization problem $\min_{\mU} \fK(\mU)$ $\text{ s.t. } \mU\kernel\mU^{\top}=\eye{K},\ \kernel\succ 0$, if $\nabla_{\mU}\fK(\mU)$ denotes the Euclidean gradient of $\fK(\mU)$, then the Riemannian gradient under the canonical metric is given by: 
 \begin{equation}
     \grad \fK(\mU) = \nabla_{\mU}\fK(\mU)\kernel^{-1} - \mU\nabla_{\mU}\fK(\mU)^{\top}\mU.
 \end{equation}
\end{theorem}
\begin{proof}
The proof follows directly from the results in ~\cite{edelman1998geometry}[Section 4.5].
\end{proof}
For the retraction of the iterates on to the manifold, we use the generalized polar decomposition~\cite{higham2010canonical} as suggested in~\cite{boumal2014manopt}. As in the previous section, next we derive the expressions for the Euclidean gradients.
\begin{proposition}
Let $f(\mY)$ be a differentiable matrix function, then the matrix gradient $\nabla_{\mY^{\top}} f(\mY) = \left(\nabla_{\mY} f(\mY)\right)^{\top}$.
\label{prop:1}
\end{proposition}
\begin{lemma}
For matrices $\mY,\mA,\text{ and } \mD$ of appropriate sizes, if $f(\mY)=\trace{\left(\mY\odot\mA\right)\mD}$, then $\nabla_{\mY}f(\mY)\!=\!\mA\odot\mD^T$.
\begin{proof}
If $\va_{i:},\vy_{i:},\vd_{:j}$ represent the $i$-th row and $j$-th column of matrices $\mA,\mY,\mD$ respectively, then 
\begin{align}
    f(\mY) &= \sum_{i}(\va_{i:}\odot\vy_{i:})^{\top}\vd_{:i} = \sum_{i}\vy_{i:}^{\top}\left(\va_{i:}\odot\vd_{:i}\right).
\end{align}
Then, the gradient w.r.t. $\vy_{ij}$, i.e., $\nabla_{\vy_{ij}} f(\mY)=\va_{ij}\odot\vd_{ji}$, and we have the desired result.
\end{proof}
\label{thm:1}
\end{lemma}

\begin{lemma}
Let $f(\mY)=\trace{\mYY^{\top}\mYY\mD}$, where $\mD$ is a symmetric matrix. Then, $\nabla_{\mY} f(\mY) = 4\mY\odot\mYY\mD$.
\begin{proof}
To simplify the notation, let us use $\mA=\mYY$, then using Proposition~\ref{prop:1}, and applying chain-rule:
\begin{align}
\nabla_{\mY} f(\mY) &= 2\left(\nabla_{\bullet^\top}\!\!\trace{(\bullet^{\!\!\top}\!\!\odot\!\!\mY^{\!\!\top})\mA\mD}\right)^{\!\!\top}\!\!\notag\\
&\ \ \ \ \ \ \ \ \ \ \ \ \ \ \ \ \ \ \ \ \ \ \ \ \ \ \ \ \ +2\nabla_{\bullet}\!\!\trace{\!\mA^\top\!\left(\bullet\!\odot\!\mY\right)\!\mD}\notag\\
                    &= 2\left(\mY^{\top}\odot(\mA\mD)^\top\right)^{\top} + 2\mY\odot\left(\mD\mA^\top\right)^{\top}\label{eq:kgods-7}\\  
                    &= 2\mY\odot\left(\mA\mD\right) + 2\mY\odot(\mA\mD^\top)\notag,
\end{align}
where we used Lemma~\ref{thm:1} to obtain~\eqref{eq:kgods-7}. Using the symmetry of $\mD$, we have the result.
\end{proof}
\label{thm:2}
\end{lemma}
\begin{theorem}
 Let $\mE_n=\one_n\one_n^{\top}$, $\sqY=\mY\odot\mY$, and $\sqZ=\mZ\odot\mZ$, the Euclidean gradient of $\kods(\mY,\mZ)$ in~\eqref{eq:kgods-5} is:
\begin{align}
\nabla_{\mZ} \kods(\mY,\mZ) &= 2\lambda \mZ\odot\sqZ\mE_n + \mZ\odot\sqY\left[2\kernel-\lambda\mE_n\right]-2\eta\mZ\notag\\
\nabla_{\mY} \kods(\mY,\mZ) &= \gamma \mY\odot\sqY\mE_n + \mY\odot\sqZ\left[2\kernel-\lambda\mE_n\right]-2\eta\mY\notag,
\end{align}
where $\gamma=2+2\lambda$.
\begin{proof}
The result directly follows by applying Lemma~\ref{thm:1} and Lemma~\ref{thm:2} to the formulation in~\eqref{eq:kgods-5}.
\end{proof}
\end{theorem}

\subsection{Optimization Initialization}
\label{sec:init}
Due to the non-convexity of our objective, there could be multiple local solutions. To this end, we resort to the following initialization of our optimization variables, which we found to be empirically beneficial. Specifically, for the GODS optimization, we first sort all the training points based on their Euclidean distances from the origin. Next, we randomly select a suitable number ($3\times \#$ hyperplanes in our experiments) of such sorted points near and far from the origin, compute a compact singular value decomposition (thin-SVD) of these points, and initialize the GODS subspaces using these orthonormal frames from the SVD. The intercepts ($\vb$) are initialized to zero. For KODS, we initialize the dual variables as $\frac{1}{nK}$, where $K$ is the number of hyperplanes. 

\section{Experiments}
\label{sec:expts}
In this section, we provide experiments demonstrating the performance of our proposed schemes on several one-class tasks. We will introduce these tasks and the associated datasets briefly next along with detailing the data features used. 

\begin{figure}[t]
\centering
\includegraphics[width=3.0cm, height=3.0cm]{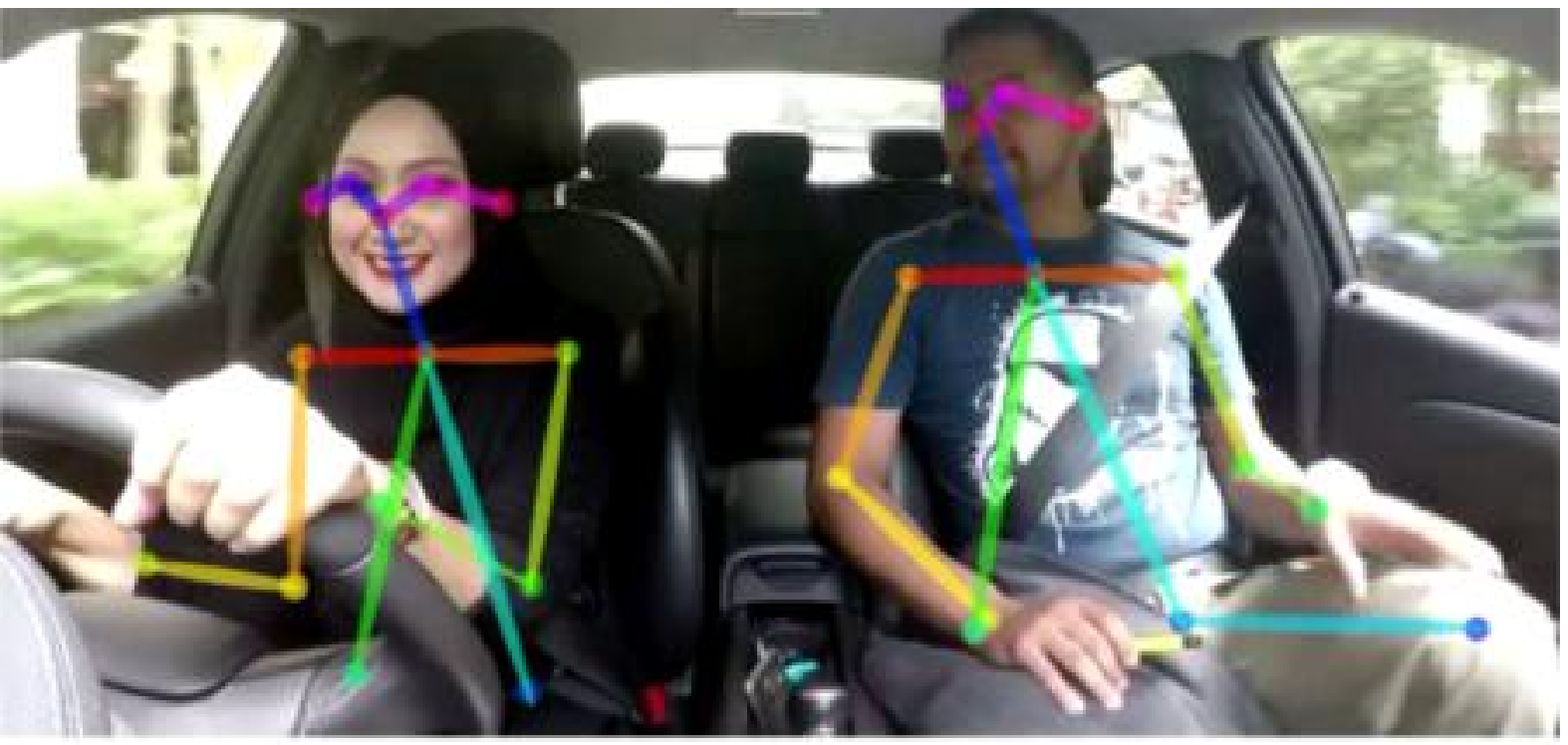}
\includegraphics[width=3.0cm, height=3.0cm]{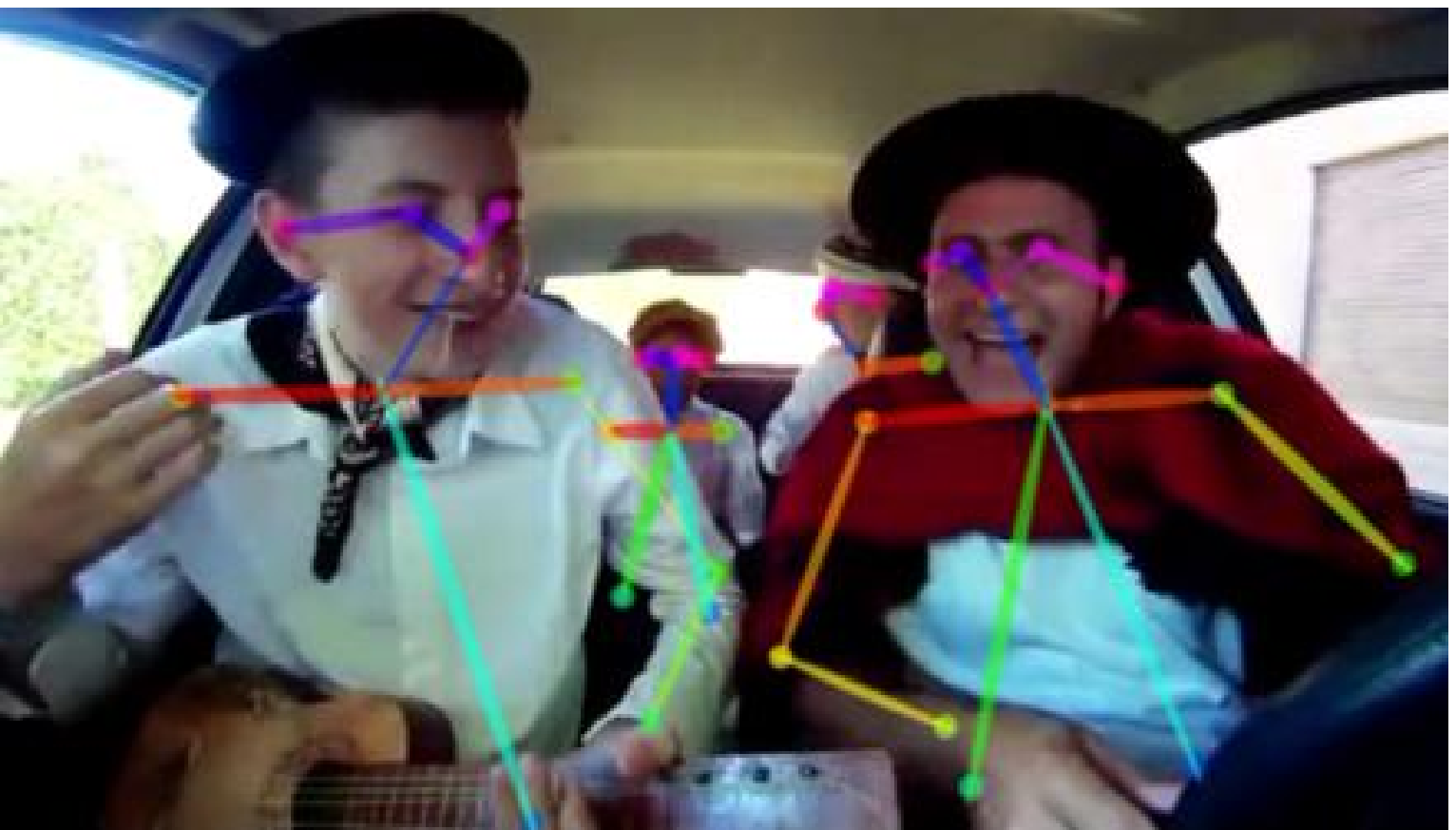}
\includegraphics[width=3.0cm, height=3.0cm]{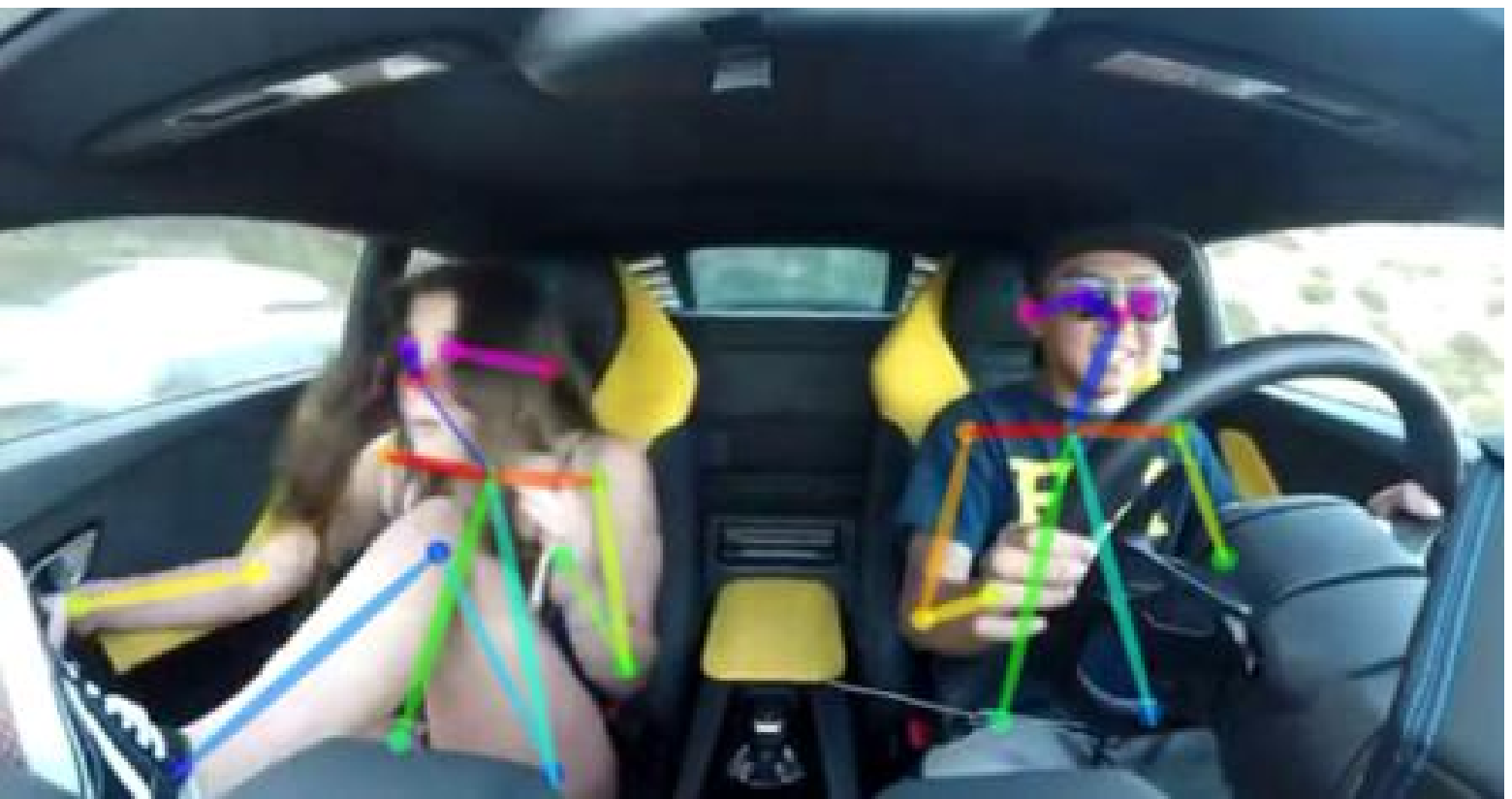}
\includegraphics[width=3.0cm, height=3.0cm]{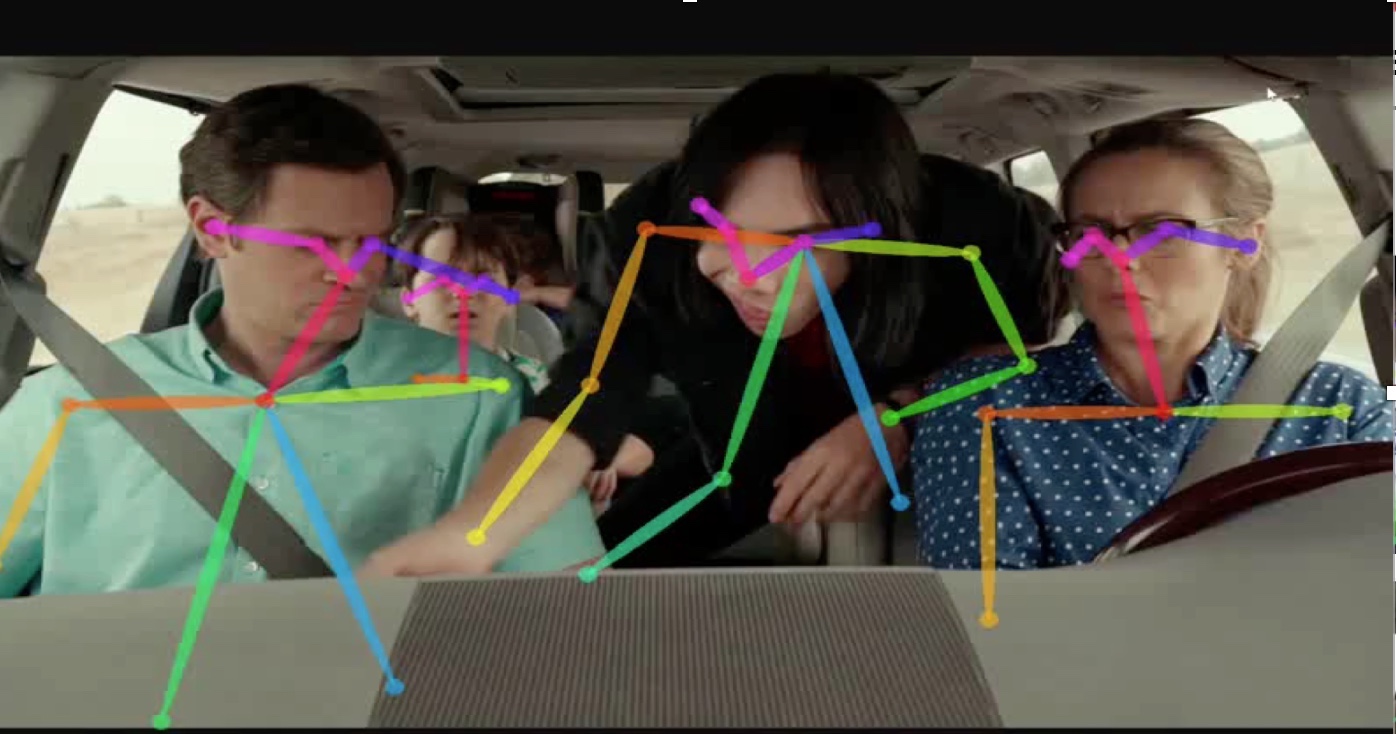}

\caption{Frames from our Dash-Cam-Pose dataset. The left-top frame has poses in-position (one-class), while the rest of the frames are from videos labeled out-of-position.}
\label{fig:dash-cam-pose}
\end{figure}

\begin{figure}[t]
\centering
\includegraphics[width=0.7\linewidth,trim={0cm 0cm 0cm 0cm},clip]{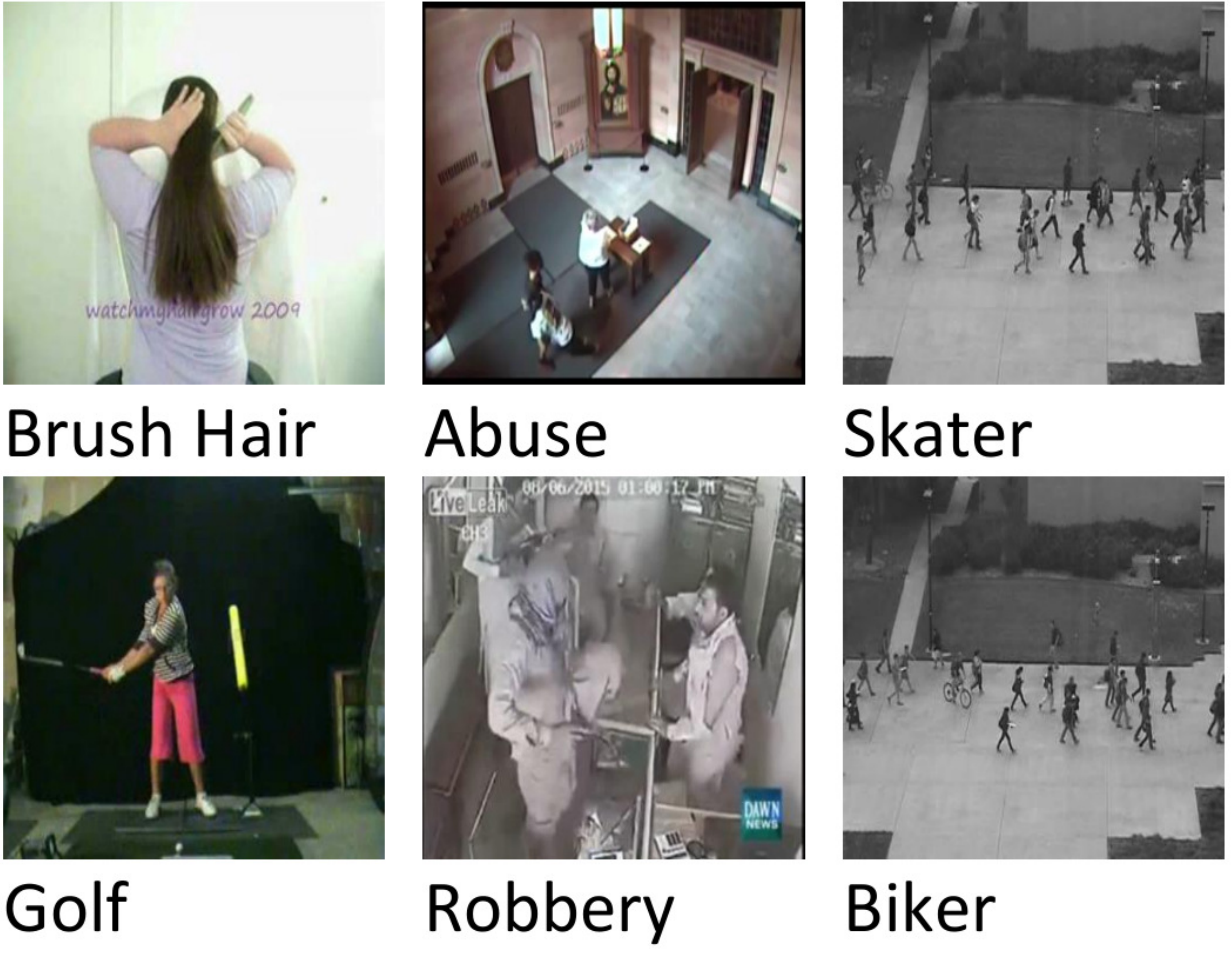}
\caption{Some examples from JHMDB (first column), UCF-Crime (second column), and USCD Ped2 (third column) datasets, with respective categories.}
\label{fig:example}
\end{figure}

 \subsection{Dash-Cam-Pose Dataset}
 Out-of-position (OOP) human pose detection is an important problem with regard to the safety of passengers in a vehicle.  While, there are public datasets for human pose estimation, they are usually annotated for generic pose estimation tasks, and neither do they contain any in-vehicle poses as captured by a dashboard camera, nor are they annotated for pose anomalies. To this end, we collected 104 videos, each 20-30 min long, from the Internet (including Youtube, ShutterStock, and Hollywood road movies). As these videos were originally recorded for diverse reasons, there are significant shifts in camera angles, perspectives, locations of the camera, scene changes, etc. We encourage the interested reader to refer to~\cite{GODS} for more details on this dataset and its collection. Next, we manually selected clips from these videos that are found interesting for our task. To extract as many clips as possible from these videos, we first segmented each video into three second clips at 30fps, which resulted in approximately 7000 clips. Next, we selected only those clips where the camera is approximately placed on the dashboard looking inwards, which amounted to 4,875 clips, totalling 4.06 hours. We annotated each clip with a weak binary label based on the poses of humans in the front seat. Specifically, if all the front-seat humans (passengers and the driver) are seated in-position, the clip was given a positive label, while if any human is seated OOP (based on~\cite{nordhoff2005motor,duma1996airbag}) for the entire 3s, the clip was labeled as negative.  Next, we used Open Pose~\cite{cao2016realtime} on each clip to extract a sequence of poses for every person. Our final Dash-Cam-Pose dataset consists of 4875 short videos, 1.06 million poses, of which 310,996 are OOP.
 
 We explore two pose-sequence representations for this task: (i) a simple bag-of-words (BoW) model, and (ii) using a Temporal Convolutional Network (TCN) ~\cite{kim2017interpretable} consisting of residual units with 1D convolutional layers capturing both local and global information via convolutions for each joint across time. For BoW, we use 1024 pose centroids computed using K-Means clustering. For the TCN, we use the following procedure. The poses from each person in each frame are vectorized and stacked into the temporal dimension. For each pose thus passed through TCN, we extract features from the last pooling layer, using a model pre-trained on the NTU-RGBD dataset~\cite{shahroudy2016ntu} (for 3D skeleton action recognition) to produce 256-D features for every clip. As our pre-trained TCN model takes 3D poses as input, we pad our Dash-Cam-Poses with zeros in the third dimension.
 
 We use a four-fold cross-validation for evaluation on Dash-Cam-Pose. Specifically, we divide the entire dataset into four non-overlapping splits, each split consisting of approximately 1/4-th the dataset, of which roughly 2/3rd's are labeled as positive (in-pose) and the rest as OOP. We use only the positive data in each split to train our one-class models. Once the models are trained, we evaluate on the held out split. For every embedded-pose feature, we use the binary classification accuracy against the ground truth. The evaluation is repeated on all the four splits and the performance averaged.

\subsection{Public Datasets}
\noindent\textbf{JHMDB dataset:} is a video action recognition dataset~\cite{jhuang2013towards} consisting of 968 clips with 21 classes (see Figure~\ref{fig:example} for example frames). To adapt the dataset for a one-class evaluation, we use a one-versus-rest strategy by  choosing sequences from an action class as ``normal'' while those from the rest 20 classes are treated as ``abnormal''. To evaluate the performance on the entire dataset, we cycle over the 21 classes, and the scores are averaged. For representing the frames, we use an ImageNet pre-trained VGG-16 model and extract features from the `fc-6' layer (4096-D). 

\noindent\textbf{UCF-Crime dataset:} is the largest publicly available real-world anomaly detection dataset~\cite{sultani2018real}, consisting of 1900 surveillance videos and 13 crime categories such as \emph{fighting}, \emph{robbery}, etc. and several ``normal'' activities, such as the daily walking, running and driving. Illustrative video frames from this dataset and their class labels are shown in Figure~\ref{fig:example}. To encode the videos, we use the state-of-the-art Inflated-3D (I3D) neural network~\cite{carreira2017quo}. Specifically, video frames from non-overlapping sliding windows (8 frames each) are passed through the I3D network; features are extracted from the `Mix\_5c' network layer, that are then reshaped to 2048-D vectors. For anomaly detections on the test set, we first map back the features classified as anomalies by our scheme to the frame-level and apply the official evaluation metrics~\cite{sultani2018real}.

\noindent\textbf{UCSD Ped2 dataset:} contains 16 videos in the training and 12 videos in the test set. There are 12 abnormal events in the test videos, such as the \emph{Biker}, \emph{Cart}, \emph{Skater}, etc. To encode the video data, we apply the deep autoencoder with causal 3D convolutions~\cite{bai2018empirical} trained to minimize the reconstruction loss. We extract features from the bottleneck layer of this model to be input to our algorithm. As the videos can be of arbitrary length, the pipeline is trained on clips from temporal sliding windows with a stride of one and consisting of 16 frames. As anomalous events are labeled frame-wise in this dataset, we use the averaged clip-level predictions within a window as the prediction for the center frame in that window. We use the evaluation metrics on these frame-level predictions similar to~\cite{liu2018future,luo2017revisit,abati2019latent}. 
\subsection{Experimental Setup}
Before using the above features in our algorithms, we found that it is beneficial to unit-normalize them. However, we do report results without such normalization on other datasets in Section~\ref{sec:uci}. These scaled features are then used in our GODS formulations, the optimization schemes for which are implemented using ManOpt~\cite{boumal2014manopt} and PyManOpt~\cite{JMLR:v17:16-177}. We use the conjugate gradient scheme for optimization, which typically converges in about 200 iterations. We initialize the iterates using the approach described in Section~\ref{sec:init}. The hyper-parameters in our models are chosen via cross-validation, and the sensitivities of these parameters are evaluated in the next section. We use regularization constants $\nu=1$ and $\lambda=1$. We use the inference criteria described in Section~\ref{sec:classification} for classifying a test point as in-class or an anomaly.

\subsection{Evaluation Metrics}
On the UCF-Crime dataset, we follow the official evaluation protocol, reporting AUC as well as the false alarm rate. For other datasets, we use the $F1$ score to reflect the sensitivity and accuracy of our classification models. As the datasets we use - especially the Dash-Cam-Pose -- are imbalanced across the two classes, having a single performance metric over the entire dataset may fail to characterize the quality of the discrimination for each class separately, which is of primary importance for the one-class task. To this end, we also report $\text{True Negative Rate}\ TNR=\frac{TN}{N}$, $\text{Negative\ Predictive\ Value}\ NPV=\frac{TN}{TN+FN}$, and $\overline{F1}=\frac{2\times TNR\times NPV}{TNR+NPV}$, alongside standard F1 scores. We will use $\overline{F1}$ on the Dash-Cam-Pose dataset and $F1$ score on other datasets. Informally, $\overline{F1}$ is the same as $F1$ with the positive and negative categories switched.

\subsection{Ablative Studies}
\label{exp_result}
\noindent\textbf{Synthetic Experiments: } To gain insights into the inner workings of our schemes, we present results on several 2D synthetic toy datasets. In Figure~\ref{fig:bods-vis}--\ref{fig:gods-vis-arb}, we show three plots with 100 points distributed as (i) Gaussian and (ii) some arbitrary distribution\footnote{The data follows the formula $f(x)=\sqrt{x}*(x+sign(randn)*rand)$, where randn and rand are standard MATLAB functions.}. We show the BODS hyperplanes in the Figure~\ref{fig:bods-vis}, and the GODS 2D subspaces in Figures~\ref{fig:gods-vis},~\ref{fig:gods-vis-arb} with the hyperplanes belonging to each subspace shown in same color. As the plots show, our models are able to orient the subspaces such that they confine the data within a minimal volume. In Figure~\ref{fig:kods-vis}, we show 300 data points (black dots) distributed along a 2D ring, a situation when a rectilinear GODS may fail (as the inner circle does not contain the one-class). As seen in Figures~\ref{fig:kods-vis-w1} and~\ref{fig:kods-vis-w2}, the two kernelized KODS hyperplanes capture the outer and inner decision regions separately, and their combined decision region is able to capture the ring structure of the input data, as seen in Figure~\ref{fig:kods-vis}. In Figure~\ref{fig:kods-3d}, we plot the decision surfaces for 3D data points.

\noindent\textbf{Choice of Manifold and Initialization.} As described in  Section~\ref{gods_extension}, our GODS algorithm may assume several optimization manifolds based on the type of regularization used between the classifier hyperplanes. In the Figure~\ref{manifold}, we evaluate three different manifold choices, namely (i) Stiefel manifold, (ii) oblique manifold, and (iii) Euclidean manifold. We also evaluate three different hyperplane initialization strategies: (i) random, (ii) SVD (as describved in Section~\ref{sec:init}), and (iii) mean of the data features. From the Figure, it can be seen that the Stiefel manifold and SVD initialization works best compared to oblique and Euclidean manifolds and against random or mean initializations, consistently on the three datasets. 
\begin{figure}[ht]
	\begin{center}
        \subfigure[Dash-Cam-Pose-BOW($\overline{F1}$)]{\label{ini-subfig:0}\includegraphics[width=0.49\linewidth,trim={0cm 0cm 0cm 0cm},clip]{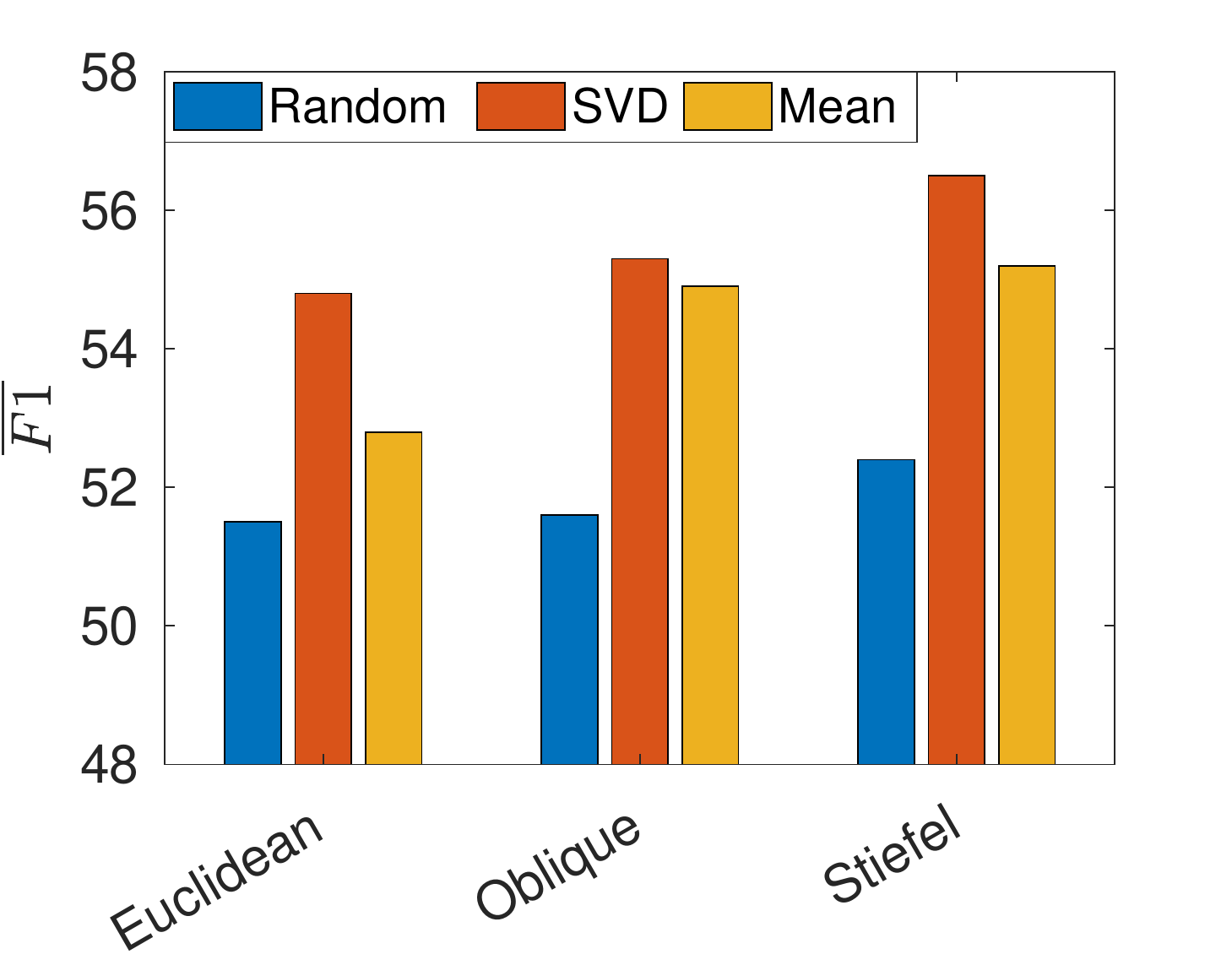}}
        \subfigure[Dash-Cam-Pose-TCN($\overline{F1}$)]{\label{ini-subfig:1}\includegraphics[width=0.49\linewidth,trim={0cm 0cm 0cm 0cm},clip]{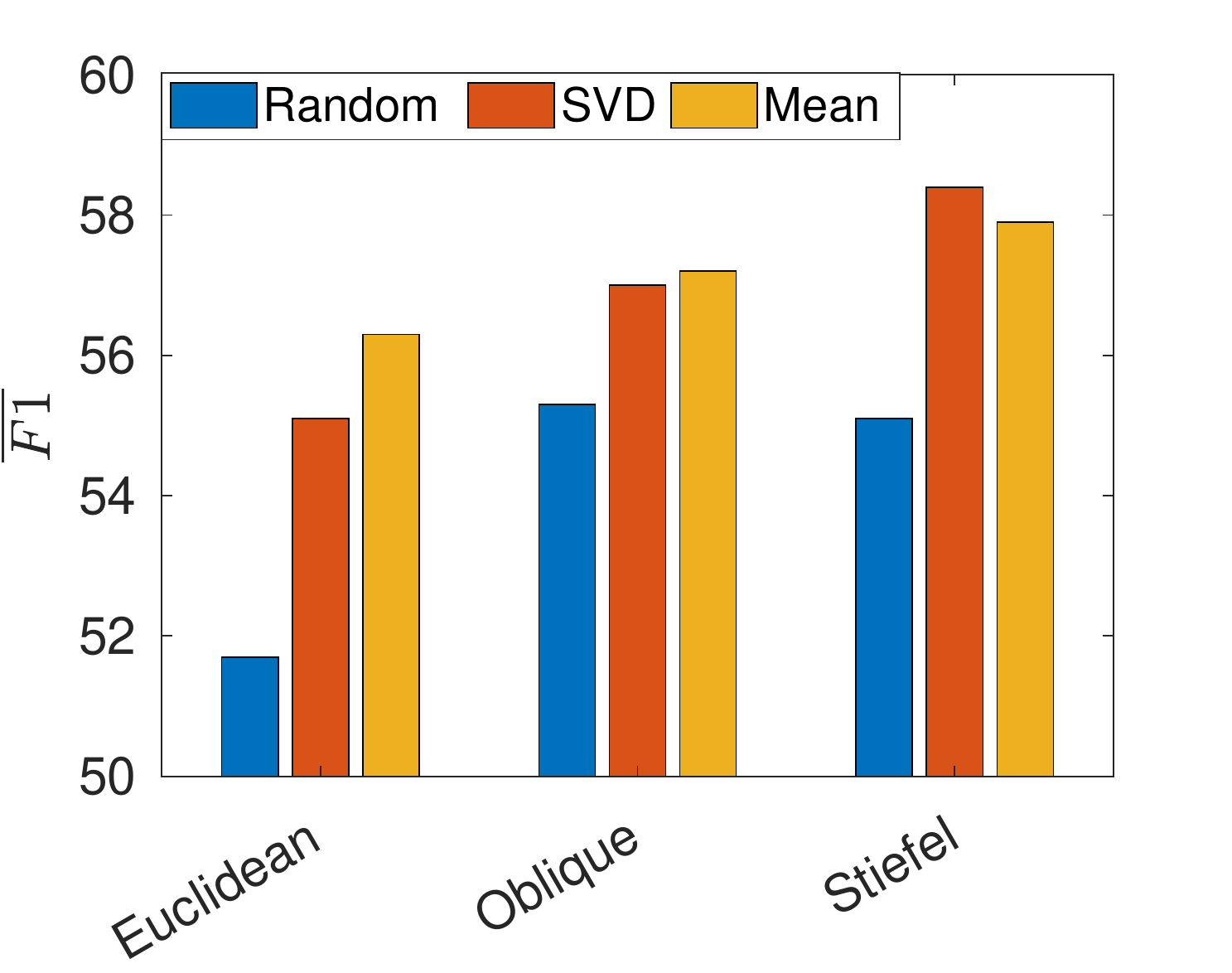}}
        \subfigure[JHMDB($F1$)]{\label{ini-subfig:2}\includegraphics[width=0.49\linewidth,trim={0cm 0cm 0cm 0cm},clip]{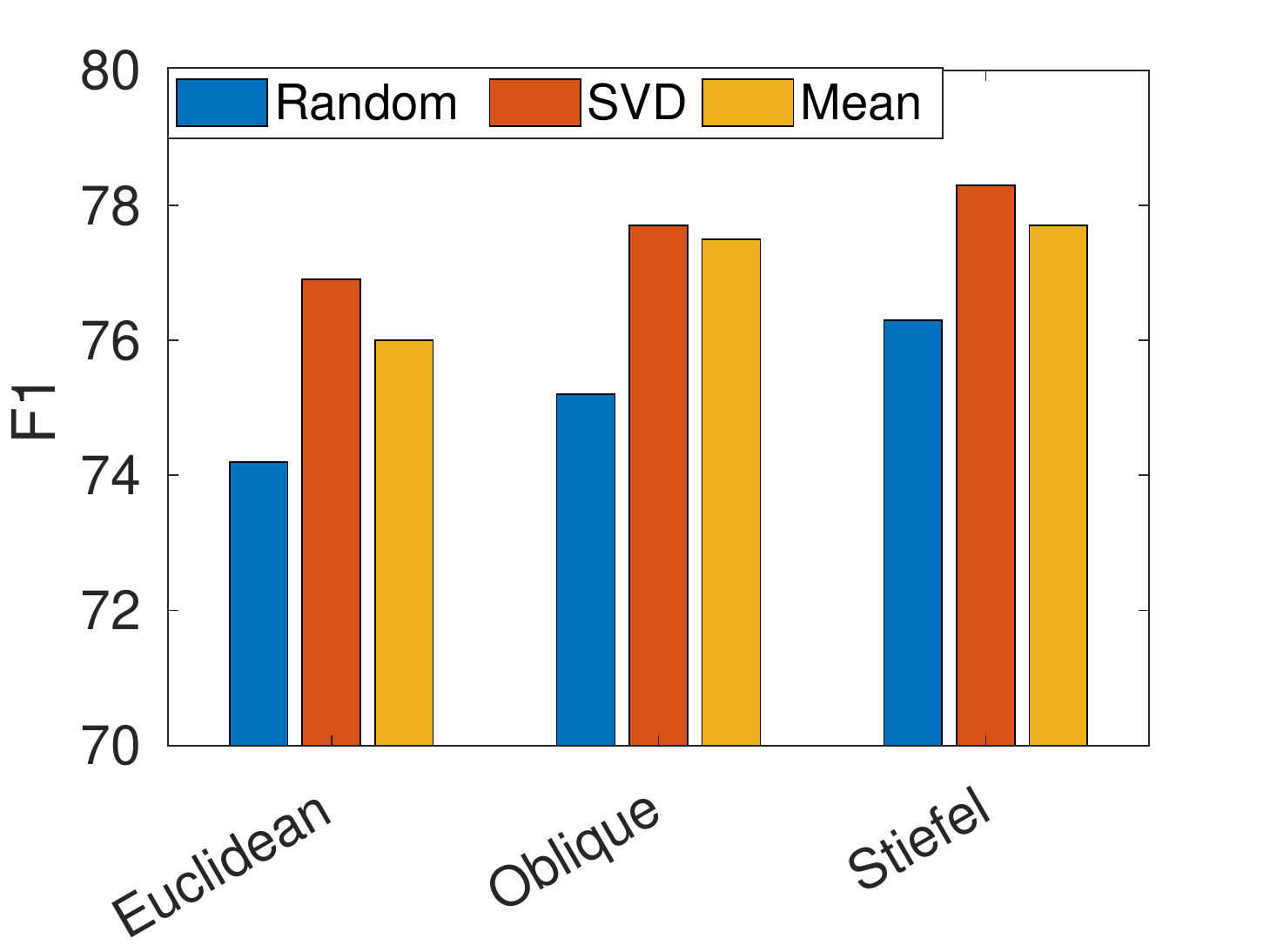}}
        \subfigure[UCF-Crime($F1$)]{\label{ini-subfig:3}\includegraphics[width=0.49\linewidth,trim={0cm 0cm 0cm 0cm},clip]{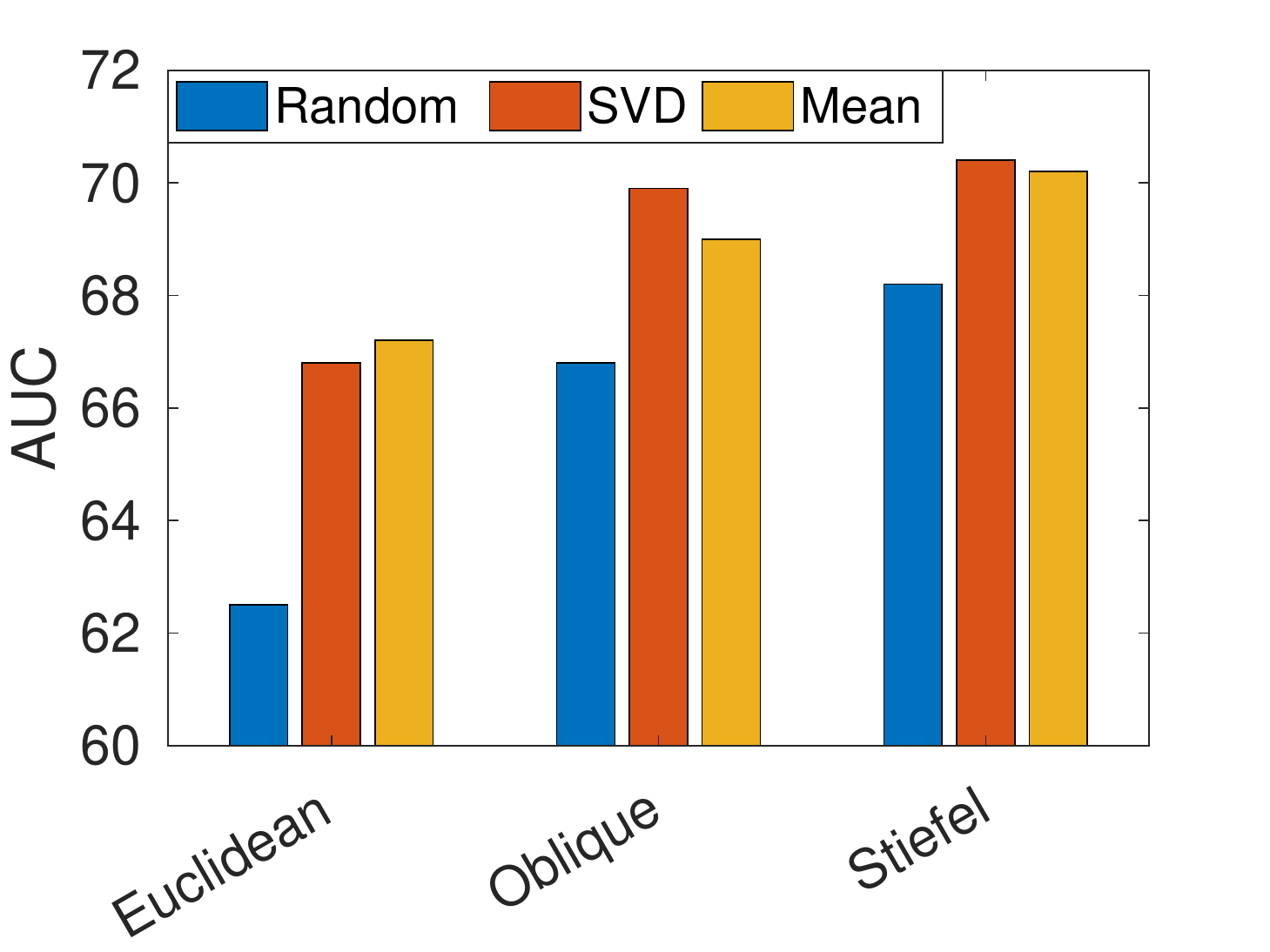}}
	\end{center}
	\caption{Performance of GODS with $F1$, $\overline{F1}$ and AUC for different initialization methods and manifold assumptions. }
    \label{manifold}
\end{figure}

\noindent\textbf{Choice of Normalization and soft-orthogonality.} In Table~\ref{l2norm}, we investigate the effectiveness of the L2 normalization and soft-orthogonality proposed in Equation~\eqref{eq:gods0} and~\eqref{eq:euc}. For fair comparison, we implement the GODS algorithm with Stifel, non-compact Stiefel, Euclidean and Oblique manifolds and evaluate on three different datasets (in Dash-Cam-Pose dataset, we use BOW and TCN features as input). From the experimental result, it is clear that the L2 norm helps to maintain a good performance in the GODS algorithm (verifying our assumptions in Section~\ref{sec:unit_sphere}). Thus, we will apply L2 norm by default in our following experiments.

\begin{table}
\caption{Comparisons of GODS variants on vision datasets. We compare: (i) $\gods$ (Eq.~\eqref{Problem3}) using the Stiefel manifold, (ii) $\gods_N$ using the non-compact Stiefel (Eq.~\eqref{eq:ncstiefel}), (iii) the Euclidean (Eq.~\eqref{eq:euc}), and (iv) the oblique manifolds (Eq.~\eqref{eq:gods0}). We compare under (i) $\ell_2$ unit-normalization of inputs and (ii) $C$, under soft-orthogonality (Eqs.~\eqref{eq:gods0},~\eqref{eq:euc}). We report $F1$, $\overline{F1}$ and AUC scores (in \%) for JHMDB, Dash-Cam-Pose (DCP), and UCF-Crime datasets.}
\begin{tabular}{|p{0.5mm}|p{8mm}|c|c|c|c|}
\hline
                           & Type & DCP-BOW               & DCP-TCN               & JHMDB & UCF-Crime \\\hline
            \rottxt{$\gods$} & $\ell_2$     & \textbf{56.3}      & \textbf{58.4}    & \textbf{77.7}    & 70.0    \\
                             & $\neq\!\!\ell_2$   & 53.8      & 56.6     & 75.2 &68.9  \\
                             & $\neq\!\!\ell_2\!\!+\!\!C$ & 53.5      & 55.4    & 75.1 &67.3   \\
                             & $\ell_2\!\!+\!\!C$     & 54.9      & 58.5    & 77.5 & \textbf{70.3}          \\\hline\hline
   
            \rottxt{$\gods_N$}& $\ell_2$     & 53.4      & 57.6    &75.2 & 68.6        \\
                              & $\neq\!\!\ell_2$ & 52.8      & 55.9     & 74.1 & 68.1   \\
                              & $\neq\!\!\ell_2\!\!+\!\!C$& 52.4      & 54.3    & 72.9 & 67.7         \\
                              & $\ell_2\!\!+\!\!C$   & 53.3      & 56.5     & 74.8 & 69.1       \\\hline\hline                        
          \rottxt{$\gods_E$} & $\ell_2$     & 55.1     & 55.8    & 76.7 & 66.8         \\
                              & $\neq\!\!\ell_2$ & 53.9      & 54.1     & 74.7 &64.6         \\
                              & $\neq\!\!\ell_2\!\!+\!\!C$& 53.2     & 53.7     & 73.6 & 63.4     \\
                              & $\ell_2\!\!+\!\!C$   & 54.8      & 55.1    & 74.9 &64.8      \\\hline\hline
           \rottxt{$\gods_O$}& $\ell_2$       & 55.5     & 56.8     &77.4 &66.7        \\
                              & $\neq\!\!\ell_2$   & 54.3      & 55.2     & 76.1 & 64.9        \\
                              & $\neq\!\!\ell_2\!\!+\!\!C$ & 53.8      & 54.5     & 75.4 & 64.8     \\
                              & $\ell_2\!\!+\!\!C$     & 54.8     & 55.9     & 76.8 & 66.2      \\\hline
 
\end{tabular}
\label{l2norm}
\end{table}

\noindent\textbf{Sensitivity of Margin $\eta$.} The hyperparameter $\eta$ decides the support margin between the hyperplanes and the one-class data. In Figure~\ref{eta}, we analyze the performance sensitivity against changes in $\eta$ on the JHMDB, UCF-Crime, UCSD-Ped2, and the Dash-Cam-Pose datasets. To ensure the learning will not ignore the changes in $\eta$, we increased the regularization constants on the two terms involving $\eta$ in~\eqref{eq:7}. On both datasets, the TPR (true positive rate) increases for increasing $\eta$, while the TNR decreases with a higher value of $\eta$. This is because the distances between the two orthonormal frames ($\mW_1,\mW_2$) may become larger to satisfy the new margin constraints imposed by $\eta$. Thus, more points will be included between the two frames and classified as positive. As $F1$ relies on the classification sensitivity of the positive data, while the $\overline{F1}$ relies on the negative classifications, they show opposite trends. Observing these trends, we fix $\eta=0.3$ in our experiments. 

\begin{figure}[ht]
	\begin{center}
        \subfigure[Dash-Cam-Pose($\overline{F1}$)]{\label{dcp-subfig:0}\includegraphics[width=0.49\linewidth,trim={0cm 0cm 0cm 0cm},clip]{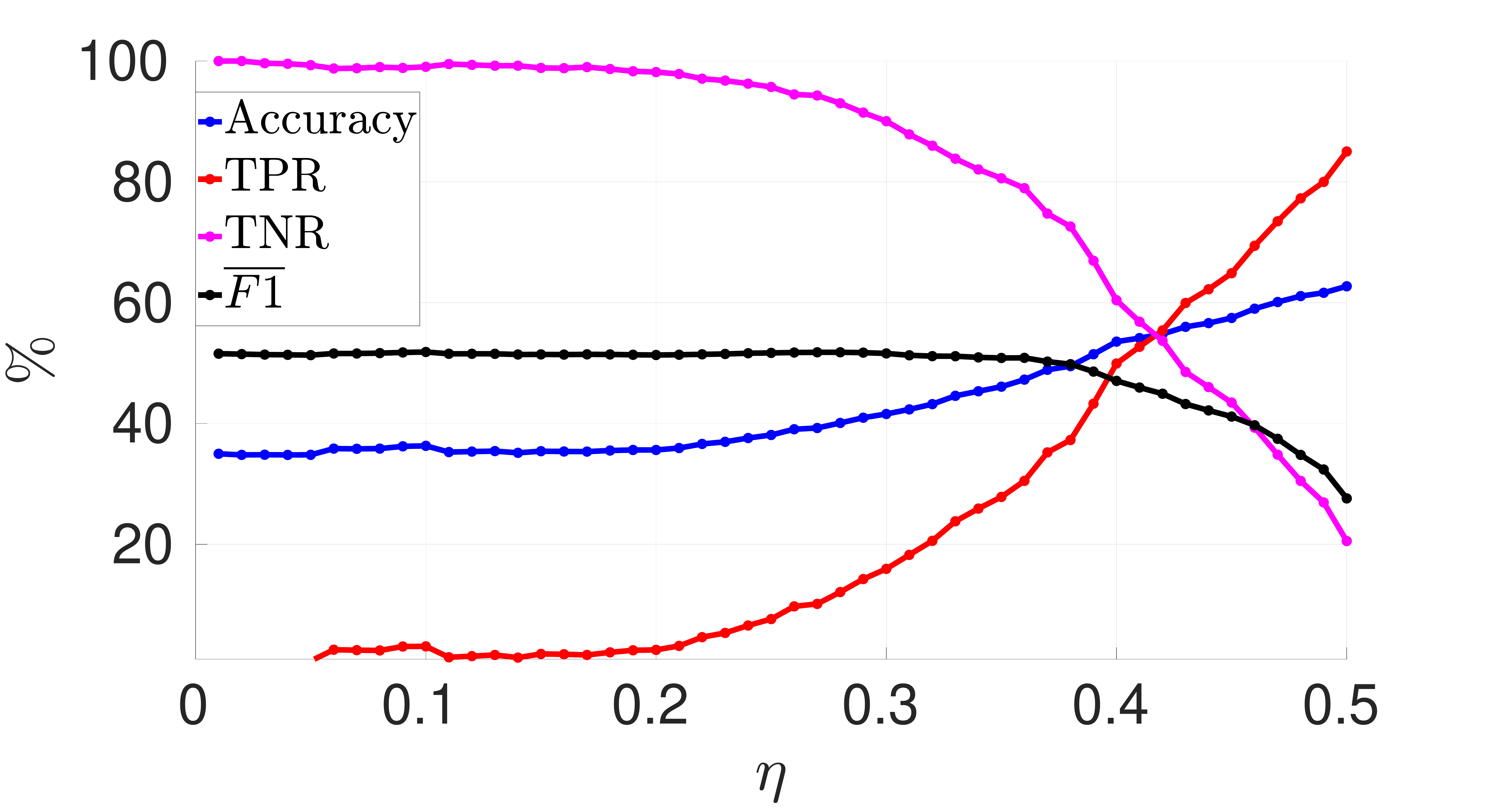}}
        \subfigure[JHMDB($F1$)]{\label{jhmdb-subfig:1}\includegraphics[width=0.49\linewidth,trim={0cm 0cm 0cm 0cm},clip]{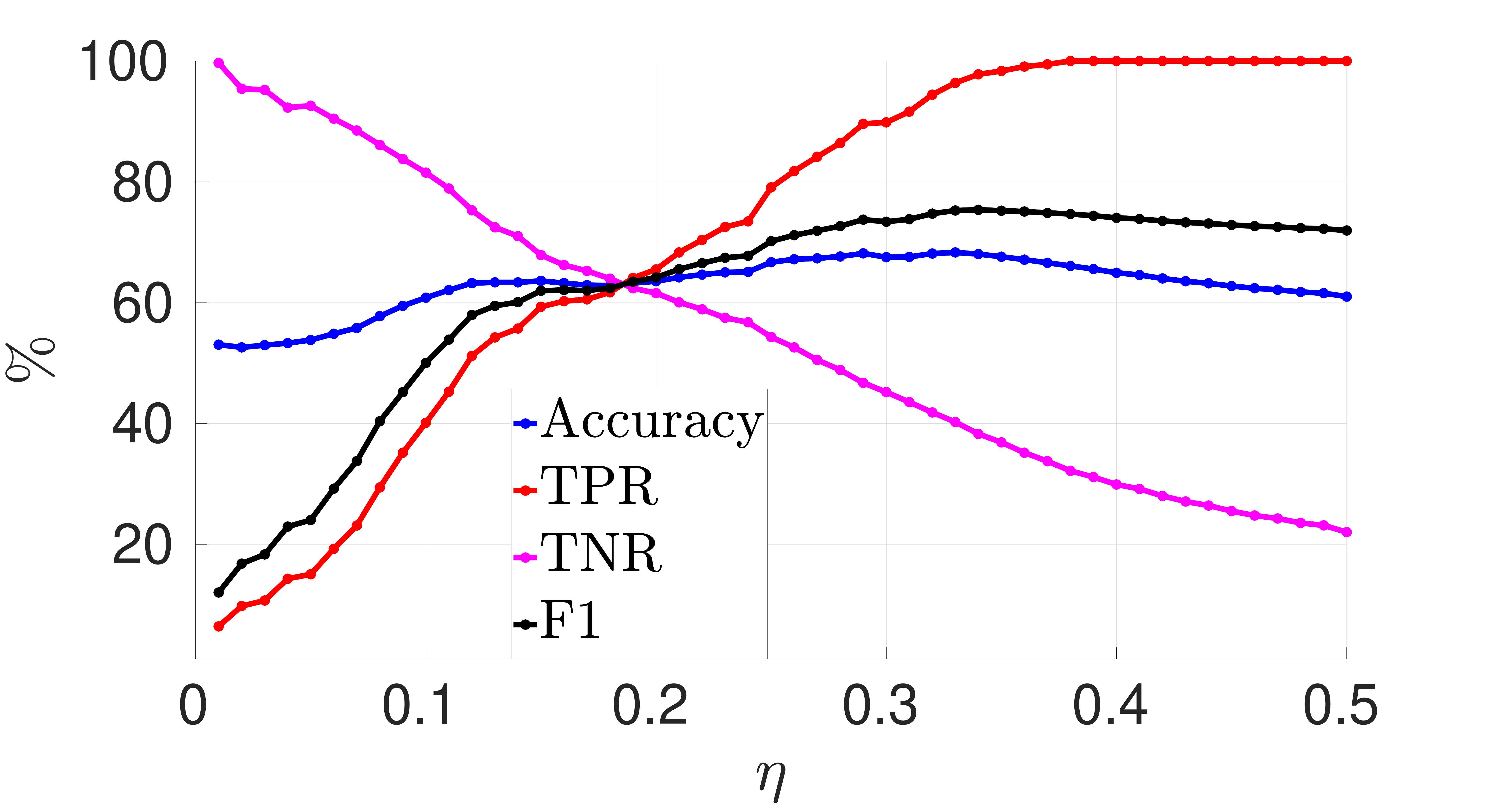}}
        \subfigure[UCF-Crime($AUC$)]{\label{ucf-subfig:2}\includegraphics[width=0.49\linewidth,trim={0cm 0cm 0cm 0cm},clip]{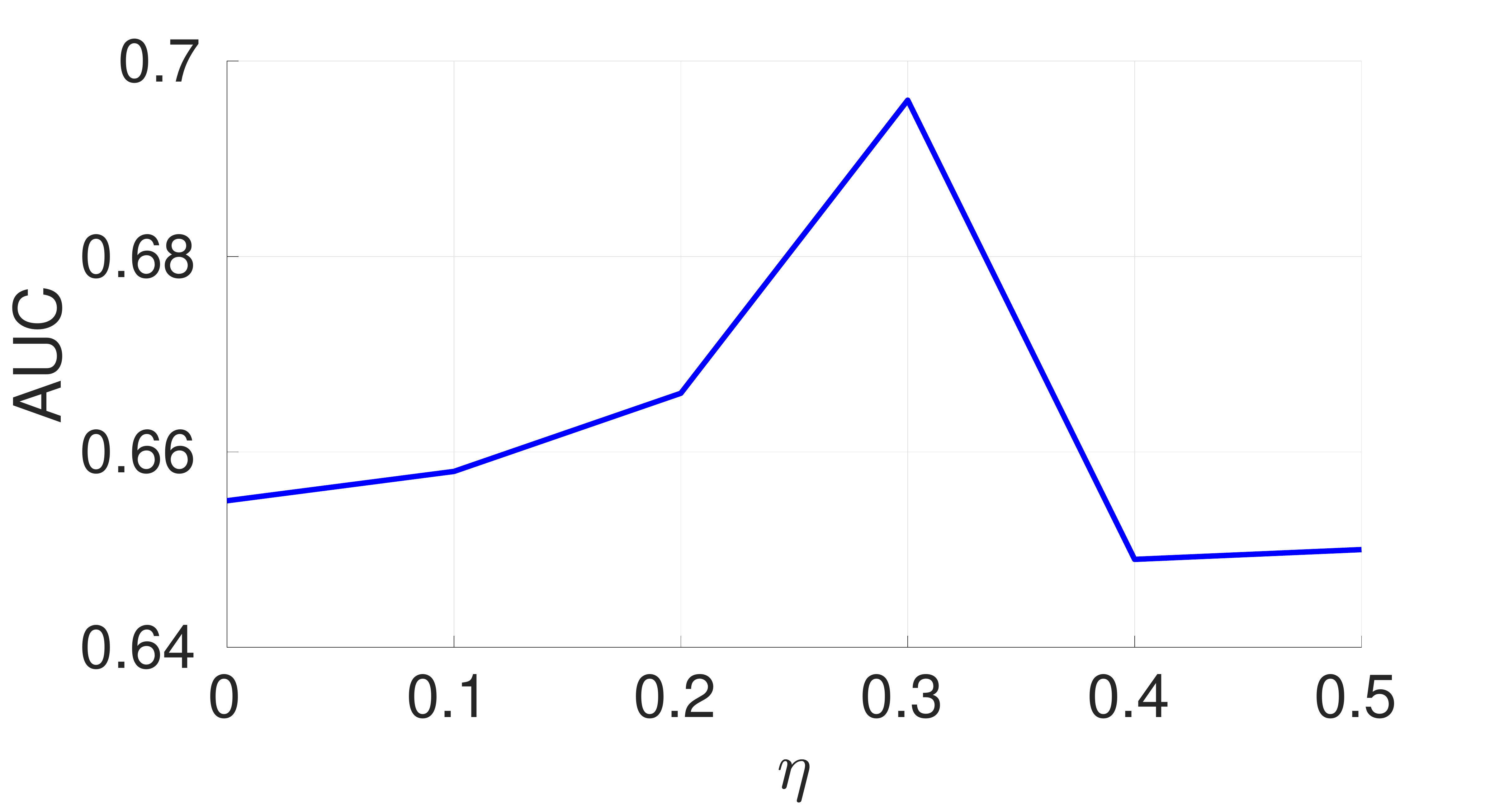}}
        \subfigure[UCSD Ped2($AUC$)]{\label{ucsd-subfig:3}\includegraphics[width=0.49\linewidth,trim={0cm 0cm 0cm 0cm},clip]{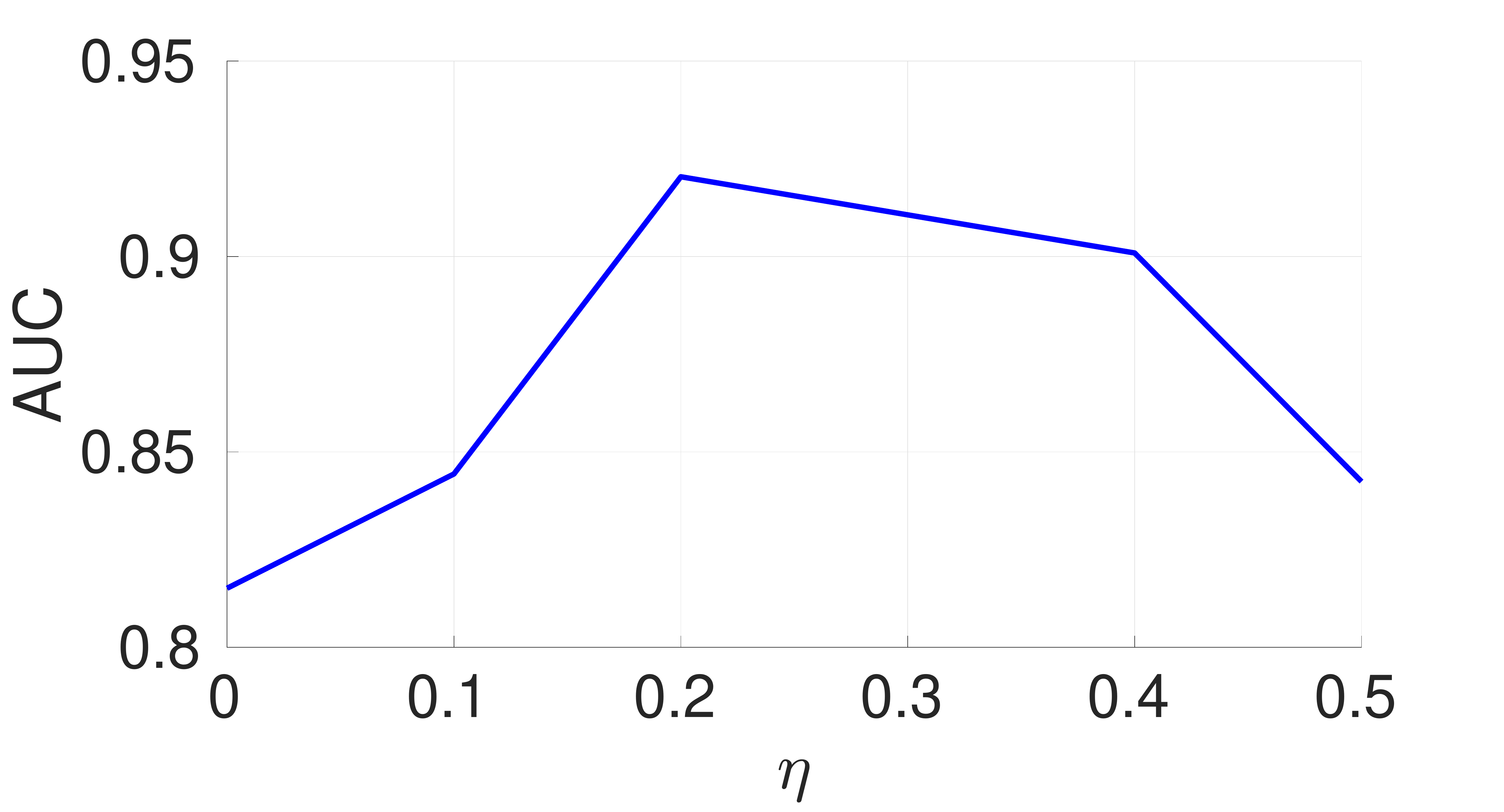}}
	\end{center}
	\caption{Performance of BODS with $F1$ and $\overline{F1}$ for increasing $\eta$. }
    \label{eta}
\end{figure}

\noindent\textbf{Number of Hyperplanes $K$.} In Figure~\ref{K_parameter}, we plot the influence of increasing number of hyperplanes on our four datasets. We find that after a certain number of hyperplanes, the performance saturates, which is expected, and suggests that more hyperplanes might lead to overfitting to the positive class. We also find that the TCN embedding is significantly better than the BoW model (by nearly 3\%) on the Dash-Cam-Pose dataset when using our proposed methods.  Surprisingly, S-SVDD is found to perform quite inferior against ours; note that this scheme learns a low-dimensional subspace to project the data to (as in PCA), and applies SVDD on this subspace. We believe, these subspaces perhaps are common to the negative points as well that it cannot be suitably discriminated, leading to poor performance. We make a similar observation on the other datasets as well. 

\begin{figure}[ht]
	\begin{center}
        \subfigure[UCF-Crime]{\label{ucf-subfig:0}\includegraphics[width=0.49\linewidth,trim={0cm 0cm 0cm 0cm},clip]{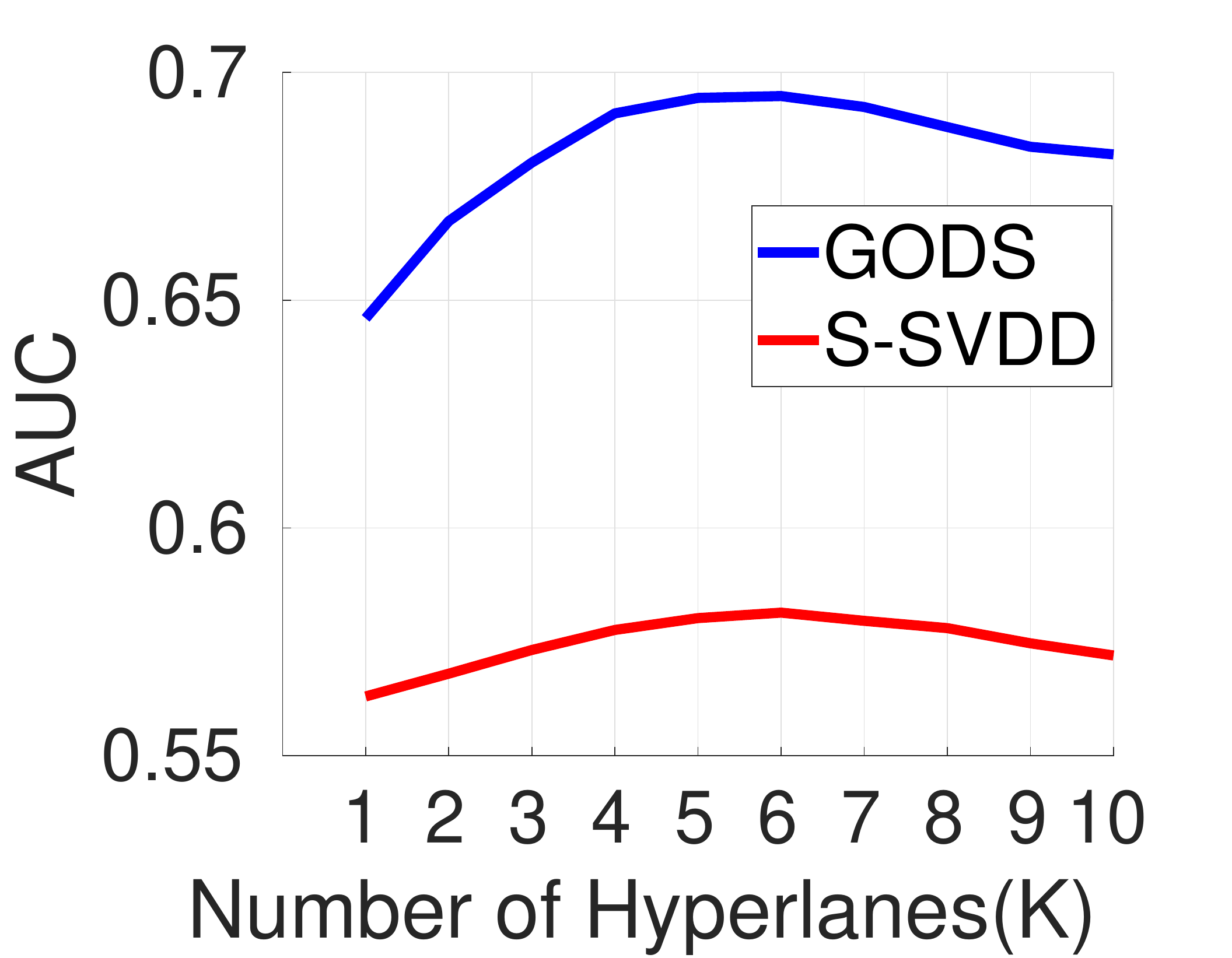}}
        \subfigure[Dash-Cam-Pose]{\label{dcp-subfig:1}\includegraphics[width=0.49\linewidth,trim={0cm 0cm 0cm 0cm},clip]{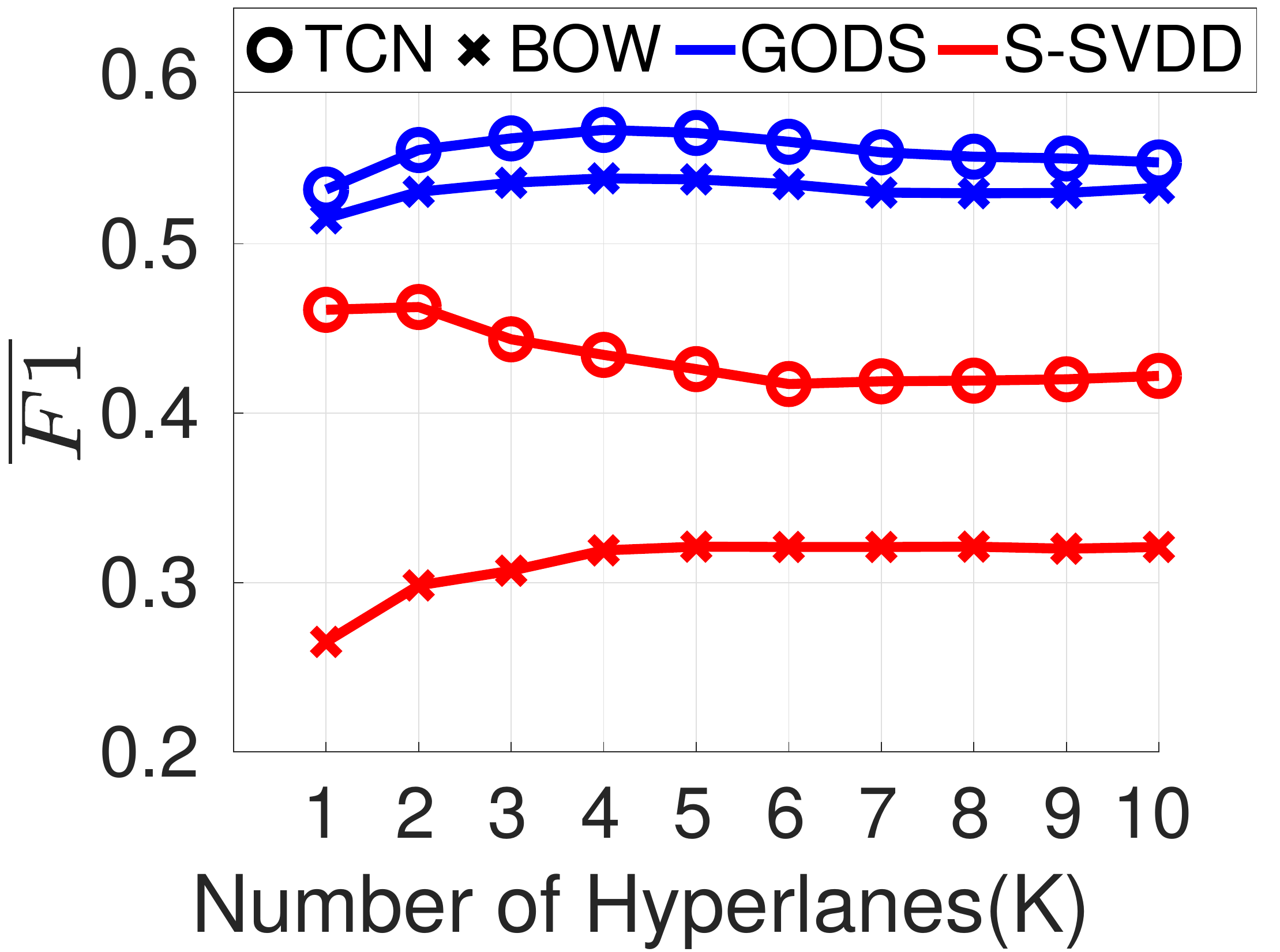}}
        \subfigure[JHMDB]{\label{jhmdb-subfig:2}\includegraphics[width=0.49\linewidth,trim={0cm 0cm 0cm 0cm},clip]{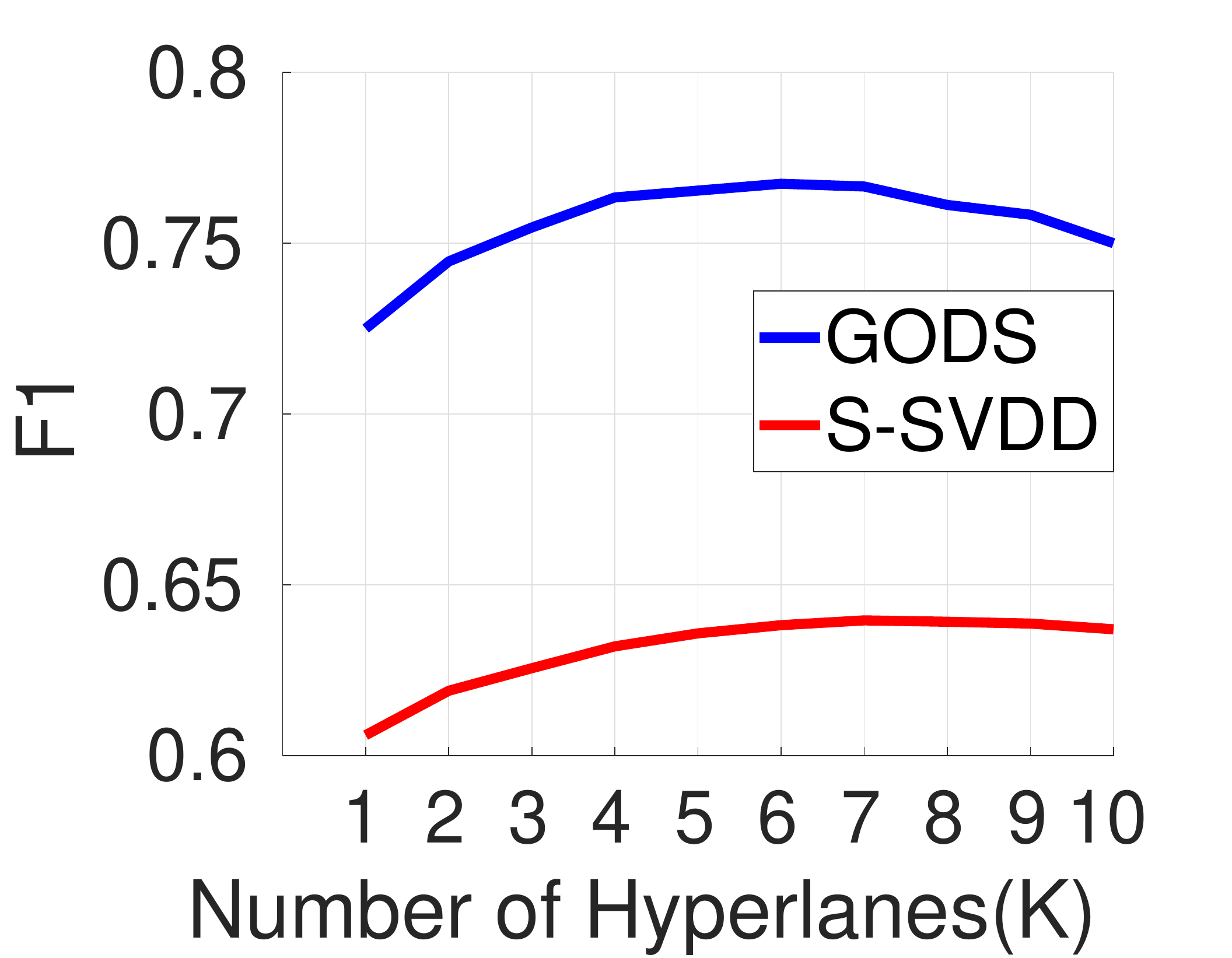}}
        \subfigure[UCSD Ped2]{\label{ped2-subfig:3}\includegraphics[width=0.49\linewidth,trim={0cm 0cm 0cm 0cm},clip]{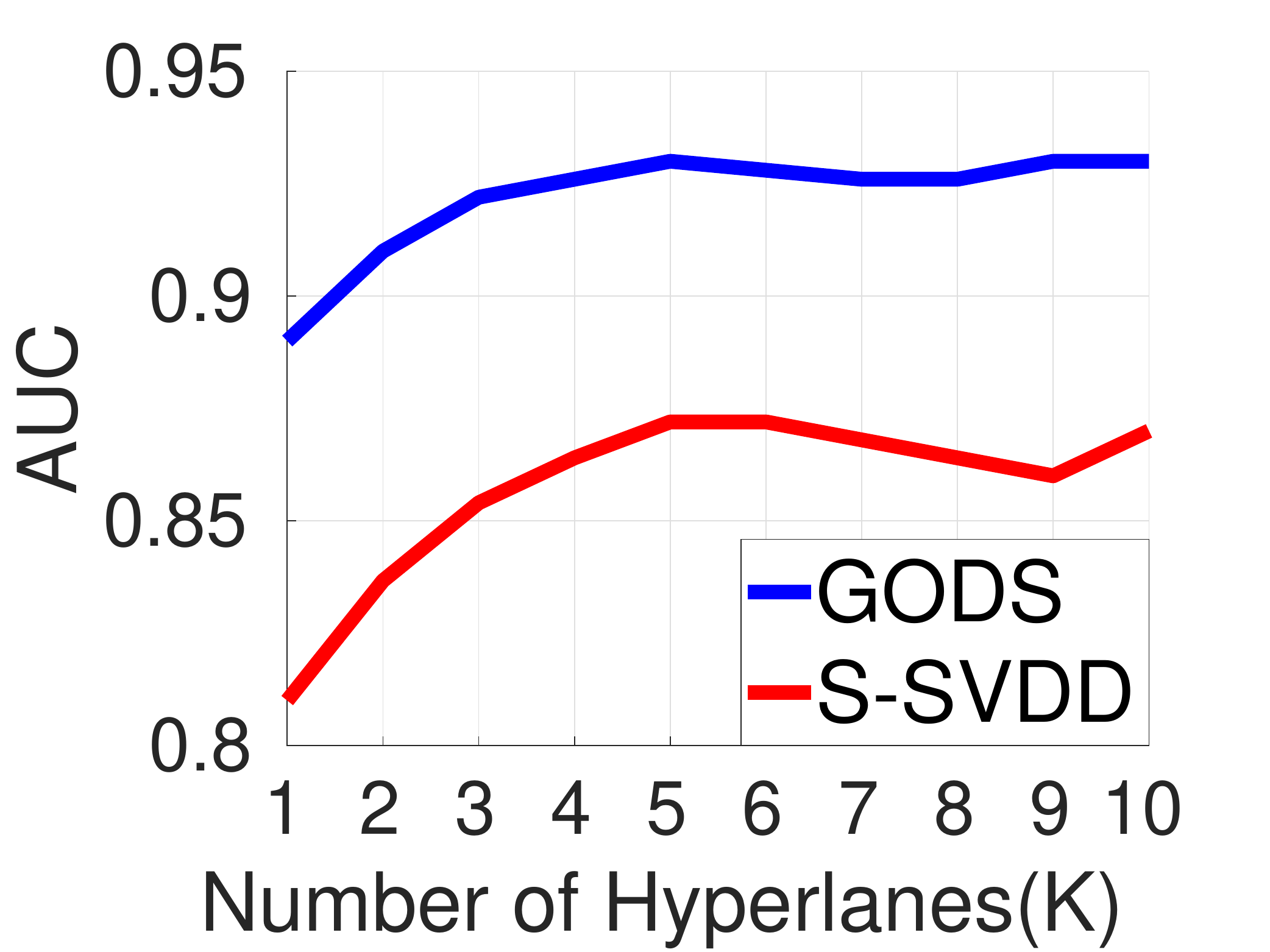}}
	\end{center}
	\caption{Performance of GODS for an increasing number of subspaces.}
    \label{K_parameter}
\end{figure}

\noindent\textbf{Kernel Choices and Number of Subspaces in $K$.} In Figure~\ref{fig:kernel}, we demonstrate the performance of KODS on the datasets for various choices of the embedding kernels, and also when increasing the number of Hilbert space classifiers $K$ for each choice of the kernel feature map on every dataset we use. Specifically, we experiment with linear, RBF, and polynomial kernels on JHMDB, UCF-Crime, UCSD Ped2, and Dash-Cam-Pose (with TCN) datasets, while use the Chi-square~\cite{moh2007chi} and Histogram Intersection kernels~\cite{barla2003histogram} on the Dash-Cam-Pose dataset with the Bag-of-Words features. We set $\sigma = 0.1$ for the bandwidth in the RBF kernel, polynomial kernel degree is set to $3$, and use 1024 words in the Bag-of-Words representation. In Figure~\ref{fig:kernel}, we see that the performance saturates with increasing number of hyperplanes. The RBF kernel seems to work better on the JHMDB and UCF-Crime datasets, while the polynomial kernel demonstrates higher performances on the UCSD-Ped2 and the Dash-Cam-Pose datasets (with TCN). For the Bag-of-Words features on the Dash-Cam-Pose dataset, the Chi-square kernel shows better performance than the Historgram Intersection kernel.

\begin{figure*}[ht]
	\begin{center}
        \subfigure[Dash-Cam-Pose-BOW]{\label{dcp-bow-subfig:0}\includegraphics[width=0.19\linewidth,trim={0cm 0cm 0cm 0cm},clip]{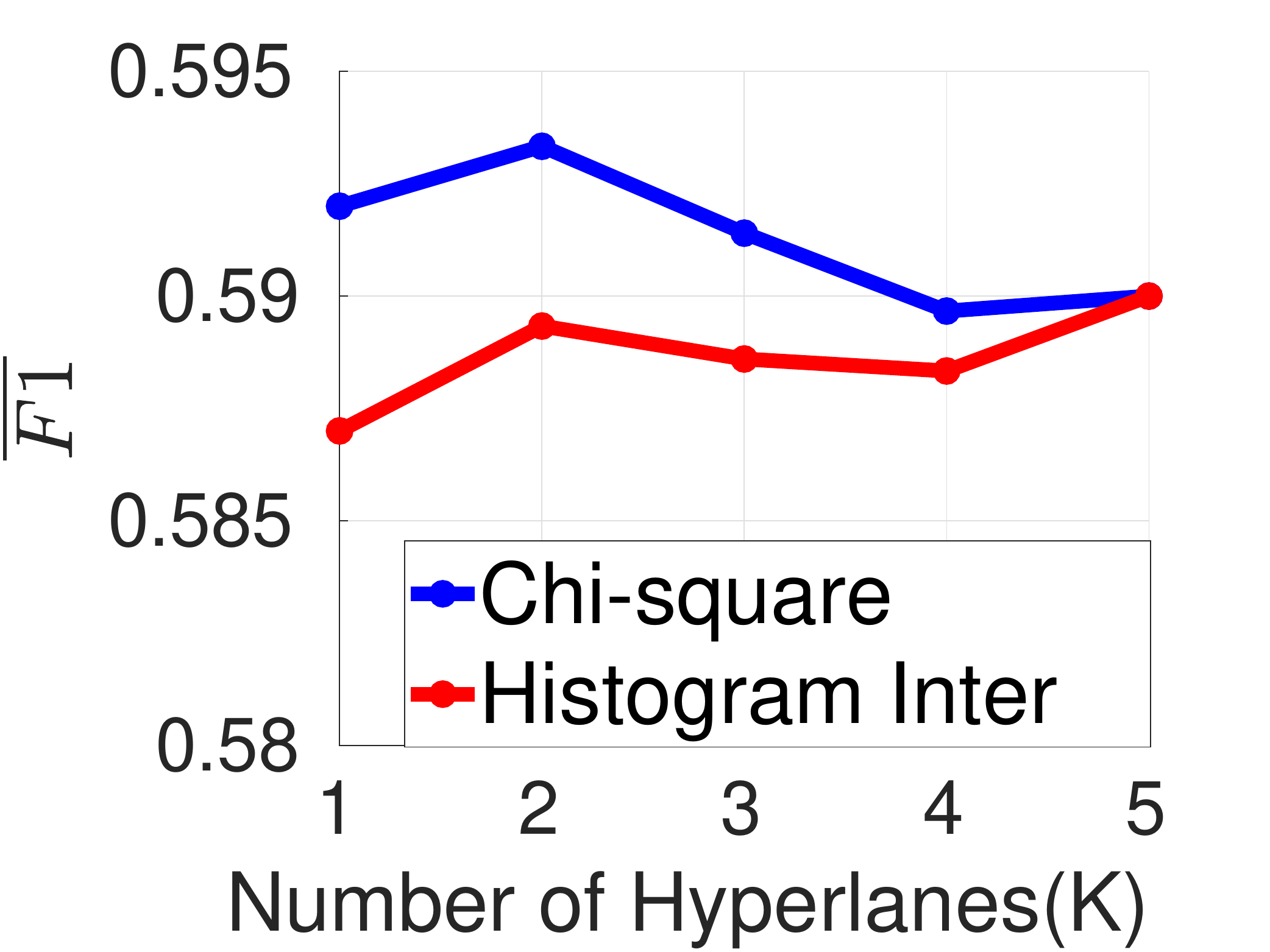}}
        \subfigure[Dash-Cam-Pose-TCN]{\label{dcp-tcn-subfig:1}\includegraphics[width=0.19\linewidth,trim={0cm 0cm 0cm 0cm},clip]{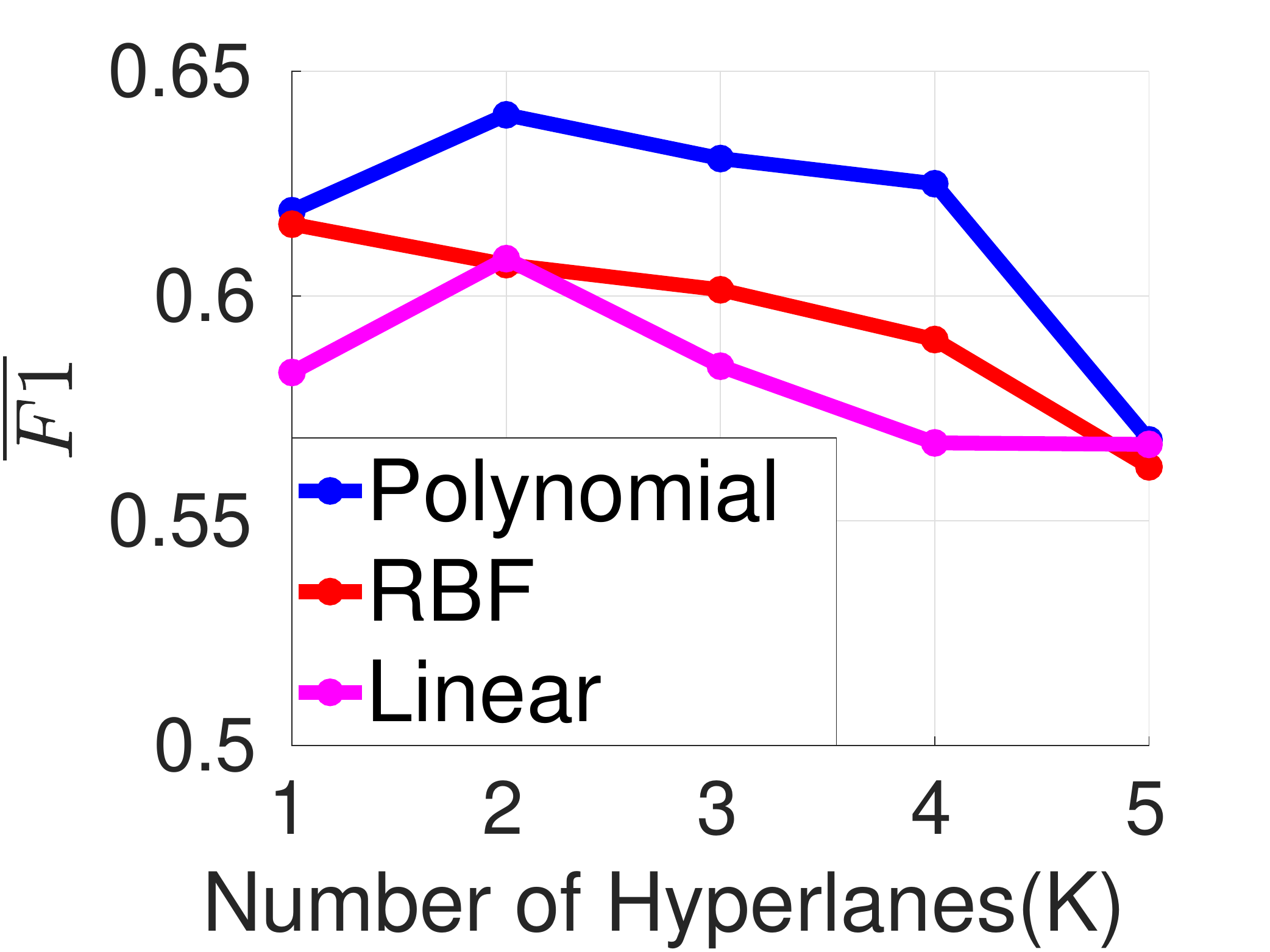}}
        \subfigure[UCF-Crime]{\label{ucf-crime-subfig:2}\includegraphics[width=0.19\linewidth,trim={0cm 0cm 0cm 0cm},clip]{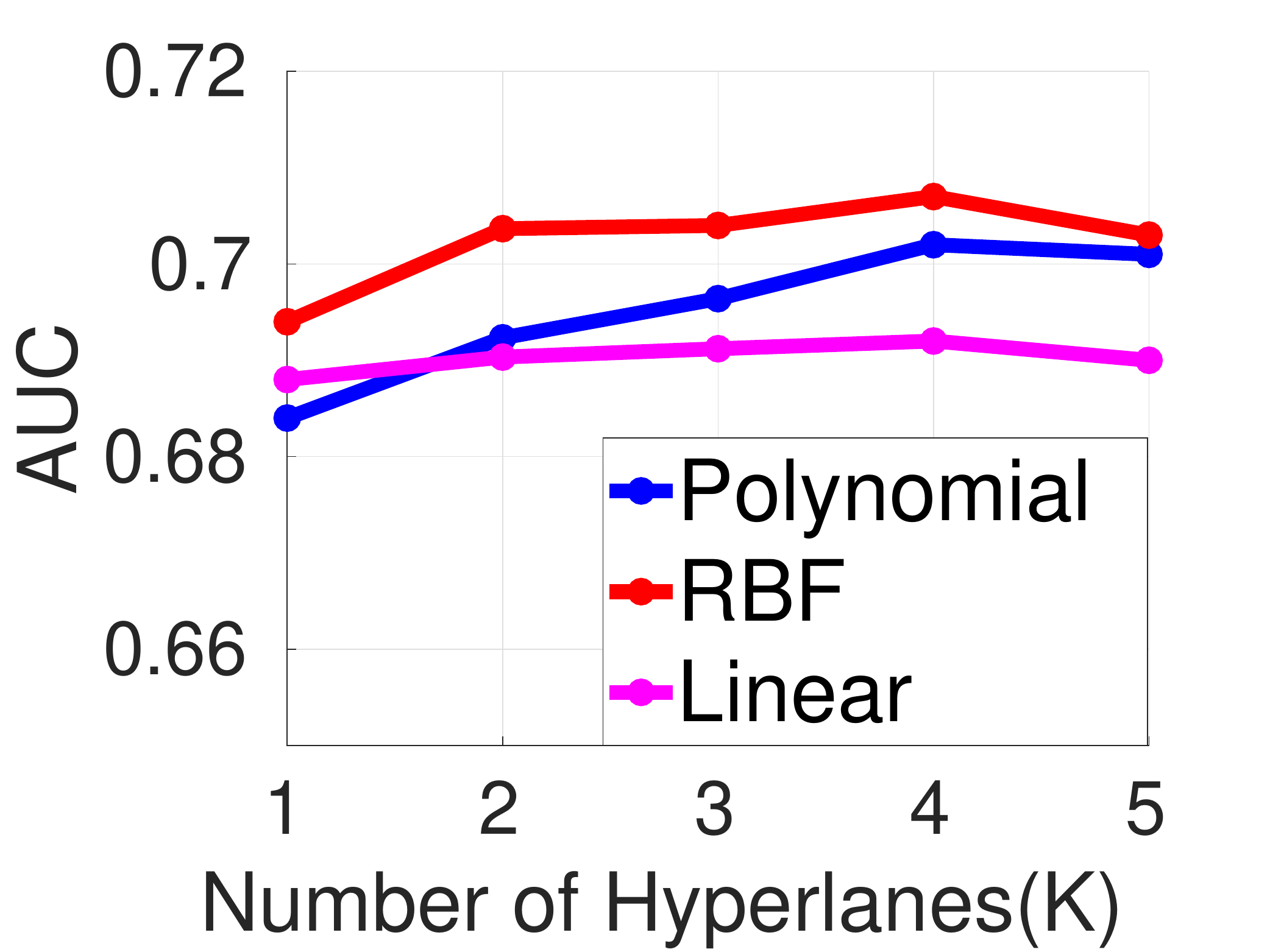}}
        \subfigure[JHMDB]{\label{subfig:4}\includegraphics[width=0.19\linewidth,trim={0cm 0cm 0cm 0cm},clip]{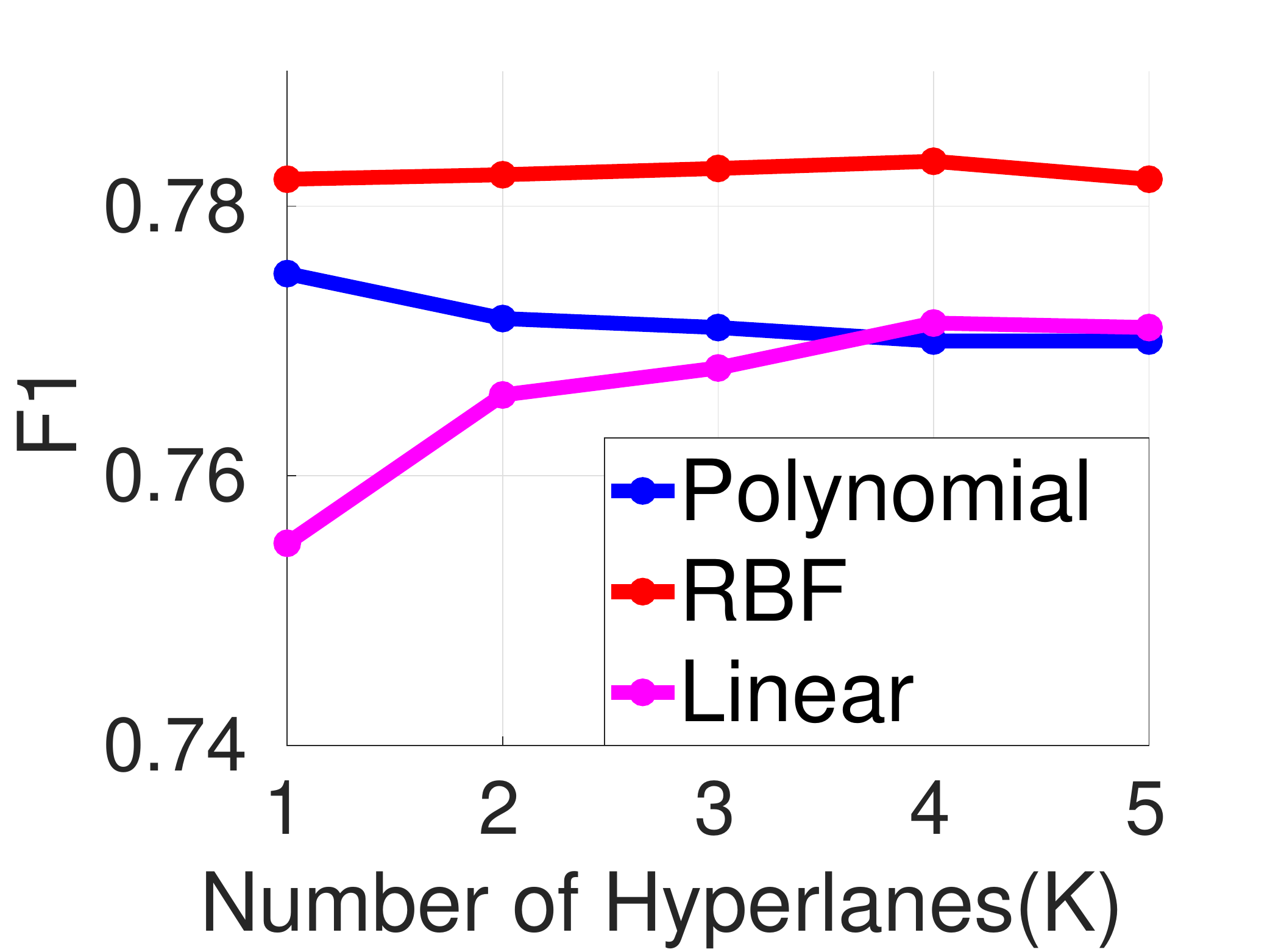}}
        \subfigure[UCSD Ped2]{\label{subfig:5}\includegraphics[width=0.19\linewidth,trim={0cm 0cm 0cm 0cm},clip]{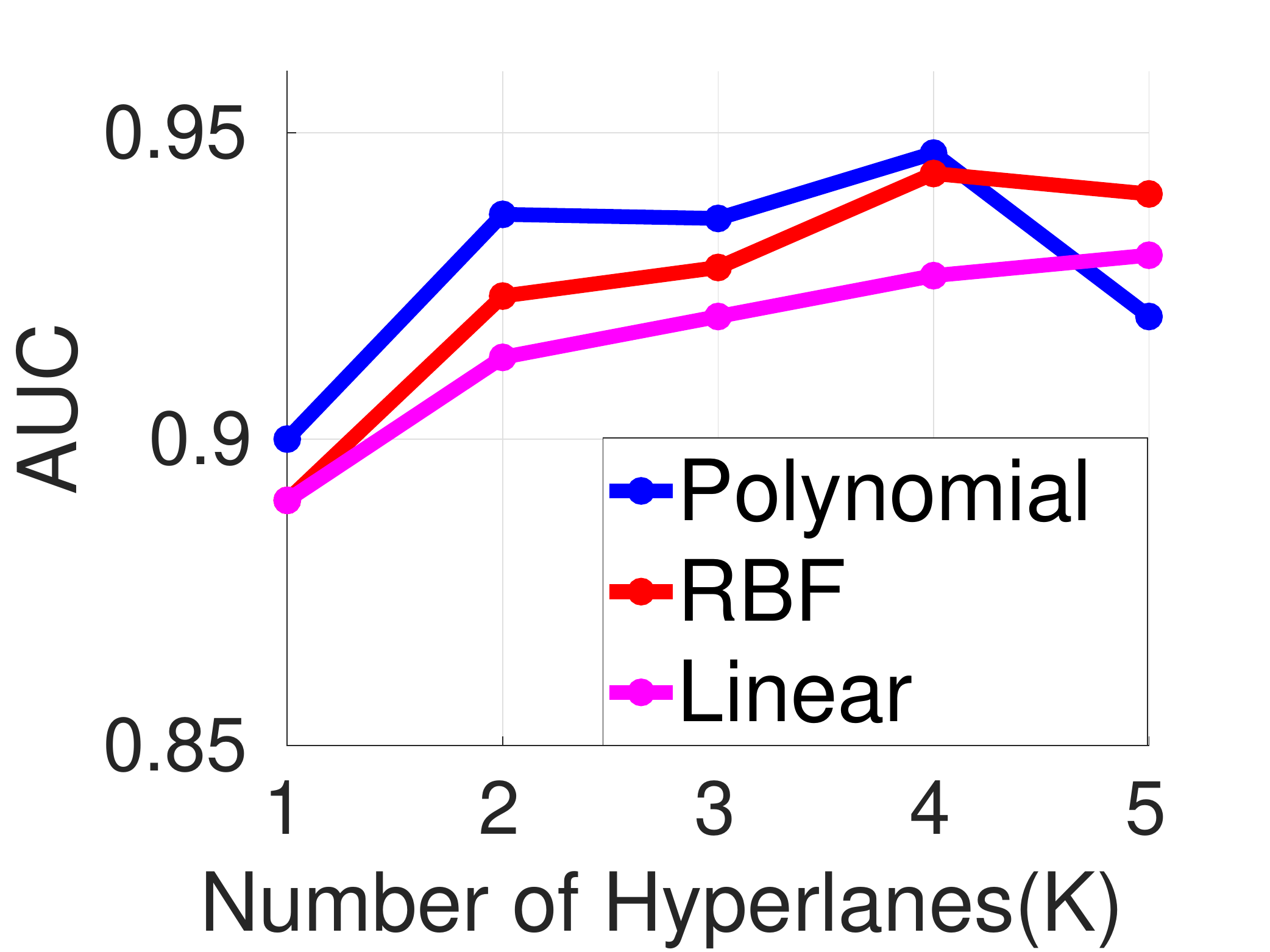}}
	\end{center}
	\caption{Performance of KODS in different kernel type on various datasets for an increasing number of subspaces.}
    \label{fig:kernel}
\end{figure*}

\noindent\textbf{Empirical Convergence.} 
In Figures~\ref{optsubfig:0} and~\ref{gods-optsubfig:1}, we show the empirical convergences of our GODS algorithm on the JHMDB dataset using the original GODS formulation in~\eqref{Problem3}. We show the convergence in the objective value as well as the magnitude of the Riemannian gradients. As is clear, our algorithm convergences in about 200 iterations on this dataset. We repeat this experiment on the KODS formulation~\eqref{eq:kgods-5} using different kernel maps. As is seen from Figures~\ref{optsubfig:3} and~\ref{optsubfig:4}, while the convergence is slower compared to that in GODS -- perhaps due to our approximations -- it does converge suitably for appropriate kernel choices. 

\noindent\textbf{Running Time.} In the Figure~\ref{optsubfig:2}, we demonstrate the time taken for training our different models. For this analysis, we use an Intel i7-6800K 3.4GHz CPU with 6 cores. We implement the different algorithms in the Matlab, run it on the same data, and record the training time with an increasing number of training samples. For GODS, KODS, and S-SVDD, we use 3 hyperplanes in the subspaces. It can be seen that the GODS, BODS, and KODS are not substantially more computationally expensive against prior methods, while remaining  empirically superior (Table~\ref{soa}). 

\begin{figure*}[ht]
	\begin{center}
        \subfigure[GODS-objective]{\label{optsubfig:0}\includegraphics[width=0.19\linewidth,trim={0cm 0cm 0cm 0cm},clip]{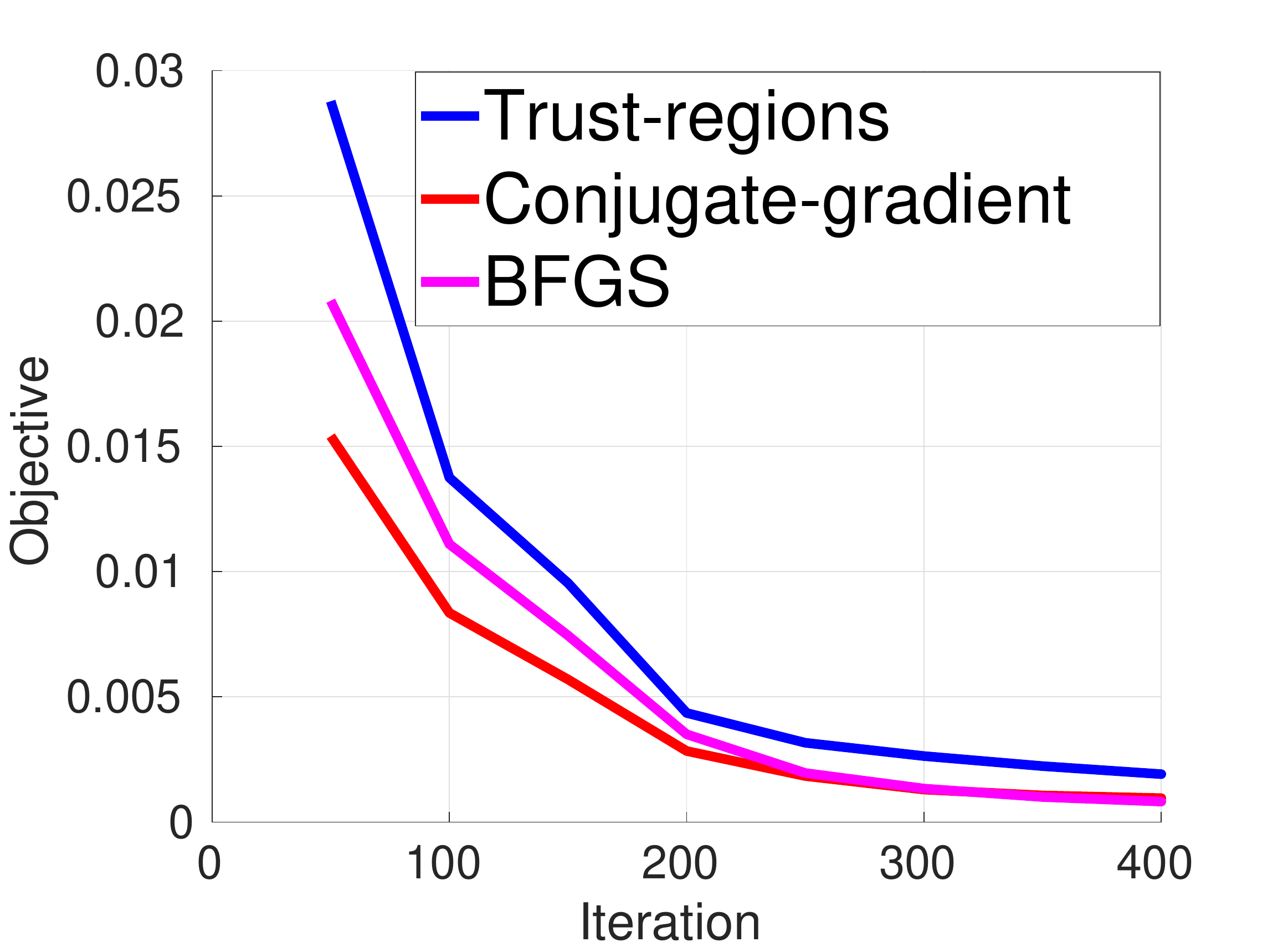}}
        \subfigure[GODS-gradient]{\label{gods-optsubfig:1}\includegraphics[width=0.19\linewidth,trim={0cm 0cm 0cm 0cm},clip]{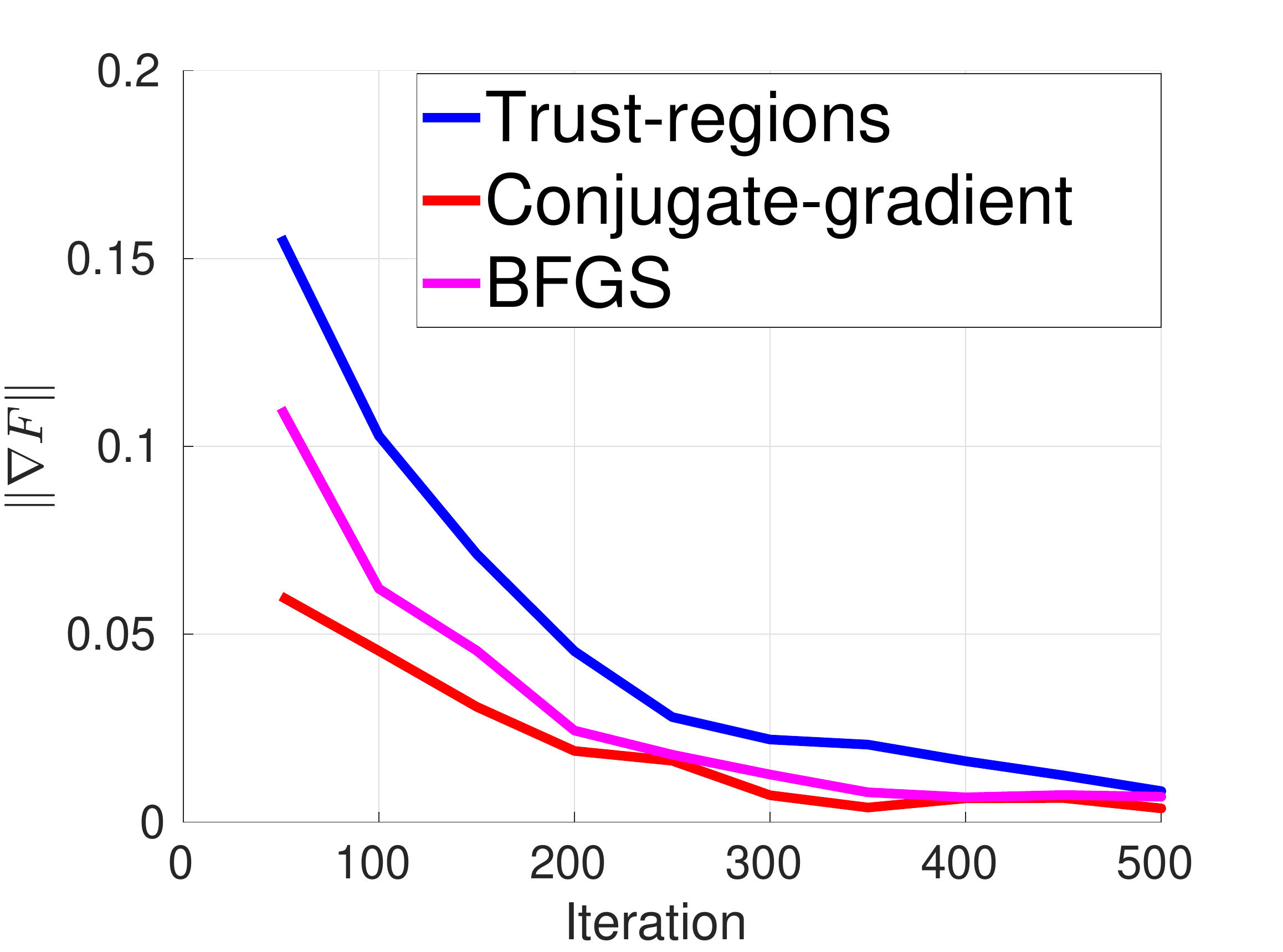}}
        \subfigure[KODS-objective]{\label{optsubfig:3}\includegraphics[width=0.19\linewidth,trim={0cm 0cm 0cm 0cm},clip]{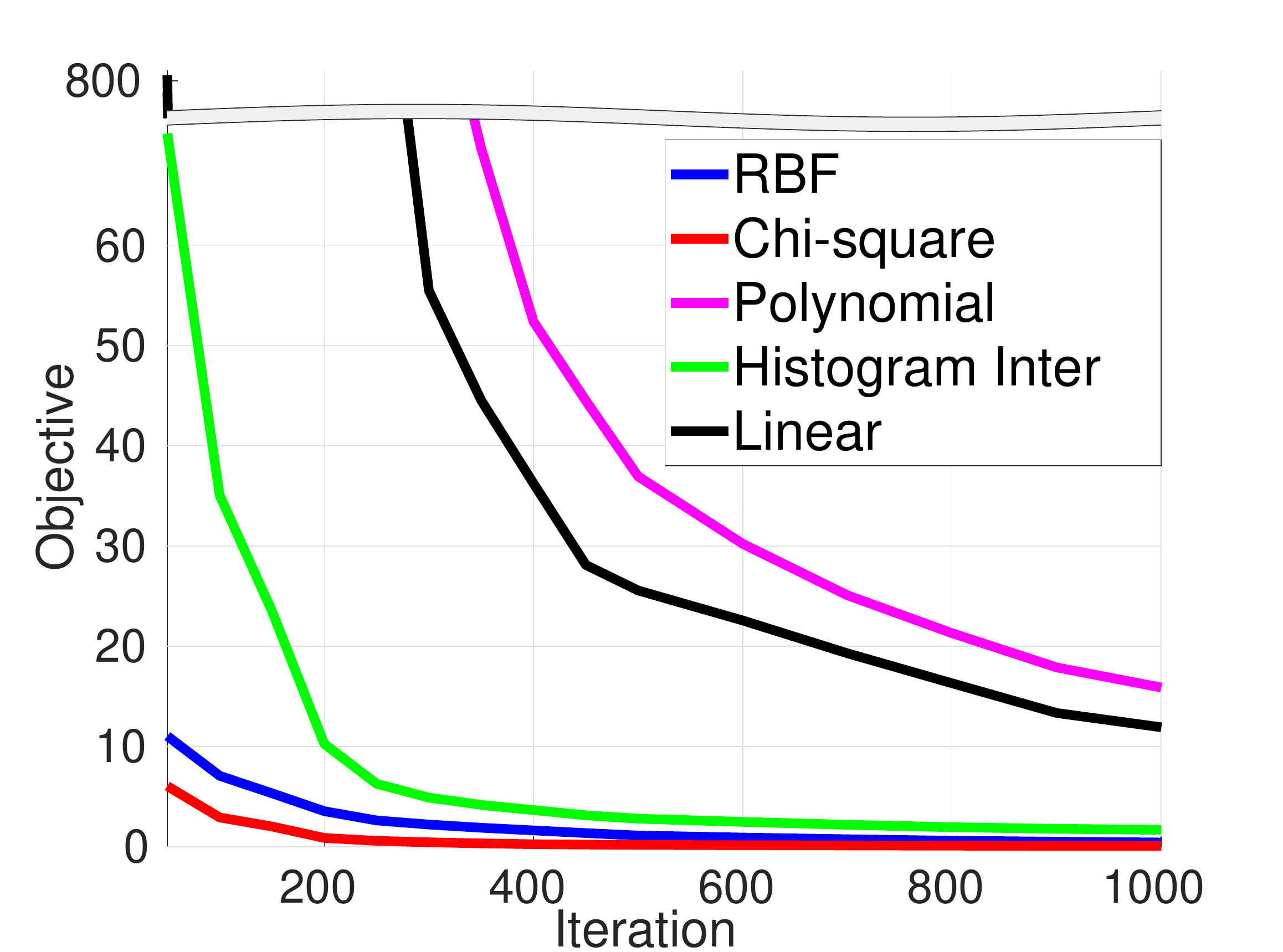}}
        \subfigure[KODS-gradient]{\label{optsubfig:4}\includegraphics[width=0.19\linewidth,trim={0cm 0cm 0cm 0cm},clip]{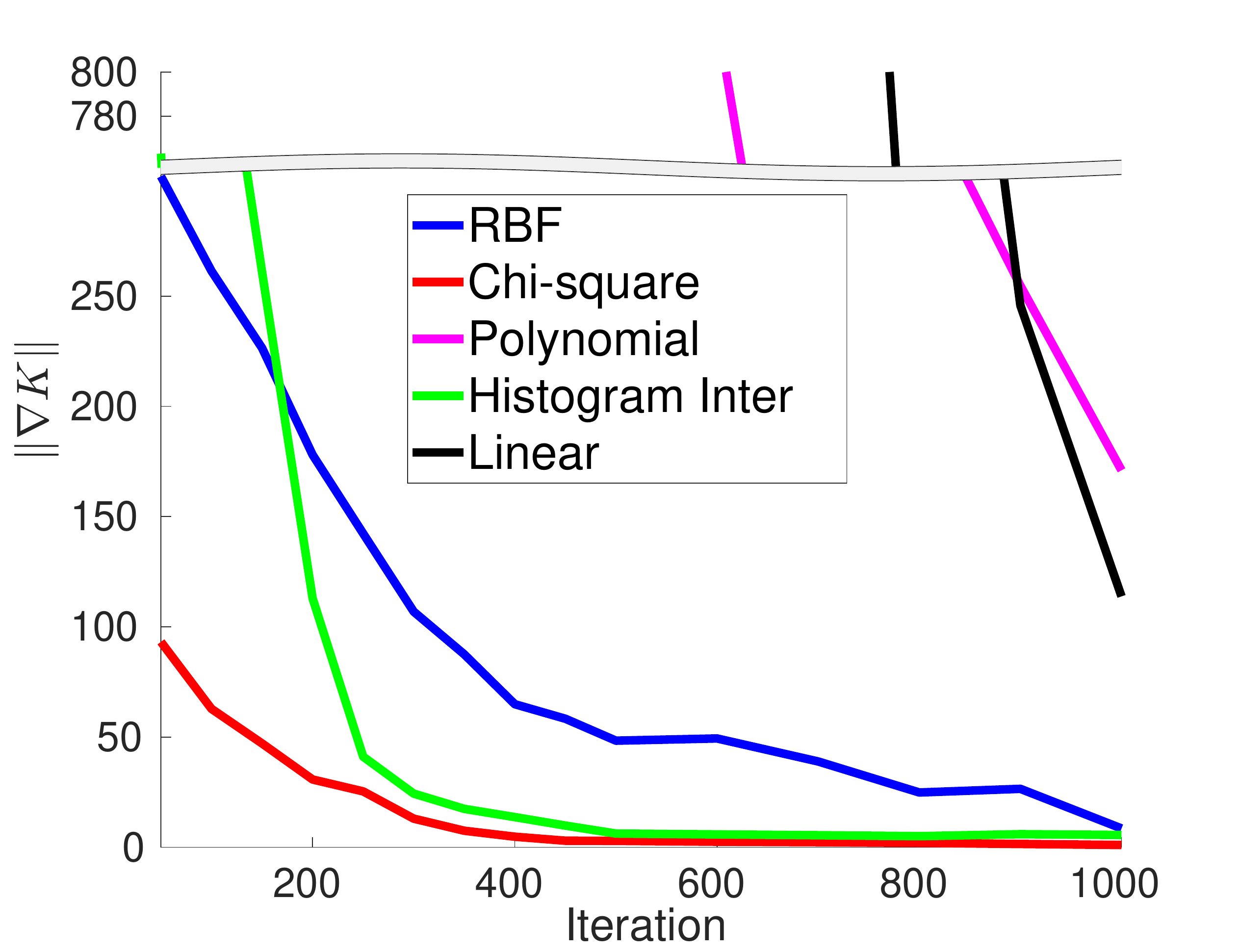}}
    	\subfigure[Running time]{\label{optsubfig:2}\includegraphics[width=0.19\linewidth,trim={0cm 0cm 0cm 0cm},clip]{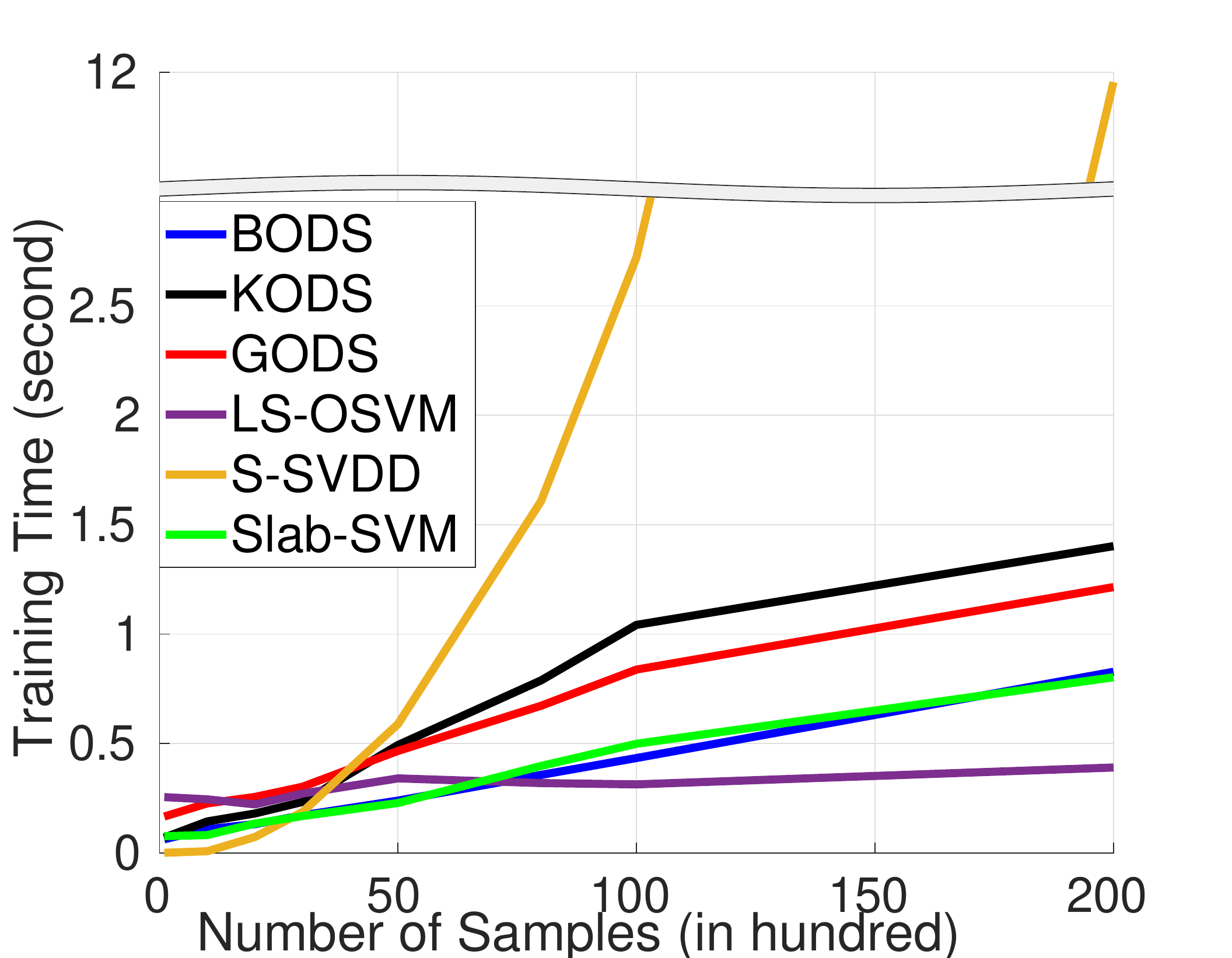}}
	\end{center}
	\caption{(a--d) show optimization convergence on JHMDB dataset. (a-b) GODS~\eqref{Problem3} against the choice of optimization schemes, (c-d) KODS~\eqref{eq:kgods-5} under different kernels (using CG optimizer). (e) compares running time.}
    \label{opt}
\end{figure*}

\subsection{State-of-the-Art Comparisons} In Table~\ref{soa}, we compare our variants to the state-of-the-art methods. As alluded to earlier, for our Dash-Cam-Pose dataset, as its positive and negative classes are imbalanced, we resort to reporting the $\overline{F1}$ score on the test set. From the table, our variants are seen to outperform prior methods by a significant margin; especially our GODS and KODS schemes demonstrate the best performances on different tasks. For example, using TCN, KODS  outperforms other kernelized prior variants by over 20\%. Similarly, on the JHMDB dataset, both GODS and KODS are better than the next best kernel-based method (K-OC-SVM) by about 20\%, and improves the classification accuracy by over 30\%. Overall, the experiments clearly substantiate the performance benefits afforded by our methods on the one-class task. 

In Table~\ref{tab:ucf-soa}-left, we present results against the state of the art on the UCF-Crime dataset using the AUC metric and false alarm rates; we use the standard threshold of 50\%. While, our results are lower than~\cite{sultani2018real} by 4\% in AUC and 0.2 larger in false alarm rate, their problem setup is completely different from ours in that they use weakly labeled abnormal videos as well in their training, which we do not use and which as per definition is not a one-class problem. Thus, our results are incomparable to theirs. Against other methods on this dataset, our schemes are about 5-10\% better. In the Table~\ref{tab:ucf-soa}-right, we provide the performance (AUC) on the UCSD Ped2 dataset. Compared with the recent state-of-the-art methods, both GODS and KODS achieves similar performances using the 3D autoencoder features. Among these methods, Luo et al.~\cite{luo2017remembering} and Liu et al.~\cite{luo2017revisit} propose ConvLSTM-AE and S-RNN respectively, which rely on the recurrent neural networks, that might be hard to train. Abati et al.~\cite{abati2019latent} also uses the 3D autoencoder features similar to ours, but also employs additional constraints for building the Conditional Probability Density (CPD), which is more expensive compared to our solution. We also note that Table~\ref{tab:ucf-soa}-right lists prior methods that use deep learning models, such as the ConvLSTM autoencoder~\cite{luo2017remembering}, Stacked-RNN~\cite{luo2017revisit}, and GANs~\cite{ravanbakhsh2018plug}; our GODS variants offer competitive performances against them. 

 \begin{table}[]
 \centering
 \caption{Average performances on the Dash-Cam-Pose and JHMDB datasets. Dash-Cam-Pose uses the $\overline{F1}$ score while JHMDB uses $F1$ score as evaluation metric (classification accuracy is shown in the brackets). K-OC-SVM and K-SVDD are the RBF kernelized variants.}
 \label{soa}
\begin{tabular}{l|l|l|l}\hline
Method   & CarPose\_BOW          & CarPose\_TCN          & JHMDB                             \\\hline
OC-SVM~\cite{scholkopf2001estimating}   & 0.167 (0.517)                & 0.279(0.527)                & 0.301 (0.568)               \\\hline
SVDD~\cite{tax2004support}     & 0.448 (0.489)                & 0.477(0.482)                & 0.407 (0.566)              \\\hline
K-OC-SVM~\cite{scholkopf2001estimating} & 0.327 (0.495)                & 0.361(0.491)                & 0.562 (0.412)          \\\hline
K-SVDD~\cite{tax2004support}    & 0.476 (0.477)                & 0.489 (0.505)                & 0.209 (0.441)                \\\hline
K-PCA~\cite{hoffmann2007kernel}    &  0.145 (0.502) & 0.258 (0.492) & 0.245 (0.557) \\\hline
Slab-SVM~\cite{fragoso2016one} & 0.468 (0.568) & 0.498 (0.577) & 0.643 (0.637) \\\hline
KNFST~\cite{bodesheim2013kernel} & 0.345 (0.487) & 0.368 (0.496) & 0.667 (0.501) \\\hline
KNN~\cite{hautamaki2004outlier} & 0.232 (0.475) & 0.276 (0.488) & 0.643 (0.492) \\\hline
LS-OSVM~\cite{choi2009least}  & 0.234 (0.440)                & 0.246(0.460)                & 0.663(0.582)               \\\hline
S-SVDD~\cite{sohrab2018subspace}   & 0.325 (0.490)           & 0.464 (0.500)           & 0.642 (0.498)         \\\hline\hline
BODS     & 0.523 (0.582)                & 0.532 (0.579)                & 0.725 (0.714)              \\\hline
GODS     & 0.563 (0.629)          & 0.584 (0.601)           & 0.777 (0.752)   \\\hline
KODS     & \textbf{0.596 (0.642)}          & \textbf{0.664 (0.604) }          & \textbf{0.785 (0.726)}
\end{tabular}
\end{table}

\begin{table}[htbp]
\centering
\caption{Performances on UCF-Crime dataset (left) and UCSD Ped2 dataset (right). $^*$Setup is different. }
\begin{tabular}{l|c|c}
\hline
 \multicolumn{3}{c}{\textbf{UCF Crime Dataset}}\\
\hline
\textbf{Method}       & \textbf{AUC} & \textbf{FAR} \\\hline
Random       & 0.50                    & -                                    \\\hline
Hasan et al.~\cite{hasan2016learning} & 0.51                    & 27.2                                 \\\hline
Lu et al.~\cite{lu2013abnormal}    & 0.66                   & 3.1                                  \\\hline
$^*$Waqas et al.~\cite{sultani2018real} & 0.75                   & 1.9                                  \\\hline
Sohrab et al.~\cite{sohrab2018subspace} &0.59 &10.5 \\\hline\hline
BODS & 0.68               &2.7      \\\hline           
GODS        & 0.70                   & 2.1         \\\hline
KODS    & \textbf{0.71}                   & \textbf{2.1} 
\end{tabular}
\begin{tabular}{l|c}
\hline
\multicolumn{2}{c}{\textbf{UCSD Ped2 Dataset}}\\
\hline
\textbf{Method}       & \textbf{AUC}  \\\hline
ConvLSTM-AE.~\cite{luo2017remembering}  &0.88\\\hline
GANs~\cite{ravanbakhsh2018plug} & 0.88\\\hline
S-RNN~\cite{luo2017revisit}    & 0.92                                                 \\\hline
Autoregression~\cite{abati2019latent} & \textbf{0.95}                                               \\\hline
FFP+MC, Liu et al.~\cite{liu2018future} & \textbf{0.95} \\\hline\hline
BODS & 0.91                \\      \hline    
GODS        & 0.93                        \\\hline
KODS    & \textbf{0.95}                 
\end{tabular}
\label{tab:ucf-soa}
\end{table}

\begin{table*}[]
 \centering
 \caption{Performances on UCI datasets. $N$ is number of samples, $D$ is the feature dimension and $T$ is the target (positive) class.}
\begin{tabular}{|c|l|l|l|l|c|c|c|c|c|c|c|}\hline
No. & Dataset  & N      & D          & T    & OC-SVM~\cite{sohrab2018subspace}  & K-OC-SVM  & SVDD~\cite{sohrab2018subspace}     & S-SVDD~\cite{sohrab2018subspace} & K-SVDD    & GODS & KODS      \\ \hline
1   & Sonar    & 208    & 60         & Mines         & 53.5\% & 56.5\% & 62.5\% & 63.8\% & 60.9\%  & 71.6\% & 72.5\%  \\ \hline
2   & Pump     & 1500   & 64         & Normal        & 63.2\% & 60.1\% & 84.6\% & 85.7\% & 83.6\%  & 87.6\% & 88.2\%   \\ \hline
3   & Scale    & 625    & 4          & Left          & 68.8\% & 67.6\% & 70.3\% & 90.7\% & 73.4\%  & 92.3\% & 92.5\%  \\ \hline
4   & Survival   & 306    & 3          & Survived     & 64.4\% & 74.3\% & 83.4\% & 84.1\% & 83.5\%  & 87.6\% & 88.1\% \\\hline
5   & Banknote  & 1372    & 5         & No          & 65.7\% & 70.0\% & 76.4\% & 90.8\% & 80.4\% & 94.7\% & 95.8\% \\\hline
\end{tabular}
\label{uci}
\end{table*}

\begin{table}[]
\centering

\caption{Comparisons of GODS variants on UCI datasets. We compare: (i) $\gods$~\eqref{Problem3} using the Stiefel manifold, (ii) $\gods_N$ using the non-compact Stiefel~\eqref{eq:ncstiefel}, (iii) the Euclidean~\eqref{eq:euc}, and (iv) the oblique manifolds~\eqref{eq:gods0}. We compare under (i) $\ell_2$ unit-normalization of inputs and (ii) $C$, under soft-orthogonality (~\eqref{eq:gods0},~\eqref{eq:euc}). We report F1 scores (in \%) and standard deviations over 5 trials.}

\begin{tabular}{|p{0.5mm}|p{8mm}|c|c|c|c|c|}
\hline
                           & Type & Sonar               & Pump               & Scale & Survival & Banknote \\\hline
            \rottxt{$\gods$} & $\ell_2$     & \textbf{71.6$\pm$2.6}      & \textbf{87.6$\pm$0.6}    & 89.2$\pm$5.4     & 87.6$\pm$1.9      & 94.7$\pm$3.8        \\
                             & $\neq\!\!\ell_2$   & 69.2$\pm$3.9      & 85.5$\pm$0.9     & 92.3$\pm$4.5     & 86.5$\pm$8.9      & 90.5$\pm$2.1       \\
                             & $\neq\!\!\ell_2\!\!+\!\!C$ & 68.3$\pm$4.1      & 85.2$\pm$0.8     & 92.2$\pm$4.1     & 87.0$\pm$9.5      & 90.4$\pm$2.1        \\
                             & $\ell_2\!\!+\!\!C$     & 71.6$\pm$3.7      & 87.0$\pm$0.6     & 88.7$\pm$3.6     & \textbf{87.8$\pm$1.8}      & 95.1$\pm$1.7        \\\hline\hline
   
            \rottxt{$\gods_N$}& $\ell_2$     & 69.2$\pm$8.6      & 86.7$\pm$7.1     & 84.8$\pm$7.0     & 85.3$\pm$0.3      & 90.9$\pm$6.2        \\
                              & $\neq\!\!\ell_2$ & 68.7$\pm$1.5      & 85.9$\pm$5.1     & 89.7$\pm$9.1     & 82.9$\pm$4.8      & 82.4$\pm$2.6        \\
                              & $\neq\!\!\ell_2\!\!+\!\!C$& 66.9$\pm$2.2      & 85.7$\pm$4.1     & 89.5$\pm$9.5     & 85.2$\pm$7.2      & 88.4$\pm$4.5        \\
                              & $\ell_2\!\!+\!\!C$   & 69.1$\pm$5.1      & 86.4$\pm$0.6     & 86.9$\pm$4.2     & 86.3$\pm$7.8      & 90.9$\pm$6.4        \\\hline\hline                        
          \rottxt{$\gods_E$} & $\ell_2$     & 66.7$\pm$2.3      & 85.9$\pm$0.3     & 79.8$\pm$5.1     & 84.8$\pm$1.7      & 93.8$\pm$6.7        \\
                              & $\neq\!\!\ell_2$ & 68.1$\pm$4.5      & 85.8$\pm$0.6     & 94.4$\pm$6.5     & 85.9$\pm$1.1      & 90.6$\pm$3.0        \\
                              & $\neq\!\!\ell_2\!\!+\!\!C$& 68.2$\pm$3.0      & 85.9$\pm$0.6     & 95.0$\pm$2.7     & 86.5$\pm$7.5      & 91.4$\pm$2.2        \\
                              & $\ell_2\!\!+\!\!C$   & 68.7$\pm$2.4      & 86.2$\pm$0.2     & 81.0$\pm$3.9     & 86.5$\pm$1.2      & 94.1$\pm$1.5        \\\hline\hline
           \rottxt{$\gods_O$}& $\ell_2$       & 69.1$\pm$3.5      & 87.0$\pm$0.6     & 81.6$\pm$3.7       & 85.7$\pm$6.1      & 94.0$\pm$1.2        \\
                              & $\neq\!\!\ell_2$   & 69.3$\pm$3.5      & 84.6$\pm$1.2     & 97.1$\pm$3.6     & 85.5$\pm$1.5      & 94.1$\pm$3.9        \\
                              & $\neq\!\!\ell_2\!\!+\!\!C$ & 70.5$\pm$4.6      & 86.5$\pm$0.7     & \textbf{97.1$\pm$2.9}     & 85.9$\pm$8.7      & 95.3$\pm$3.2        \\
                              & $\ell_2\!\!+\!\!C$     & 70.6$\pm$4.0      & 87.3$\pm$0.7     & 82.3$\pm$4.3     & 85.9$\pm$3.3      & \textbf{95.6$\pm$3.1}      \\\hline
 
\end{tabular}
\label{uci2}
\end{table}

\subsection{Performance on UCI datasets}
\label{sec:uci}
As the reader might acknowledge, the algorithms proposed in this paper are not specialized to only computer vision datasets, but could be applied for the anomaly detection task on any data mining, machine learning, or robotics task. To this end, in Table~\ref{uci}, we evaluate GODS and KODS on five datasets downloaded from UCI datasets\footnote{http://archive.ics.uci.edu/ml/index.php} and TU delft pattern recognition lab website\footnote{http://homepage.tudelft.nl/n9d04/occ/index.html}. These datasets are: (i) sonar, the task in which is to discriminate between sonar signals bounced off a metal cylinder and those bounced off a roughly cylindrical rock, (ii) the Delft pump dataset, the task in which is to detect abnormal condition of a submersible pump, (iii) the Scale dataset is to classify the balance scale tip to the right, tip to the left, or be balanced, (iv) the Haberman's survival dataset records the survival of patients who had undergone surgery for breast cancer, and (v) the Banknote dataset is to detect if the feature from an image passing the evaluation of an authentication procedure for banknotes.

We follow the evaluation protocol in the recent paper~\cite{sohrab2018subspace} for all the datasets and compare the performances to those reported in that paper. Specifically, the evaluation uses a split of 70/30 for the positive class; the model training is performed on the 70\%, and tested on the remaining 30\% positive class data and the negative (anomalous) data (which is not used in training). We repeat the split in the positive class five times and report the average performance on the five trials. For datasets having more than two classes, we pick one class as positive, while the remaining as negative, and follow the same protocol as above. In Table~\ref{uci}, we report the performances of GODS and KODS against those reported in~\cite{sohrab2018subspace}. In the table,  $N$ represents the number of samples in the dataset, $D$ denotes the feature dimension of each sample, and $T$ is the target class picked as positive (same as in~\cite{sohrab2018subspace}). We use three hyperplanes in the GODS subspaces, and set the sensitivity margin $\eta=0.3$ in both GODS and KODS. For KODS, we use polynomial kernel with degree as 3. As is clear from the table, we outperform the previous state-of-the-art results on all datasets. Specifically, GODS is substantially better than the previous best method S-SVDD by 2--8\%, and the KODS is even better by 1--2\%. This is because our GODS algorithm better characterizes the data distribution from positive classes and thus producing higher cost for anomalies during the inference. For KODS, the kernel embedding would further bring advantages in learning the decision regions better fitting the normal samples.

In the Table~\ref{uci2}, we evaluate the various extensions of GODS as described in Section~\ref{gods_extension}. Through these experiments, we evaluate the impact of unit-normalization on the data inputs and the orthogonality assumptions on the hyperplanes, and analyze the adequacy of each variant when such assumptions may not be relevant. In the Table~\ref{uci2}, each column contains the results for one dataset while the four rows in one column are the result for one variant of GODS. From top to the bottom, we have 4 settings relaxing different constraints; they are: 1) unit normalization applied on the inputs, denoted $\ell_2$, 2) no unit norm is enforced, $\neq\ell_2$, 3) $\neq\ell_2$ but with soft-orthogonality constraints ($C$) as described in~\eqref{eq:gods0},~\eqref{eq:euc}, and 4) $\ell_2$ norm and $C$ are used together. We also report the standard deviations associated with each experiment over the five trials. 

From the experimental results, it is found that the orthogonal constraint is generally helpful, whenever applicable. For instance, the results in the third and fourth row in Euclidean and Oblique manifolds are better than the ones without orthogonal constraints by up to 2\%. In terms of the $\ell_2$ norm constraints, it depends on the nature of the data points. For example, in the Scale dataset, each dimension of the data captures the presence of some semantic attribute and thus $\ell_2$ norm may not make much sense on them. However, for the vision datasets or the other four UCI datasets, the feature normalization could  bound the data allowing the one-class model to better capture the distribution.

\section{Conclusions}
\label{sec:conclude}
In this paper, we presented a novel one-class learning formulation -- called GODS -- using pairs of complementary classifiers; these classifiers are oriented so as to circumscribe the data density within a rectilinear space of minimal volume. We explored variants of GODS via relaxing the various constraints in our problem setup, as well as introducing kernel feature maps. Due to the orthonormality we impose on the classifiers, our objectives are non-convex, and solving for which we resorted to Riemannian optimization frameworks on the Stiefel manifold and its variants. We presented experiments on a diverse set of anomaly detection tasks, demonstrating state-of-the-art performances. We further analyzed the generalizability of our framework to non-vision data by presenting experiments on five UCI datasets; our results outperforming prior baselines by significant margins. 

 An potential direction to extend this work is perhaps to use more than two classifiers in the GODS framework. While, we experimented with a variant of this idea using multiple orthonormal frames, its performance was poor. We presume this inferior performance is perhaps due to the lack of appropriate regularizations across the classifiers and the absence of suitable complementarity conditions between the classifiers and the data. We plan to pursue this research direction in a future paper. 

Further, there are several aspects of our scheme that needs rigorous treatment. For example, deriving generalization bounds on GODS is one such. Analysis of the representation complexity within a computational learning theory framework is yet another direction. An analysis of our optimization landscape is a direction that could help better initialize our schemes. Extending our framework as a module within an end-to-end neural network is an interesting direction as well. 

\noindent\textbf{Acknowledgements.} We thank the anonymous reviewers whose insightful questions have helped us significantly improve the quality of this work.

{\small
\bibliographystyle{IEEEtran}
\bibliography{one-class-bib}

\begin{thebibliography}{100}
\providecommand{\url}[1]{#1}
\csname url@samestyle\endcsname
\providecommand{\newblock}{\relax}
\providecommand{\bibinfo}[2]{#2}
\providecommand{\BIBentrySTDinterwordspacing}{\spaceskip=0pt\relax}
\providecommand{\BIBentryALTinterwordstretchfactor}{4}
\providecommand{\BIBentryALTinterwordspacing}{\spaceskip=\fontdimen2\font plus
\BIBentryALTinterwordstretchfactor\fontdimen3\font minus
  \fontdimen4\font\relax}
\providecommand{\BIBforeignlanguage}[2]{{%
\expandafter\ifx\csname l@#1\endcsname\relax
\typeout{** WARNING: IEEEtran.bst: No hyphenation pattern has been}%
\typeout{** loaded for the language `#1'. Using the pattern for}%
\typeout{** the default language instead.}%
\else
\language=\csname l@#1\endcsname
\fi
#2}}
\providecommand{\BIBdecl}{\relax}
\BIBdecl

\bibitem{chandala2009anomaly}
V.~Chandala, A.~Banerjee, and V.~Kumar, ``Anomaly detection: A survey, {ACM}
  computing surveys,'' \emph{University of Minnesota}, 2009.

\bibitem{bishop1994novelty}
C.~M. Bishop, ``Novelty detection and neural network validation,''
  \emph{VISSP}, vol. 141, no.~4, pp. 217--222, 1994.

\bibitem{ritter1997outliers}
G.~Ritter and M.~T. Gallegos, ``Outliers in statistical pattern recognition and
  an application to automatic chromosome classification,'' \emph{Pattern
  Recognition Letters}, vol.~18, no.~6, pp. 525--539, 1997.

\bibitem{gardner2006one}
A.~B. Gardner, A.~M. Krieger, G.~Vachtsevanos, and B.~Litt, ``One-class novelty
  detection for seizure analysis from intracranial eeg,'' \emph{JMLR}, vol.~7,
  no. Jun, pp. 1025--1044, 2006.

\bibitem{khan2009survey}
S.~S. Khan and M.~G. Madden, ``A survey of recent trends in one class
  classification,'' in \emph{Irish conference on artificial intelligence and
  cognitive science}.\hskip 1em plus 0.5em minus 0.4em\relax Springer, 2009,
  pp. 188--197.

\bibitem{sabokrou2018adversarially}
M.~Sabokrou, M.~Khalooei, M.~Fathy, and E.~Adeli, ``Adversarially learned
  one-class classifier for novelty detection,'' in \emph{CVPR}, 2018, pp.
  3379--3388.

\bibitem{del2016discriminative}
A.~Del~Giorno, J.~A. Bagnell, and M.~Hebert, ``A discriminative framework for
  anomaly detection in large videos,'' in \emph{ECCV}, 2016.

\bibitem{saligrama2012video}
V.~Saligrama and Z.~Chen, ``Video anomaly detection based on local statistical
  aggregates,'' in \emph{CVPR}.\hskip 1em plus 0.5em minus 0.4em\relax IEEE,
  2012, pp. 2112--2119.

\bibitem{liu2018future}
W.~Liu, W.~Luo, D.~Lian, and S.~Gao, ``Future frame prediction for anomaly
  detection--a new baseline,'' in \emph{CVPR}, 2018, pp. 6536--6545.

\bibitem{buades2005non}
A.~Buades, B.~Coll, and J.-M. Morel, ``A non-local algorithm for image
  denoising,'' in \emph{CVPR}.\hskip 1em plus 0.5em minus 0.4em\relax IEEE,
  2005, pp. 60--65.

\bibitem{xia2015learning}
Y.~Xia, X.~Cao, F.~Wen, G.~Hua, and J.~Sun, ``Learning discriminative
  reconstructions for unsupervised outlier removal,'' in \emph{ICCV}, 2015, pp.
  1511--1519.

\bibitem{you2017provable}
C.~You, D.~P. Robinson, and R.~Vidal, ``Provable self-representation based
  outlier detection in a union of subspaces,'' in \emph{CVPR}, 2017, pp.
  3395--3404.

\bibitem{scholkopf2001estimating}
B.~Sch{\"o}lkopf, J.~C. Platt, J.~Shawe-Taylor, A.~J. Smola, and R.~C.
  Williamson, ``Estimating the support of a high-dimensional distribution,''
  \emph{Neural computation}, vol.~13, no.~7, pp. 1443--1471, 2001.

\bibitem{choi2009least}
Y.-S. Choi, ``Least squares one-class support vector machine,'' \emph{Pattern
  Recognition Letters}, vol.~30, no.~13, pp. 1236--1240, 2009.

\bibitem{wang2013online}
T.~Wang, J.~Chen, Y.~Zhou, and H.~Snoussi, ``Online least squares one-class
  support vector machines-based abnormal visual event detection,''
  \emph{Sensors}, vol.~13, no.~12, pp. 17\,130--17\,155, 2013.

\bibitem{tax2004support}
D.~M. Tax and R.~P. Duin, ``Support vector data description,'' \emph{Machine
  learning}, vol.~54, no.~1, pp. 45--66, 2004.

\bibitem{abati2019latent}
D.~Abati, A.~Porrello, S.~Calderara, and R.~Cucchiara, ``Latent space
  autoregression for novelty detection,'' in \emph{CVPR}, 2019, pp. 481--490.

\bibitem{chalapathy2018anomaly}
R.~Chalapathy, A.~K. Menon, and S.~Chawla, ``Anomaly detection using one-class
  neural networks,'' \emph{arXiv preprint arXiv:1802.06360}, 2018.

\bibitem{luo2017revisit}
W.~Luo, W.~Liu, and S.~Gao, ``A revisit of sparse coding based anomaly
  detection in stacked {RNN} framework,'' in \emph{ICCV}, 2017, pp. 341--349.

\bibitem{hasan2016learning}
M.~Hasan, J.~Choi, J.~Neumann, A.~K. Roy-Chowdhury, and L.~S. Davis, ``Learning
  temporal regularity in video sequences,'' in \emph{CVPR}, 2016, pp. 733--742.

\bibitem{schlegl2017unsupervised}
T.~Schlegl, P.~Seeb{\"o}ck, S.~M. Waldstein, U.~Schmidt-Erfurth, and G.~Langs,
  ``Unsupervised anomaly detection with generative adversarial networks to
  guide marker discovery,'' in \emph{IPMI}.\hskip 1em plus 0.5em minus
  0.4em\relax Springer, 2017, pp. 146--157.

\bibitem{ravanbakhsh2019training}
M.~Ravanbakhsh, E.~Sangineto, M.~Nabi, and N.~Sebe, ``Training adversarial
  discriminators for cross-channel abnormal event detection in crowds,'' in
  \emph{WACV}.\hskip 1em plus 0.5em minus 0.4em\relax IEEE, 2019, pp.
  1896--1904.

\bibitem{wang2018learning}
J.~Wang and A.~Cherian, ``Learning discriminative video representations using
  adversarial perturbations,'' in \emph{ECCV}, 2018.

\bibitem{edelman1998geometry}
A.~Edelman, T.~A. Arias, and S.~T. Smith, ``The geometry of algorithms with
  orthogonality constraints,'' \emph{SIAM journal on Matrix Analysis and
  Applications}, vol.~20, no.~2, pp. 303--353, 1998.

\bibitem{boothby1986introduction}
W.~M. Boothby, \emph{An introduction to differentiable manifolds and Riemannian
  geometry}.\hskip 1em plus 0.5em minus 0.4em\relax Academic press, 1986, vol.
  120.

\bibitem{absil2009optimization}
P.-A. Absil, R.~Mahony, and R.~Sepulchre, \emph{Optimization algorithms on
  matrix manifolds}.\hskip 1em plus 0.5em minus 0.4em\relax Princeton
  University Press, 2009.

\bibitem{trivedi2007looking}
M.~M. Trivedi, T.~Gandhi, and J.~McCall, ``Looking-in and looking-out of a
  vehicle: Computer-vision-based enhanced vehicle safety,'' \emph{IEEE TITS},
  vol.~8, no.~1, pp. 108--120, 2007.

\bibitem{trivedi2004occupant}
M.~M. Trivedi, S.~Y. Cheng, E.~M. Childers, and S.~J. Krotosky, ``Occupant
  posture analysis with stereo and thermal infrared video: Algorithms and
  experimental evaluation,'' \emph{IEEE Transactions on Vehicular Technology},
  vol.~53, no.~6, pp. 1698--1712, 2004.

\bibitem{sultani2018real}
W.~Sultani, C.~Chen, and M.~Shah, ``Real-world anomaly detection in
  surveillance videos,'' in \emph{CVPR}, 2018, pp. 6479--6488.

\bibitem{li2013anomaly}
W.~Li, V.~Mahadevan, and N.~Vasconcelos, ``Anomaly detection and localization
  in crowded scenes,'' \emph{TPAMI}, vol.~36, no.~1, pp. 18--32, 2013.

\bibitem{jhuang2013towards}
H.~Jhuang, J.~Gall, S.~Zuffi, C.~Schmid, and M.~J. Black, ``Towards
  understanding action recognition,'' in \emph{ICCV}, 2013.

\bibitem{GODS}
J.~Wang and A.~Cherian, ``{GODS:} generalized one-class discriminative
  subspaces for anomaly detection,'' in \emph{ICCV}, 2019.

\bibitem{matteoli2014overview}
S.~Matteoli, M.~Diani, and J.~Theiler, ``An overview of background modeling for
  detection of targets and anomalies in hyperspectral remotely sensed
  imagery,'' \emph{IEEE Journal of Selected Topics in Applied Earth
  Observations and Remote Sensing}, vol.~7, no.~6, pp. 2317--2336, 2014.

\bibitem{lazarevic2003comparative}
A.~Lazarevic, L.~Ertoz, V.~Kumar, A.~Ozgur, and J.~Srivastava, ``A comparative
  study of anomaly detection schemes in network intrusion detection,'' in
  \emph{ICDM}.\hskip 1em plus 0.5em minus 0.4em\relax SIAM, 2003.

\bibitem{heller2003one}
K.~A. Heller, K.~M. Svore, A.~D. Keromytis, and S.~J. Stolfo, ``One class
  support vector machines for detecting anomalous windows registry accesses,''
  in \emph{Workshop on Data Mining for Computer Security}, 2003.

\bibitem{popoola2012video}
O.~P. Popoola and K.~Wang, ``Video-based abnormal human behavior
  recognition—a review,'' \emph{IEEE Transactions on Systems, Man, and
  Cybernetics}, vol.~42, no.~6, pp. 865--878, 2012.

\bibitem{krawczyk2015one}
B.~Krawczyk, ``One-class classifier ensemble pruning and weighting with firefly
  algorithm,'' \emph{Neurocomputing}, vol. 150, pp. 490--500, 2015.

\bibitem{saleh2013object}
B.~Saleh, A.~Farhadi, and A.~Elgammal, ``Object-centric anomaly detection by
  attribute-based reasoning,'' in \emph{CVPR}, 2013, pp. 787--794.

\bibitem{pimentel2014review}
M.~A. Pimentel, D.~A. Clifton, L.~Clifton, and L.~Tarassenko, ``A review of
  novelty detection,'' \emph{Signal Processing}, vol.~99, pp. 215--249, 2014.

\bibitem{pang2020deep}
G.~Pang, C.~Shen, L.~Cao, and A.~v.~d. Hengel, ``Deep learning for anomaly
  detection: A review,'' \emph{arXiv preprint arXiv:2007.02500}, 2020.

\bibitem{tsybakov1997nonparametric}
A.~B. Tsybakov \emph{et~al.}, ``On nonparametric estimation of density level
  sets,'' \emph{The Annals of Statistics}, vol.~25, no.~3, pp. 948--969, 1997.

\bibitem{nolan1991excess}
D.~Nolan, ``The excess-mass ellipsoid,'' \emph{Journal of multivariate
  analysis}, vol.~39, no.~2, pp. 348--371, 1991.

\bibitem{dalal2005histograms}
N.~Dalal and B.~Triggs, ``Histograms of oriented gradients for human
  detection,'' in \emph{ECCV}, 2005.

\bibitem{dalal2006human}
N.~Dalal, B.~Triggs, and C.~Schmid, ``Human detection using oriented histograms
  of flow and appearance,'' in \emph{ECCV}.\hskip 1em plus 0.5em minus
  0.4em\relax Springer, 2006, pp. 428--441.

\bibitem{zhang2005semi}
D.~Zhang, D.~Gatica-Perez, S.~Bengio, and I.~McCowan, ``Semi-supervised adapted
  hmms for unusual event detection,'' in \emph{CVPR}.\hskip 1em plus 0.5em
  minus 0.4em\relax IEEE, 2005, pp. 611--618.

\bibitem{xu2010robust}
H.~Xu, C.~Caramanis, and S.~Sanghavi, ``Robust {PCA} via outlier pursuit,'' in
  \emph{NIPS}, 2010, pp. 2496--2504.

\bibitem{kim2009observe}
J.~Kim and K.~Grauman, ``Observe locally, infer globally: a space-time mrf for
  detecting abnormal activities with incremental updates,'' in
  \emph{CVPR}.\hskip 1em plus 0.5em minus 0.4em\relax IEEE, 2009.

\bibitem{choi2012context}
M.~J. Choi, A.~Torralba, and A.~S. Willsky, ``Context models and out-of-context
  objects,'' \emph{Pattern Recognition Letters}, vol.~33, no.~7, pp. 853--862,
  2012.

\bibitem{park2012abnormal}
S.~Park, W.~Kim, and K.~M. Lee, ``Abnormal object detection by canonical
  scene-based contextual model,'' in \emph{ECCV}, 2012.

\bibitem{wu2010chaotic}
S.~Wu, B.~E. Moore, and M.~Shah, ``Chaotic invariants of lagrangian particle
  trajectories for anomaly detection in crowded scenes,'' in \emph{CVPR}.\hskip
  1em plus 0.5em minus 0.4em\relax IEEE, 2010, pp. 2054--2060.

\bibitem{calderara2011detecting}
S.~Calderara, U.~Heinemann, A.~Prati, R.~Cucchiara, and N.~Tishby, ``Detecting
  anomalies in people’s trajectories using spectral graph analysis,''
  \emph{CVIU}, vol. 115, no.~8, pp. 1099--1111, 2011.

\bibitem{tung2011goal}
F.~Tung, J.~S. Zelek, and D.~A. Clausi, ``Goal-based trajectory analysis for
  unusual behaviour detection in intelligent surveillance,'' \emph{Image and
  Vision Computing}, vol.~29, no.~4, pp. 230--240, 2011.

\bibitem{itti2000saliency}
L.~Itti and C.~Koch, ``A saliency-based search mechanism for overt and covert
  shifts of visual attention,'' \emph{Vision research}, vol.~40, no. 10-12, pp.
  1489--1506, 2000.

\bibitem{judd2009learning}
T.~Judd, K.~Ehinger, F.~Durand, and A.~Torralba, ``Learning to predict where
  humans look,'' in \emph{ICCV}, 2009.

\bibitem{cong2011sparse}
Y.~Cong, J.~Yuan, and J.~Liu, ``Sparse reconstruction cost for abnormal event
  detection,'' in \emph{CVPR}.\hskip 1em plus 0.5em minus 0.4em\relax IEEE,
  2011, pp. 3449--3456.

\bibitem{lu2013abnormal}
C.~Lu, J.~Shi, and J.~Jia, ``Abnormal event detection at 150 fps in {M}atlab,''
  in \emph{CVPR}, 2013, pp. 2720--2727.

\bibitem{zhao2011online}
B.~Zhao, L.~Fei-Fei, and E.~P. Xing, ``Online detection of unusual events in
  videos via dynamic sparse coding,'' in \emph{CVPR}.\hskip 1em plus 0.5em
  minus 0.4em\relax IEEE, 2011, pp. 3313--3320.

\bibitem{li2013visual}
C.~Li, Z.~Han, Q.~Ye, and J.~Jiao, ``Visual abnormal behavior detection based
  on trajectory sparse reconstruction analysis,'' \emph{Neurocomputing}, vol.
  119, pp. 94--100, 2013.

\bibitem{ren2016comprehensive}
H.~Ren, H.~Pan, S.~I. Olsen, and T.~B. Moeslund, ``A comprehensive study of
  sparse codes on abnormality detection,'' \emph{arXiv preprint
  arXiv:1603.04026}, 2016.

\bibitem{deng2009imagenet}
J.~Deng, W.~Dong, R.~Socher, L.-J. Li, K.~Li, and L.~Fei-Fei, ``{ImageNet}: A
  large-scale hierarchical image database,'' in \emph{CVPR}.\hskip 1em plus
  0.5em minus 0.4em\relax IEEE, 2009, pp. 248--255.

\bibitem{girshick2015fast}
R.~Girshick, ``{Fast R-cnn},'' in \emph{ICCV}, 2015.

\bibitem{ruff2018deep}
L.~Ruff, N.~Goernitz, L.~Deecke, S.~A. Siddiqui, R.~Vandermeulen, A.~Binder,
  E.~M{\"u}ller, and M.~Kloft, ``Deep one-class classification,'' in
  \emph{ICML}, 2018.

\bibitem{xu2017detecting}
D.~Xu, Y.~Yan, E.~Ricci, and N.~Sebe, ``Detecting anomalous events in videos by
  learning deep representations of appearance and motion,'' \emph{CVIU}, vol.
  156, pp. 117--127, 2017.

\bibitem{perera2018learning}
P.~Perera and V.~M. Patel, ``Learning deep features for one-class
  classification,'' \emph{arXiv preprint arXiv:1801.05365}, 2018.

\bibitem{liang2017enhancing}
S.~Liang, Y.~Li, and R.~Srikant, ``Enhancing the reliability of
  out-of-distribution image detection in neural networks,'' \emph{arXiv
  preprint arXiv:1706.02690}, 2017.

\bibitem{xu2015learning}
D.~Xu, E.~Ricci, Y.~Yan, J.~Song, and N.~Sebe, ``Learning deep representations
  of appearance and motion for anomalous event detection,'' in \emph{BMVC},
  2015.

\bibitem{chong2017abnormal}
Y.~S. Chong and Y.~H. Tay, ``Abnormal event detection in videos using
  spatiotemporal autoencoder,'' in \emph{International Symposium on Neural
  Networks}.\hskip 1em plus 0.5em minus 0.4em\relax Springer, 2017, pp.
  189--196.

\bibitem{luo2017remembering}
W.~Luo, W.~Liu, and S.~Gao, ``Remembering history with convolutional lstm for
  anomaly detection,'' in \emph{ICME}.\hskip 1em plus 0.5em minus 0.4em\relax
  IEEE, 2017.

\bibitem{lee2018simple}
K.~Lee, K.~Lee, H.~Lee, and J.~Shin, ``A simple unified framework for detecting
  out-of-distribution samples and adversarial attacks,'' in \emph{NIPS}, 2018.

\bibitem{principled_srikant}
S.~Liang, Y.~Li, and R.~Srikant, ``Principled detection of out-of-distribution
  examples in neural networks,'' in \emph{ICLR}, 2018.

\bibitem{sabokrou1609fully}
M.~Sabokrou, M.~Fayyaz, M.~Fathy, and R.~Klette, ``Fully convolutional neural
  network for fast anomaly detection in crowded scenes,'' \emph{arXiv preprint
  arXiv:1609.00866}, 2016.

\bibitem{ravanbakhsh2017abnormal}
M.~Ravanbakhsh, M.~Nabi, E.~Sangineto, L.~Marcenaro, C.~Regazzoni, and N.~Sebe,
  ``Abnormal event detection in videos using generative adversarial nets,'' in
  \emph{ICIP}.\hskip 1em plus 0.5em minus 0.4em\relax IEEE, 2017.

\bibitem{park2020learning}
H.~Park, J.~Noh, and B.~Ham, ``Learning memory-guided normality for anomaly
  detection,'' in \emph{Proceedings of the IEEE/CVF Conference on Computer
  Vision and Pattern Recognition}, 2020, pp. 14\,372--14\,381.

\bibitem{mohseni2020self}
S.~Mohseni, M.~Pitale, J.~Yadawa, and Z.~Wang, ``Self-supervised learning for
  generalizable out-of-distribution detection,'' in \emph{Proceedings of the
  AAAI Conference on Artificial Intelligence}, vol.~34, no.~04, 2020, pp.
  5216--5223.

\bibitem{zisselman2020deep}
E.~Zisselman and A.~Tamar, ``Deep residual flow for out of distribution
  detection,'' in \emph{Proceedings of the IEEE/CVF Conference on Computer
  Vision and Pattern Recognition}, 2020, pp. 13\,994--14\,003.

\bibitem{li2020efficient}
J.~Li, L.~Fuxin, and S.~Todorovic, ``Efficient {R}iemannian optimization on the
  stiefel manifold via the cayley transform,'' in \emph{ICLR}, 2020.

\bibitem{gould2016differentiating}
S.~Gould, B.~Fernando, A.~Cherian, P.~Anderson, R.~S. Cruz, and E.~Guo, ``On
  differentiating parameterized argmin and argmax problems with application to
  bi-level optimization,'' \emph{arXiv preprint arXiv:1607.05447}, 2016.

\bibitem{lai2015mixture}
V.~Lai, D.~Nguyen, K.~Nguyen, and T.~Le, ``Mixture of support vector data
  descriptions,'' in \emph{Conference on Information and Computer
  Science}.\hskip 1em plus 0.5em minus 0.4em\relax IEEE, 2015.

\bibitem{lee2007density}
K.~Lee, D.-W. Kim, K.~H. Lee, and D.~Lee, ``Density-induced support vector data
  description,'' \emph{IEEE Transactions on Neural Networks}, vol.~18, no.~1,
  pp. 284--289, 2007.

\bibitem{tax2001one}
D.~M.~J. Tax, ``One-class classification: concept-learning in the absence of
  counter-examples,'' Ph.D. dissertation, Delft University of Technology, 2001.

\bibitem{sohrab2018subspace}
F.~Sohrab, J.~Raitoharju, M.~Gabbouj, and A.~Iosifidis, ``Subspace support
  vector data description,'' in \emph{ICPR}.\hskip 1em plus 0.5em minus
  0.4em\relax IEEE, 2018.

\bibitem{candes2011robust}
E.~J. Cand{\`e}s, X.~Li, Y.~Ma, and J.~Wright, ``Robust principal component
  analysis?'' \emph{Journal of the ACM}, vol.~58, no.~3, p.~11, 2011.

\bibitem{de2003framework}
F.~De~La~Torre and M.~J. Black, ``A framework for robust subspace learning,''
  \emph{IJCV}, vol.~54, no. 1-3, pp. 117--142, 2003.

\bibitem{hoffmann2007kernel}
H.~Hoffmann, ``Kernel {PCA} for novelty detection,'' \emph{Pattern
  recognition}, vol.~40, no.~3, pp. 863--874, 2007.

\bibitem{nguyen2009robust}
M.~H. Nguyen and F.~Torre, ``Robust kernel principal component analysis,'' in
  \emph{NIPS}, 2009, pp. 1185--1192.

\bibitem{xu2013outlier}
H.~Xu, C.~Caramanis, and S.~Mannor, ``Outlier-robust {PCA}: the
  high-dimensional case,'' \emph{IEEE Transactions on Information Theory},
  vol.~59, no.~1, pp. 546--572, 2013.

\bibitem{bodesheim2013kernel}
P.~Bodesheim, A.~Freytag, E.~Rodner, M.~Kemmler, and J.~Denzler, ``Kernel null
  space methods for novelty detection,'' in \emph{CVPR}, 2013, pp. 3374--3381.

\bibitem{liu2014unsupervised}
W.~Liu, G.~Hua, and J.~R. Smith, ``Unsupervised one-class learning for
  automatic outlier removal,'' in \emph{CVPR}, 2014, pp. 3826--3833.

\bibitem{fragoso2016one}
V.~Fragoso, W.~Scheirer, J.~Hespanha, and M.~Turk, ``One-class slab support
  vector machine,'' in \emph{ICPR}.\hskip 1em plus 0.5em minus 0.4em\relax
  IEEE, 2016, pp. 420--425.

\bibitem{wang2018video}
J.~Wang, A.~Cherian, F.~Porikli, and S.~Gould, ``Video representation learning
  using discriminative pooling,'' in \emph{CVPR}, 2018.

\bibitem{lee2001ssvm}
Y.-J. Lee and O.~L. Mangasarian, ``Ssvm: A smooth support vector machine for
  classification,'' \emph{Computational optimization and Applications},
  vol.~20, no.~1, pp. 5--22, 2001.

\bibitem{mangasarian2001lagrangian}
O.~L. Mangasarian and D.~R. Musicant, ``Lagrangian support vector machines,''
  \emph{Journal of Machine Learning Research}, vol.~1, no. Mar, pp. 161--177,
  2001.

\bibitem{jegou2009burstiness}
H.~J{\'e}gou, M.~Douze, and C.~Schmid, ``On the burstiness of visual
  elements,'' in \emph{2009 IEEE conference on computer vision and pattern
  recognition}.\hskip 1em plus 0.5em minus 0.4em\relax IEEE, 2009, pp.
  1169--1176.

\bibitem{koniusz2016higher}
P.~Koniusz, F.~Yan, P.-H. Gosselin, and K.~Mikolajczyk, ``Higher-order
  occurrence pooling for bags-of-words: Visual concept detection,'' \emph{IEEE
  transactions on pattern analysis and machine intelligence}, vol.~39, no.~2,
  pp. 313--326, 2016.

\bibitem{graf2003classification}
A.~B. Graf, A.~J. Smola, and S.~Borer, ``Classification in a normalized feature
  space using support vector machines,'' \emph{IEEE Transactions on Neural
  Networks}, vol.~14, no.~3, pp. 597--605, 2003.

\bibitem{ranjan2017l2}
R.~Ranjan, C.~D. Castillo, and R.~Chellappa, ``L2-constrained softmax loss for
  discriminative face verification,'' \emph{arXiv preprint arXiv:1703.09507},
  2017.

\bibitem{cherian2018non}
A.~Cherian, S.~Sra, S.~Gould, and R.~Hartley, ``Non-linear temporal subspace
  representations for activity recognition,'' in \emph{Proceedings of the IEEE
  Conference on Computer Vision and Pattern Recognition}, 2018, pp. 2197--2206.

\bibitem{tax2000feature}
D.~Tax and R.~Duin, ``Feature scaling in support vector data descriptions,''
  \emph{Learning from Imbalanced Datasets}, pp. 25--30, 2000.

\bibitem{herbrich2001pac}
R.~Herbrich and T.~Graepel, ``A pac-bayesian margin bound for linear
  classifiers: Why svms work,'' \emph{Advances in neural information processing
  systems}, pp. 224--230, 2001.

\bibitem{smola1998learning}
A.~J. Smola and B.~Sch{\"o}lkopf, \emph{Learning with kernels}.\hskip 1em plus
  0.5em minus 0.4em\relax Citeseer, 1998, vol.~4.

\bibitem{absil2006joint}
P.-A. Absil and K.~A. Gallivan, ``Joint diagonalization on the oblique manifold
  for independent component analysis,'' in \emph{ICASSP}.\hskip 1em plus 0.5em
  minus 0.4em\relax IEEE, 2006.

\bibitem{wen2013feasible}
Z.~Wen and W.~Yin, ``A feasible method for optimization with orthogonality
  constraints,'' \emph{Mathematical Programming}, vol. 142, no. 1-2, pp.
  397--434, 2013.

\bibitem{muirhead2009aspects}
R.~J. Muirhead, \emph{Aspects of multivariate statistical theory}.\hskip 1em
  plus 0.5em minus 0.4em\relax John Wiley \& Sons, 2009, vol. 197.

\bibitem{higham2010canonical}
N.~J. Higham, C.~Mehl, and F.~Tisseur, ``The canonical generalized polar
  decomposition,'' \emph{SIAM Journal on Matrix Analysis and Applications},
  vol.~31, no.~4, pp. 2163--2180, 2010.

\bibitem{boumal2014manopt}
N.~Boumal, B.~Mishra, P.-A. Absil, and R.~Sepulchre, ``Manopt, a matlab toolbox
  for optimization on manifolds,'' \emph{JMLR}, vol.~15, no.~1, pp. 1455--1459,
  2014.

\bibitem{nordhoff2005motor}
L.~S. Nordhoff, \emph{Motor vehicle collision injuries: biomechanics,
  diagnosis, and management}.\hskip 1em plus 0.5em minus 0.4em\relax Jones \&
  Bartlett Learning, 2005.

\bibitem{duma1996airbag}
S.~M. Duma, T.~A. Kress, D.~J. Porta, C.~D. Woods, J.~N. Snider, P.~M. Fuller,
  and R.~J. Simmons, ``Airbag-induced eye injuries: a report of 25 cases,''
  \emph{Journal of Trauma and Acute Care Surgery}, vol.~41, no.~1, pp.
  114--119, 1996.

\bibitem{cao2016realtime}
Z.~Cao, T.~Simon, S.-E. Wei, and Y.~Sheikh, ``Realtime multi-person 2d pose
  estimation using part affinity fields,'' \emph{arXiv preprint
  arXiv:1611.08050}, 2016.

\bibitem{kim2017interpretable}
T.~S. Kim and A.~Reiter, ``Interpretable 3d human action analysis with temporal
  convolutional networks,'' in \emph{CVPRW}, 2017.

\bibitem{shahroudy2016ntu}
A.~Shahroudy, J.~Liu, T.-T. Ng, and G.~Wang, ``{NTU RGB+ D}: A large scale
  dataset for 3d human activity analysis,'' in \emph{CVPR}, 2016.

\bibitem{carreira2017quo}
J.~Carreira and A.~Zisserman, ``Quo vadis, action recognition? a new model and
  the kinetics dataset,'' in \emph{CVPR}.\hskip 1em plus 0.5em minus
  0.4em\relax IEEE, 2017, pp. 4724--4733.

\bibitem{bai2018empirical}
S.~Bai, J.~Z. Kolter, and V.~Koltun, ``An empirical evaluation of generic
  convolutional and recurrent networks for sequence modeling,'' \emph{arXiv
  preprint arXiv:1803.01271}, 2018.

\bibitem{JMLR:v17:16-177}
J.~Townsend, N.~Koep, and S.~Weichwald, ``Pymanopt: A python toolbox for
  optimization on manifolds using automatic differentiation,'' \emph{JMLR},
  vol.~17, no. 137, pp. 1--5, 2016.

\bibitem{moh2007chi}
A.~Moh'd A~Mesleh, ``Chi square feature extraction based svms arabic language
  text categorization system,'' \emph{Journal of Computer Science}, vol.~3,
  no.~6, pp. 430--435, 2007.

\bibitem{barla2003histogram}
A.~Barla, F.~Odone, and A.~Verri, ``Histogram intersection kernel for image
  classification,'' in \emph{ICIP}, vol.~3.\hskip 1em plus 0.5em minus
  0.4em\relax IEEE, 2003, pp. III--513.

\bibitem{ravanbakhsh2018plug}
M.~Ravanbakhsh, M.~Nabi, H.~Mousavi, E.~Sangineto, and N.~Sebe, ``Plug-and-play
  cnn for crowd motion analysis: An application in abnormal event detection,''
  in \emph{Winter Conference on Applications of Computer Vision (WACV)}.\hskip
  1em plus 0.5em minus 0.4em\relax IEEE, 2018, pp. 1689--1698.

\bibitem{hautamaki2004outlier}
V.~Hautamaki, I.~Karkkainen, and P.~Franti, ``Outlier detection using k-nearest
  neighbour graph,'' in \emph{Proceedings of the 17th International Conference
  on Pattern Recognition, 2004. ICPR 2004.}, vol.~3.\hskip 1em plus 0.5em minus
  0.4em\relax IEEE, 2004, pp. 430--433.

\end{thebibliography}
}
\vspace{-1.4cm}
\begin{IEEEbiography}[{\includegraphics[width=1in,height=1.25in,clip,keepaspectratio]{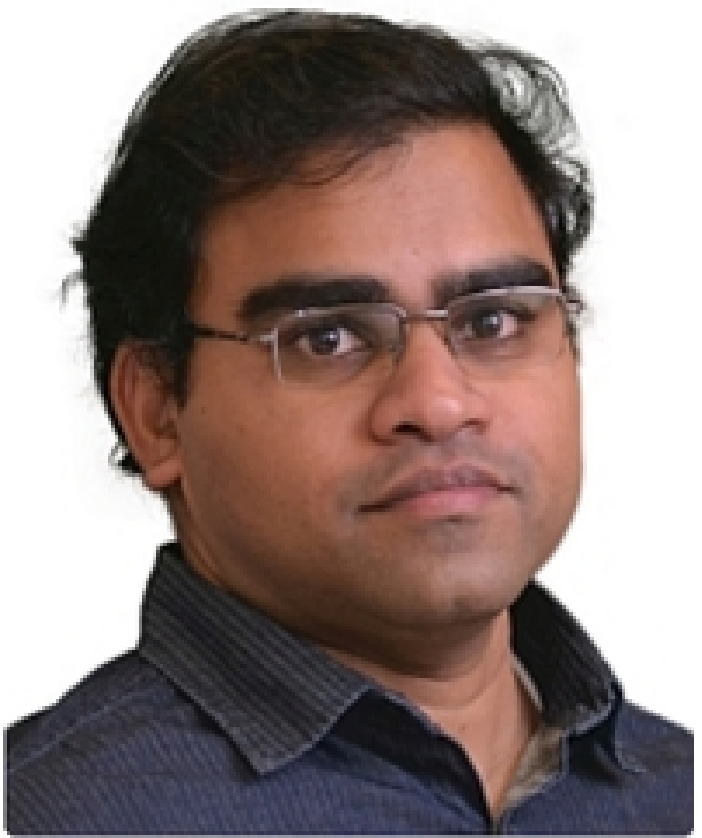}}]{Anoop Cherian}
is a Principal Research Scientist with Mitsubishi Electric Research Labs (MERL) Cambridge, MA and an Adjunct Researcher with the Australian National University. He received his M.S. and Ph.D. degrees in computer science from the University of Minnesota, Minneapolis in 2010 and 2013, respectively. He has broad interests in machine learning, computer vision, deep learning, and non-Euclidean optimization. He has co-authored more than 50 scientific articles, and is also the recipient of several awards, including the Best Student Paper Award at ICIP, 2012.
\end{IEEEbiography}
\vspace{-1.0cm}
\begin{IEEEbiography}[{\includegraphics[width=1in,height=1.25in,clip,keepaspectratio]{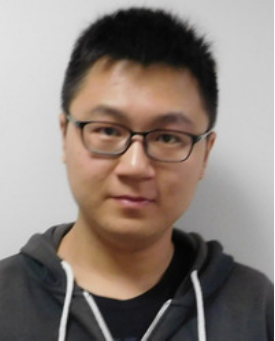}}]{Jue Wang}
is a PhD student with the Research School of Engineering at the Australian National University since 2016. He is also associated with CSIRO's Data61 in Australia. From 2010-2014, he received his double bachelor degree (honors) in Electronic Engineering from Australian National University and Beijing Institute of Technology. His research interest are in the area of computer vision and machine learning.
\end{IEEEbiography}

\end{document}